\PassOptionsToClass{pdflatex}{sn-jnl}

\documentclass[sns-vancuover]{sn-jnl}

\usepackage{tabularx}
\usepackage{booktabs}
\usepackage{amssymb}
\usepackage{pifont}
\usepackage{amsmath}
\usepackage{amsthm}
\usepackage{amssymb}
\usepackage{mathtools}
\usepackage[capitalise]{cleveref}
\usepackage{algorithm}
\usepackage{algorithmicx}
\usepackage{subfigure}
\newenvironment{newstuff}{%
    \leavevmode\color{blue}\ignorespaces%
}{%
}%

\jyear{2021}%

\theoremstyle{thmstyleone}%
\newtheorem{theorem}{Theorem}
\theoremstyle{thmstyletwo}%
\newtheorem{remark}{Remark}%

\theoremstyle{thmstylethree}%

\def\##1\#{\begin{align}#1\end{align}}
\def\$#1\${\begin{align*}#1\end{align*}}
\def\given{\,\vert\,}

\def\tr{\mathop{\text{tr}}\kern.2ex}

\def\P{{\mathbb P}}
\def\E{{\mathbb E}}

\def \inv {^{-1}}

\def \sq {^{2}}

\long\def\comment#1{}

\def\tr{\mathop{\text{Tr}}}

\def\cA{{\mathcal{A}}}
\def\cC{{\mathcal{C}}}
\def\cB{{\mathcal{B}}}
\def\cE{{\mathcal{E}}}

\def\cX{{\mathcal{X}}}
\def\cY{{\mathcal{Y}}}
\def\cZ{{\mathcal{Z}}}

\def\tr{{\text{Tr}}}

\providecommand{\hnormsq}[1]{\tfrac{1}{2}\|#1\|^2}

\newcommand{\yes}{\checkmark}
\newcommand{\no}{\ding{55}}

\newcommand*\diff{\mathop{}\!\mathrm{d}}
\newcommand{\defeq}{\vcentcolon=}

\def\real{{\mathbb{R}}}
\def\R{{\real}}

\newcommand{\commeff}{communication-efficient}

\newcommand{\la}{\langle}
\newcommand{\ra}{\rangle}

\def \ALGO {{\operatorname*{{LocalAdaSEG}}}}
\def \dualgap {{\operatorname*{{DualGap}}}}
\def \head {{\operatorname*{{head}}}}
\def \tail {{\operatorname*{{tail}}}}
\def \sumT {{\sum_{t=1}^T}}
\def \sumttoTm {{\sum_{t=1}^{T-1}}}
\def \sumttotaust {{\sum_{t= 1}^{\tau^*_m }}}
\def \sumttotaustm {{\sum_{t= 1}^{\tau^*_m - 1}}}
\def \sumtfromtaustmtoT {{\sum_{t=\tau^*_m +1}^{T}}}

\def \sumTinSp {{\sum_{t\in S+1}}}
\def \sumTnotinSp {{\sum_{t \notin S+1}}}

\def \sumM {{\sum_{m=1}^M}}
\def \sumtautotm {{\sum_{\tau =1}^{t-1}}}

\def \deltamt {{\delta_t^m}}
\def \deltamtau {{\delta_{\tau}^m}}

\def \zmt {{z^m_t}}
\def \tzmt {{\tilde{z}^m_t}}
\def \tzmtm {{\tilde{z}^m_{t-1}}}
\def \tzmtmo {{\tilde{z}^{\circ}_{t-1}}}
\def \tzmtmst {{\tilde{z}^{m,*}_{t-1}}}
\def \etamt {{\eta^m_t}}
\def \etamtau {{\eta^m_{\tau}}}
\def \etamtsq {{(\eta^m_t)^2}}

\def \cVoneT {{\mathcal{V}_1(T)}}
\def \cVmT {{\mathcal{V}_m(T)}}

\def \etamT {{\eta^m_T}}
\def \ooetamt {{\frac{1}{\eta^m_t}}}
\def \Zmt {{Z^m_t}}
\def \Zmtsq {{(Z^m_t)^2}}

\def \Zmtausq {{(Z^m_\tau)^2}}

\def \gmt {{g^m_t}}
\def \Mmt {{M^m_t}}
\def \ximt {{\xi^m_t}}

\def \tO {{\tilde{O}}}

\newtheorem{assumption}{Assumption}
\newtheorem{lemma}[theorem]{Lemma}

\begin{document}

\title[LocalAdaSEG]{Local AdaGrad-Type Algorithm for Stochastic Convex-Concave Optimization}


\author[1,2]{\fnm{Luofeng} \sur{Liao}}\email{ll3530@columbia.edu}
\author*[2]{\fnm{Li} \sur{Shen}}\email{mathshenli@gmail.com}
\author[2]{\fnm{Jia} \sur{Duan}}\email{xuelandj@gmail.com}
\author[3]{\fnm{Mladen} \sur{Kolar}}\email{mkolar@chicagobooth.edu}
\author[2]{\fnm{Dacheng} \sur{Tao}}\email{dacheng.tao@gmail.com}

\affil[1]{\orgdiv{IEOR}, \orgname{Columbia Unviersity}, \orgaddress{ \state{NY}, \country{USA}}}

\affil[2]{\orgdiv{JD Explore Academy}, \orgname{JD.com Inc}, \orgaddress{ \city{Beijing}, \country{China}}}

\affil[3]{\orgdiv{Booth School of Business}, \orgname{University of Chicago}, \orgaddress{ \state{IL}, \country{USA}}}


\abstract{Large scale convex-concave minimax problems arise in numerous applications, including game theory, robust training, and training of generative adversarial networks. Despite their wide applicability, solving such problems efficiently and effectively is challenging in the presence of large amounts of data using existing stochastic minimax methods. We study a class of stochastic minimax methods and develop a communication-efficient distributed stochastic extragradient algorithm, LocalAdaSEG, with an adaptive learning rate suitable for solving convex-concave minimax problems in the Parameter-Server model. LocalAdaSEG has three main features: (i) a periodic communication strategy that reduces the communication cost between workers and the server; (ii) an adaptive learning rate that is computed locally and allows for tuning-free implementation; and (iii) theoretically, a nearly linear speed-up with respect to the dominant variance term, arising from the estimation of the stochastic gradient, is proven in both the smooth and nonsmooth convex-concave settings. LocalAdaSEG is used to solve a stochastic bilinear game, and train a generative adversarial network. We compare LocalAdaSEG against several existing optimizers for minimax problems and demonstrate its efficacy through several experiments in both homogeneous and heterogeneous settings.
}


\maketitle

\section{Introduction}\label{sec:intro}

Stochastic minimax optimization problems
arise in applications ranging from 
game theory \cite{neumann1928theorie}, 
robust optimization \cite{delage2010distributionally}, 
and AUC Maximization \cite{guo2020communication}, 
to adversarial learning \cite{wang2019towards}
and training of generative adversarial networks (GANs) \cite{goodfellow2014generative}.
In this work, we consider
\begin{align}\label{spp}
  \min_{x\in\mathcal{X}} \max_{y\in\mathcal{Y}} 
  \; \bigg\{ F(x,y) 
  =\int_\Xi f(x,y, \xi) P(\diff \xi)\bigg\},  
\end{align}
where $\cX\subseteq \mathbb{X}$, $\cY\subseteq \mathbb{Y}$ 
are nonempty compact convex sets,
$\mathbb{X}$, $\mathbb{Y}$ are finite dimensional vector spaces,
$\xi$ is a random vector with an unknown probability distribution $P$ supported on a set $\Xi$, and 
$f:\cX\times \cY \times \Xi \to \R$ is a real valued function, which may be nonsmooth. 
Throughout the paper, we assume that the 
expectation $ \E _{\xi\sim P} [f(x,y,\xi)] $ is well defined and finite. 
For all $\xi\in \Xi$, we assume that the function $F(x,y)$ is convex in 
$x\in \cX$ and concave in $y\in \cY$. 
In addition, we assume that $F(x,y)$ is a Lipschitz continuous function.

There are \textbf{three main challenges} in developing
an efficient solver for the large-scale minimax problem \eqref{spp}.
First, the solver should generate converging iterates. In contrast to convex optimization, convergence results for minimax problems are harder to obtain.
Second, the solver should be able to take advantage of parallel computing in a communication-efficient way. Only then can it be applied to problems with large-scale datasets, which are often distributed across multiple workers. 
Third, it is desirable for the solver to choose learning rates in an adaptive manner. It is well known that, in minimax problems, solver performance is susceptible to learning rates. We discuss these 
challenges in detail below.

First, it has been shown that direct application of the (stochastic) gradient descent ascent ((S)GDA)
to solve \eqref{spp} may result in divergence of the
iterates \cite{mertikopoulos2018optimistic,
daskalakis2018training,
gidel2019negative,
mertikopoulos2018cycles}. 
Possible ways to overcome the divergence issue 
are to apply the primal-dual hybrid gradient (PDHG) or (stochastic) extragradient method 
and their variants \cite{mertikopoulos2018optimistic,
daskalakis2018training,gidel2018a,azizian2020tight,liu2020towards,zhao2021accelerated,NEURIPS2020_52aaa62e}.

Second,
it is often desirable to have a \commeff\ distributed solver to solve the stochastic minimax problem
\eqref{spp}. 
The first reason being that the minimax problem \eqref{spp} is often instantiated 
as a finite-sum problem with large-scale datasets (with the distribution $P$
being the empirical distribution over millions of data points), and thus storing and manipulating datasets on multiple workers is a must.
For example, when problem \eqref{spp} is 
specified as BigGAN \cite{brock2018large} over ImageNet \cite{deng2009imagenet}, 
the number of training samples is as many as 14 million. 
Traditional distributed SGDA on the problem \eqref{spp}
may suffer from a considerable communication
burden; reducing communication complexity of the algorithm is a major concern in our paper.
The second reason is that, in some scenarios, data are distributed on mobile devices (such as cell phones or smart watches), and due to privacy concerns, local data must stay on the device. Furthermore, frequent communication among devices is not feasible due to failures of mobile devices (network connectivity, battery level, etc.). This further motivates the design of \commeff\ distributed solvers to eliminate central data storage and improve communication efficiency. For these reasons, communication-efficient distributed solvers for minimax problems have been investigated recently
\cite{beznosikov2021distributed,deng2020local,hou2021efficient,mingruiliu2020decentralized}. 

Third, the performance of stochastic minimax solvers for \eqref{spp} is highly dependent on the learning rate tuning mechanism \cite{heusel2017gans,antonakopoulos2021adaptive}. And yet, designing a solver for \eqref{spp} with an adaptive learning rate is much more challenging compared to the convex case; the value of $F$ at an iterate $(x,y)$ does \emph{not} serve as a performance criterion. For example, for classical minimization problems, the learning rate can be tuned based on the loss evaluated at the current iterate, which directly quantifies how close the iterate is to the minimum. However, such an approach does not extend to minimax problems and, therefore, a more sophisticated approach is required for tuning the learning rate. Development of adaptive learning rate tuning mechanisms for large scale stochastic minimax problems has been explored only recently \cite{ bach2019universal,
babanezhad2020geometry,
ene2020adaptive,
antonakopoulos2021adaptive,
liu2020towards}. 
Hence, we ask

\begin{center}
  \emph{
Can we develop an efficient algorithm for the stochastic minimax problem \eqref{spp} 
that enjoys convergence guarantees, 
 communication-efficiency
and adaptivity \textbf{simultaneously}?}
\end{center}

\begin{figure*}
\centering
\label{Diagram}
\includegraphics[scale=.3]{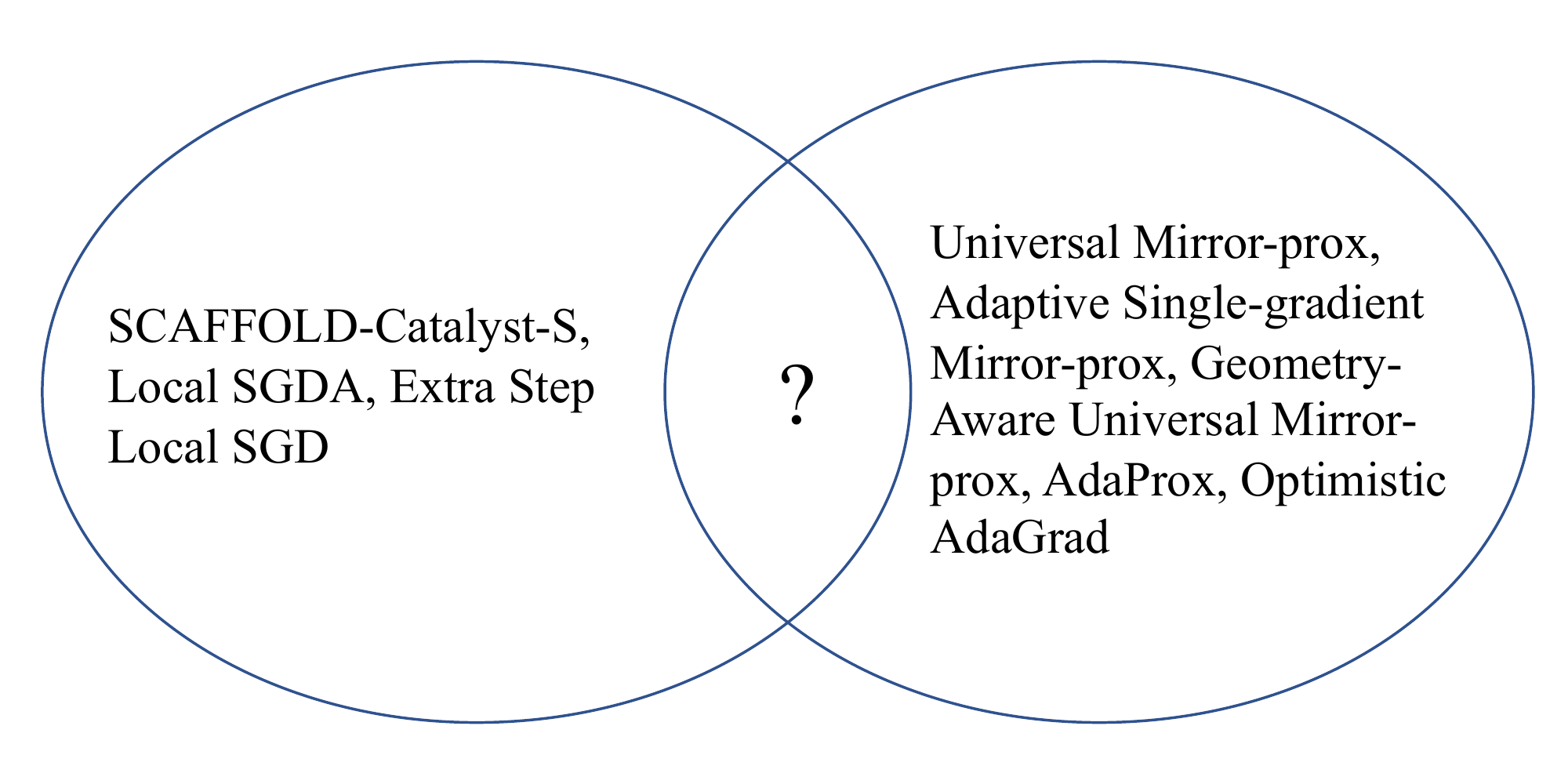}
\caption{A Venn Diagram for related works. Left circle: Communication-efficient methods for stochastic minimax problems. Right circle: Adaptive methods for stochastic minimax problems.}
\end{figure*}

We provide an affirmative answer to this question
and develop $\ALGO$ (Local Adaptive Stochastic Extragradient)
algorithm. Our contributions are three-fold: \\

{\bf{Novel communication-efficient distributed minimax algorithm.}}\ 
Fig. \ref{Diagram}~illustrates the difference between $\ALGO$ algorithm and the existing works. $\ALGO$ falls under the umbrella of the Parameter-Server model \cite{smola2010architecture} and adopts a periodic communication mechanism to reduce the communication cost between the server and the workers, similar to Local SGD/FedAvg \cite{Yu2019onthelinear,stich2018local,Li2020On} in federated learning \cite{mcmahan2021advances}. In addition, in each worker, a local stochastic extragradient algorithm with an adaptive learning rate is performed independently with multiple iterations. Every once in a while, current iterates and adaptive learning rates from all workers are sent to the server. The server computes a weighted average of the iterates, where the weights are constructed from the received local adaptive learning rates. We emphasize that adaptive learning in each worker is distinct from others and is automatically updated according to local data as is done in \cite{chen2021quantized,beznosikov2021distributed}, and different from the existing adaptive distributed algorithms \cite{xie2019local,reddi2021adaptive,chen2021cada}.\\

{\bf{Theoretically optimal convergence rate.}}\ 
Let $M$ denote the number of workers, $\sigma$ denote the variance of stochastic gradients, and $T$ denote the number of local iterations on each worker. For stochastic convex-concave minimax problems, we establish the rate in terms of the duality gap metric \cite{nemirovski2004prox,lin2020near} as $\tilde{O}({\sigma }{/}{\sqrt{MT}})$ in the \emph{nonsmooth and noise-dominant} case and the rate $\tilde{O}({\sigma }{/}{{\sqrt{MT}}}+\text{ higher-order terms})$ in \emph{smooth case with slow cumulative gradient growth}. The terms depending on the variance $\sigma$ achieve the statistical lower bound and are not improvable without further assumptions.
Therefore, the $\ALGO$ algorithm enjoys the linear speed-up property in the stochastic gradient variance term due to the periodic communication mechanism. \\

{\bf{Experimental verification}.}\  
We conduct several experiments on the stochastic bilinear game and the Wasserstein GAN \cite{arjovsky2017wasserstein} to verify the efficiency and effectiveness of the $\ALGO$ algorithm. We also extend the $\ALGO$ algorithm to solve the challenging federated GANs in a heterogeneous setting. The experimental results agree with the theoretical guarantees and demonstrate the superiority of $\ALGO$ against several existing minimax optimizers, such as SEGDA \cite{nemirovski2004prox},
 UMP \cite{bach2019universal}, ASMP \cite{ene2020adaptive}, 
 LocalSEGDA \cite{beznosikov2021distributed}, LocalSGDA \cite{deng2020local}, 
 and Local Adam \cite{beznosikov2021distributed}.  

 \section{Related Work}

Although there has been a lot of work on minimax optimization,
due to space constraints, we summarize only the most closely related
work. Our work is related to the literature on 
stochastic minimax algorithms,
adaptive minimax algorithms, and 
distributed minimax algorithms. We defer a detailed discussion of related work to \cref{sec:related_work} in the appendix.

    

Our work and the proposed $\ALGO$ contribute to the literature described above. To our knowledge, the proposed $\ALGO$ algorithm is the first distributed communication-efficient algorithm for the stochastic minimax problem and simultaneously supports the adaptive learning rate and minibatch size. Moreover, $\ALGO$ communicates only periodically to improve communication efficiency
and uses a local adaptive learning rate, computed on local data in each worker, to improve the efficiency of computation. In addition, $\ALGO$ can also be applied in a non-smooth setting with the convergence guarantee. $\ALGO$ can be seen as a distributed extension of \cite{bach2019universal} with period communication as local SGD \cite{stich2018local}. We note that only very recently a local adaptive stochastic minimax algorithm, called Local Adam, has been used heuristically to train GANs without a convergence guarantee \cite{beznosikov2021distributed}. We summarize the relationship with the existing literature in Table \ref{table:1}.

\begin{table*}[h!]
    \centering
    \begin{tabularx}{\linewidth}{|X |c |c |c|} 
        \hline
     Stochastic minimax algorithms & Nonsmooth ? & Comm. eff. ?& Adaptive ? 
     \\
      \hline\hline
     Mirror SA \cite{nemirovski2009robust},
     SMP \cite{Juditsky2011solving},
     SAMP \cite{chen2017accelerated},
     Optimal Stochastic PDHG-type \cite{zhao2021accelerated}
     & \yes & \no & \no 
     \\
     \hline
    SCAFFOLD-Catalyst-S \cite{hou2021efficient}, Local SGDA \cite{deng2020local},
    Extra Step Local SGD \cite{beznosikov2021distributed}  & \no &  \yes & \no  
    \\ 
    \hline
     Universal Mirror-prox \cite{bach2019universal},
     Adaptive Single-gradient Mirror-prox \cite{ene2020adaptive},
     Geometry-Aware Universal Mirror-prox
     \cite{babanezhad2020geometry},
      AdaProx \cite{antonakopoulos2021adaptive} & \yes & \no & \yes  
     \\
    \hline
    Optimistic AdaGrad \cite{liu2020towards} &  \no * & \no & \yes \\
      \hline\hline
     Our $\ALGO$  & \yes & \yes & \yes  \\
     \hline
    \end{tabularx}
    \caption{Comparison to related works on adaptive or communication-efficient approaches to stochastic minimax problems. Here "Nonsmooth ?" asks whether the algorithm enjoys theoretical guarantees in the nonsmooth convex-concave setting; "Comm. eff.\ ?" asks whether the proposed algorithm is communication-efficient; "Adaptive ?" asks whether the proposed algorithm requires knowledge of problem parameters.  "*": The work of \cite{liu2020towards} discusses non-convex non-concave minimax problems.}
    \label{table:1}
  \end{table*}
  
  \section{Methodology}
  
  \subsection{Notations and Assumptions}\label{sec:assum}

A point $(x^*,y^*)\in \cX\times \cY$ is called a saddle-point for the minimax problem in~\eqref{spp} if for all $(x,y)\in \cX\times\cY$,
\begin{align}
  F(x^*,y)\leq F(x^*,y^*)\leq F(x,y^*).
\end{align}
Under the assumptions stated in \cref{sec:intro}, the corresponding primal, $\min_x \{\max_y F(x,y)\}$,
and dual problem,  $\max_y \{ \min_x F(x,y)\}$, 
have optimal solutions and equal optimal values, denoted $F^*$.
The pairs of optimal solutions $(x^*,y^*)$ form 
the set of saddle-points of $F$ on $\cX\times\cY$.
We denote 
  $\mathbb{Z} = \mathbb{X}\times\mathbb{Y}$, 
  $\mathcal{Z} = \cX\times\cY$, 
  $z = (x,y) \in \mathcal{Z}$, and 
  $z^* = (x^*,y^*) \in \mathcal{Z}$.  
  We use $\| \cdot \|_\cX$, $\|\cdot\|_\cY$, and  $\|\cdot\|_\mathcal{Z}$ 
  to denote the Euclidean norms on $\mathbb{X}$, $\mathbb{Y}$, $\mathbb{Z}$, respectively,
  and let $\|\cdot \|_{\cX,*}$, $\| \cdot \|_{\cY,*}$ and $\| \cdot \|_{\mathcal{Z},*}$ 
  denote the corresponding dual norms. 
  With this notation, 
  $\|z\|_\mathcal{Z} = \sqrt{ \| x\|_\cX \sq + \| y \|_\cY\sq }$ and 
  $\| z \|_{\cZ,*} = \sqrt{\| x\|_{\cX,*}\sq + \| y \|_{\cY,*}\sq }$. 
  Throughout the paper, we focus on the Euclidean setting, but
  note that the results can readily generalize to non-Euclidean cases.

We are interested in finding a saddle-point of $F$ 
  over $\cX\times\cY$. For a candidate solution
  $\tilde{z}=(\tilde{x},\tilde{y}) \in \mathcal{Z}$, 
  we measure its quality by the duality gap, defined as
  \begin{align}\label{dualgap}
  \dualgap(\tilde{z})\defeq \max_{y\in \cY} F(\tilde{x}, y) - \min_{x\in\cX} F(x,\tilde{y}).
  \end{align}
  The duality gap is commonly used as a performance 
  criterion for general convex-concave minimax problems 
  (see, e.g., \cite{nemirovski2004prox,lin2020near}). 
  Note that for all $z\in\cZ$ 
  it holds $\dualgap(z)\geq 0$ and $\dualgap(z)=0$ 
  if and only if $z$ is a saddle-point.
  
  For the stochastic minimax problem \eqref{spp},
  we assume that neither the function $F(x,y)$ nor
  its sub/supgradients in $x$ and $y$ are available.
  Instead, we assume access to an unbiased stochastic 
  oracle $G(x,y,\xi) = [G_x(x,y,\xi), -G_y(x,y,\xi)]$, 
  such that the vector $\E_\xi[G(x,y,\xi)]$ is well-defined
  and $\E_\xi[G(x,y,\xi)] \in [\partial_x F(x,y), -\partial_y F(x,y)]$. 
  For notational convenience, we let
  \begin{align}\label{eq:gradient}
  \tilde{G}(z) \defeq G(x,y,\xi) ,\quad G(z) \defeq \E_\xi[G(x,y,\xi)] .
  \end{align}
  Below, we impose assumptions on the minimax problem~\eqref{spp} 
  and the stochastic gradient oracle~\eqref{eq:gradient}.
  \begin{assumption}[Bounded Domain]\label{as:bddomain}
  There exists $D$ such that
  $\sup_{z\in\cZ} \frac{1}{2}\|z\| \sq \leq D\sq$.
  \end{assumption}
  \begin{assumption}[Bounded Stochastic Gradients] \label{as:bdsg}
  There exists $G$ such
  that $\sup_{z\in\cZ}\|\tilde{G}(z) \|_* \leq G$, P-almost surely.
  \end{assumption}
  
Domain boundedness \cref{as:bddomain} is commonly assumed in the convex-concave minimax literature; see references in \cref{sec:intro}. However, we note that the assumption might be removed in certain settings. For example, \cite{Chen2014,monteiro2011complexity} use a perturbation-based variant of the duality gap as the convergence criterion, and \cite{antonakopoulos2021adaptive} handles unbounded domains via the notion of local norms, while \cite{zhao2021accelerated} handles unbounded domains with access to a convex optimization oracle. The almost sure boundedness Assumption~\ref{as:bdsg} on the gradient oracle seems restrictive but is common in the literature on adaptive stochastic gradient methods (see, e.g., \cite{duchi2011adaptive,chen2018on, bach2019universal, liu2020towards}). In Remark~\ref{rm:unbdgradient} we discuss how to extend our analysis to unbounded oracles.
  
\begin{assumption}[Bounded Variance] \label{as:bdvar}
   There exists $\sigma$ such that  
  $\E_{\xi}\big[ \| G(z) - \tilde{G} (z)\|_*\sq\given z\big] \leq \sigma\sq$ for $P$-almost every $z$.
\end{assumption}
  
We separately analyze the case when the saddle function $F$ is differentiable with Lipschitz gradients.
  
\begin{assumption}[Smoothness] \label{as:smooth}
  Assume that for all $z,z'\in\cZ$, we have $\| G(z) - G(z')\|_*\leq L \|z-z' \|$.
\end{assumption}

\subsection{$\ALGO$ Algorithm}
We introduce the $\ALGO$ algorithm used to solve \eqref{spp} and describe its main features. Algorithm~\ref{algo} details the procedure.

  \begin{algorithm}[h!]
  \caption{\quad $\ALGO(G_0,D; K,M,R;\alpha)$} \label{algo}
  \begin{algorithmic}[1]
  \State \textbf{Input}: $G_0$, a guess on the upper bound of gradients, 
  $D$, the diameter of the set $\cZ$, 
  $K$, communication interval, 
  $M$, the number of workers, 
  $R$, number of rounds, 
  $\alpha$, base learning rate.
  \State \textbf{Initialize}: $\eta^m_1 = D\alpha /G_0$, 
  $\tilde{z}_0=\tilde{z}^m_0=\tilde{z}^{m,*}_0 = 0$ for all $m$, and 
  $S \defeq \{0,K,2K,\dots, RK\}$.
  \For{$t=1,\dots, T=RK$, parallel for workers $m=1,\dots, M$}
  \State  
  update learning rate $ \etamt = $
  \$
   {D\alpha } \big/ {\sqrt{ G_0^2 + \sum_{\tau=1}^{t-1} 
      \frac{
          \textstyle{\| z^m_{\tau}-\tilde{z}^{m,*}_{\tau-1} \|^2 + \| z^m_{\tau}-\tilde{z}^m_{\tau } \|^2} 
      }{
          {{5} (\etamtau)^{2}} }}
      }
  \$
  \label{line:stepsize}
  \If{$t -1\in S$} \label{line:communicatebegin}
  \State worker $m$: send $(\etamt, \tzmtm)$ to server 
  \State server: compute $\tzmtmo$, the weighted average of $\{ \tzmtm\}_{m\in[M]}$, and broadcast it to workers
  \[
  w^m_t = 
  \frac{
      (\etamt)\inv 
  }{
      \textstyle\sum_{m'=1}^M (\eta^{m'}_t)\inv 
  }
  , \; \tzmtmo = \sumM w^m_t \cdot \tzmtm
  \] 
  \label{line:weights}
  \State worker $m$: set $\tzmtmst = \tzmtmo$ \label{line:sync} \label{line:communicateend}
  \Else
  \State set $\tzmtmst = \tzmtm$ \label{line:nosync}
  \EndIf
  \State update
  \begin{align*}
  \zmt &= 
  \Pi_{\cZ}[ \tzmtmst - \etamt \Mmt ]
  &&\text{ with } \Mmt = \tilde{G}(\tzmtmst)
  \\
  \tzmt &= 
  \Pi_{\cZ} [ \tzmtmst - \etamt \gmt ] 
  &&\text{ with } \gmt = \tilde{G}(\zmt)
  \end{align*} \label{line:extragradient}
  \EndFor
  \State \textbf{Output}: $\frac{1}{TM} \sumM\sumT z^m_t$
  \end{algorithmic}
  \end{algorithm}

\textbf{The Parameter-Server model.}\
$\ALGO$ uses $M$ parallel workers which, in each of $R$ rounds, independently execute $K$ steps of extragradient updates (Line~\ref{line:extragradient}). The adaptive learning rate is computed solely based on  iterates occurred in the local worker (Line~\ref{line:stepsize}). Let $S \defeq \{0,K,2K,\dots, RK=T\}$ denote the time points of communication. At a time of communication  ($t\in S+1$, Lines~\ref{line:communicatebegin}--\ref{line:communicateend}), the workers communicate and compute the weighted iterate, $\tzmtmo$, defined in Line~\ref{line:weights}. Then the next round begins with a common iterate $\tzmtmo$. Finally, $\ALGO$ outputs the average of the sequence $\{\zmt \}_{m\in[M],t\in[T]}$. Overall, each worker computes $T=KR$ extragradient steps locally, 
for a total of $2MT$ stochastic gradient calls (since each extragradient step, Line~\ref{line:extragradient}, requires two calls of gradient oracle) with $R$ rounds of communication (every $K$ steps of computation). 
  
\textbf{Extragradient step.}\ 
At the time when no communication happens ($t-1\notin S $), Line~\ref{line:extragradient} reduces to
\begin{align*}
\zmt &= 
\Pi_{\cZ}[ \tzmtm - \etamt \Mmt ]
&&\text{ with } \Mmt = \tilde{G}(\tzmtm) ,
\\
\tzmt &= 
\Pi_{\cZ} [ \tzmtm - \etamt \gmt ] 
&&\text{ with } \gmt = \tilde{G}(\zmt) ,
\end{align*}
where $\Pi_\cZ(z) = \operatorname*{argmin}_{z'\in\cZ} \|z - z'\|_2 $ is the projection operator onto the compact set $\cZ$.
The above update is just the extragradient (EG) algorithm \cite{korpel1976}
that is commonly used to solve minimax problems; see references in \cref{sec:intro}.

\textbf{Periodic averaging weights.}\ 
The proposed weighted averaging scheme in Line~\ref{line:weights} is different from existing works on local SGD and Local Adam \cite{beznosikov2021distributed}. At the time of averaging ($t\!-\! 1\!\in\! S$), $\ALGO$ pulls the averaged iterate towards the local iterate with a smaller learning rate. For the homogeneous case studied in this paper, we expect $w^m\sim 1/M$.

\textbf{Intuition of local adaptive learning rate scheme.}\ 
The adaptive learning rate scheme (Line~\ref{line:stepsize}) follows that of Bach and Levy \cite{bach2019universal} closely. To develop intuition, consider the deterministic setting where $\sigma = 0$ and define $(\deltamt)\sq \defeq \|\gmt\|_*\sq + \|\Mmt\|_*\sq $. If we ignore the projection operation, the learning rate $\etamt$ would look like $\etamt \sim 1/(1+\sum_{\tau=1}^{t-1} (\deltamtau)^2)^{1/2}$. In the nonsmooth case, the subgradients might not vanish as we approach the solution (in the case of convex optimization, consider the function $f(x)=|x|$ near $0$), and we only have $\liminf_{t\to\infty} \deltamt > 0$. This implies $\etamt$ will vanish at the rate $1/\sqrt{t}$, which is the optimal learning rate scheme for nonsmooth convex-concave problems \cite{bach2019universal, antonakopoulos2021adaptive}. For the smooth case, one might expect the sequence $\{\deltamt \}_t$ to be square-summable and $\etamt \to \eta_{\infty}^m > 0$, in which case the learning rate does not vanish. Additionally, the adaptive learning rate for each worker is locally updated to exploit the problem structure available in worker's local dataset. This makes our local adaptive learning rate scheme distinct compared to existing distributed adaptive algorithms for minimization problems \cite{xie2019local,reddi2021adaptive,chen2021cada}. Very recently, \cite{beznosikov2021distributed} used local Adam for training  conditional GANs efficiently, but they provide theoretical guarantees only for the local extragradient without adaptivity.

\textbf{Adaptivity to $(G,L,\sigma)$.}\ 
Our algorithm does not require knowledge of problem parameters such as the size of the gradients $G$, the smoothness $L$, or the variance of gradient estimates $\sigma$. Instead, we only need an initial guess of $G$, denoted $G_0$, and the diameter of the feasible set, $D$. Following \cite{bach2019universal}, we define 
\begin{align}\label{eq:defgamma}
\gamma \defeq \max\{G /G_0, G_0/G\} \geq 1. 
\end{align}
This quantity measures how good our guess is and appears in the convergence guarantees for the algorithm. Our algorithm still requires knowledge of the problem class, as we need to use a different base learning rate, $\alpha$, for smooth and nonsmooth problems; see Theorems~\ref{thm:nonsmooth} and \ref{thm:smooth}, respectively. 

\subsection{Convergence Results}\label{sec:convergence}

We state two theorems characterizing the convergence rate of $\ALGO$ for the smooth and nonsmooth problems. We use the notation $\tilde{O}$ to hide absolute constants and logarithmic factors of $T=KR$ and problem parameters. The proofs are given in \cref{sec:proof:thmnonsmooth} and \cref{sec:proof:thmsmooth} of the appendix. Recall the definition of $\gamma$ in~\eqref{eq:defgamma}.

\begin{theorem}[Nonsmooth Case] \label{thm:nonsmooth}
Assume that Assumptions \ref{as:bddomain}, \ref{as:bdsg}, and \ref{as:bdvar} hold. 
Let $\bar z = \ALGO(G_0, D; K,M,R;1)$.
Then
\[\E[\operatorname*{DualGap}(\bar z)] = \tilde{O}\bigg( \frac{\gamma G D}{\sqrt{T}} + \frac{\sigma D}{\sqrt{MT}}\bigg)\,.\]
\end{theorem} 

\begin{theorem}[Smooth Case] 
    \label{thm:smooth}
    Assume that Assumptions \ref{as:bddomain}, \ref{as:bdsg}, \ref{as:bdvar}, and \ref{as:smooth} hold.
    Let $\bar z = \ALGO(G_0, D;K,M,R;1/\sqrt{M})$. 
    Define the cumulative norms of stochastic gradients occurred on worker $m$ as
    \begin{align} \label{eq:defcVmT}
    \cVmT \defeq \E \left[\sqrt{\sumT \| \gmt \|_*\sq + \| \Mmt \|_*\sq}\right].
    \end{align}
    Then 
    \#
    \E[\operatorname*{DualGap}(\bar z)] =
    \tilde{O}   \!\bigg(\! \frac{\sigma D}{{\sqrt{MT}}}
    \!+\! \frac{D \sqrt{M}  \cVoneT}{T} 
    \!+\! \frac{\gamma \sq LD\sq M^{-1/2}}{T}
    \!+\! \frac{\gamma G D \sqrt{M}}{T}
    \bigg) .
    \label{eq:bound2}
    \#
\end{theorem}

\begin{remark}[The term $\cVoneT$.]\label{remark-linear-speedup}
Note that by symmetry $\cVmT = \mathcal{V}_1(T)$ for all $m$. Although a trivial bound on $\cVoneT$ is $\cVoneT\leq G\sqrt{2T}$, 
typically we have $\cVoneT \ll \sqrt{T}$ in practice \cite{duchi2011adaptive, Reddi2018on, chen2018universal, chen2018on, liu2020towards}, especially in the sparse data scenarios. For example, consider the bilinear saddle-point problem $\min_{x\in\cX}\max_{y\in\cY} \big\{x^\top ( \sum_{i=1}^n p_i M_i) y\big\}$, where a larger weight $p_i>0$ means the matrix $M_i$ appears more frequently in the dataset. When most of matrices with large weights are row-sparse and column-sparse, the quantity $\cVoneT$ is much smaller than $G\sqrt{2T}$. Theorem~\ref{thm:smooth_noV}, in the appendix, shows that with a different choice of the base learning rate $\alpha$ one can obtain a near linear speed-up result, which removes the dependence on $\cVoneT$: for large $T$, 
\[
\E[\operatorname*{DualGap}(\bar z)] = 
\tilde{O}\left( \frac{\sigma D}{{\sqrt{MT^{1-2\epsilon}}}}
+ \frac{\gamma\sq LD\sq}{T^{1-2\epsilon}}
+ \frac{LD\sq M }{T}
+ \frac{\gamma G D M^{3/2}}{T^{1+\epsilon}}
\right),
\]
for any $\epsilon \in (0, \tfrac12 )$.
Following the discussion in \cite{chen2018universal,liu2020towards}, when the cumulative growth of the stochastic gradient is slow, i.e., $\cVoneT = O(T^{b})$ for some $0<b<\tfrac{1}{2}$,
then the second term in~\eqref{eq:bound2} is $O(DM^{3/2}/T^{1-b})$
and
linear speed-up is achieved, since
as $T\to\infty$, the dominating term become $ O(\sigma D / \sqrt{MT})$. 
\end{remark}

\begin{remark}[Extension to unbounded stochastic gradient oracle]
\label{rm:unbdgradient}
Our analysis can be extended 
to unbounded homogeneous and light-tailed oracles using the following argument.
Let 
\$
\|G\|_\infty\defeq \sup_{z\in\cZ} \| G(z) \|_*<\infty ,
\$
which upper bounds the expectation of the 
SG oracle. Assume $\|\tilde{G}(z) - G(z)\|_* / \| G\|_\infty$ is 
independent of $z$ and follows the distribution 
of the absolute value of a standard normal. 
Define the set $\cZ'\defeq\{ \zmt, \tzmtmst\}_{t,m}$
of all iterates.
For any $0< \delta < 1$, define the event
\begin{align*}
\cE \defeq \Big\{
     \max_{z'\in\cZ'} \|\tilde{G}(z')-G(z')\|_*  \leq  G_{T,\delta } \defeq 
      \| G\|_\infty  \cdot  \big(\sqrt{2\log(4MT)} +\sqrt{2\log(2/\delta)}\big)
\Big\}.
\end{align*}
Then $\P(\cE) \geq 1-\delta$; see Appendix \ref{app:unbdsg}. 
We can repeat the proof of Theorem~\ref{thm:nonsmooth} and Theorem~\ref{thm:smooth}
on the event $\cE$ and 
interpret our results with $G$ replaced by $G_{T,\delta }$, 
which effectively substitutes $G$ with $\| G\|_\infty$ at
the cost of an extra $\log(T)$ factor.
\end{remark}

\begin{remark}[Baseline 1: Minibatch EG] \label{rm:minibatch_EG}

We comment on the performance of an obvious baseline that implements minibatch stochastic EG using $M$ workers. Suppose the algorithm takes $R$ extragradient steps, with each step using a minibatch of size $KM$, resulting in a procedure that communicates exactly $R$ times. The performance of such a minibatch EG for general nonsmooth and smooth minimax problems \cite{bach2019universal,ene2020adaptive} is, respectively,\footnote{These bounds hold due to Theorem 4 of
\cite{ene2020adaptive}, whose rates for nonsmooth and smooth problems are
of the form $O(R(G+\sigma)/\sqrt{T})$ and $O(\beta R\sq / T + R\sigma/\sqrt{T})$, respectively.
The claim follows 
with $\sigma$ in the original theorem 
statement replaced by $\sigma/\sqrt{KM}$, 
$\beta$ by $L$, $R$ by $D$, $G$ by $\|G\|_\infty$, and $T$ by $R$.}
\$ 
{O}\bigg(\frac{\sigma D}{\sqrt{KMR}} + \frac{\| G\|_\infty D}{\sqrt{R}} \bigg)
 \text{ and } 
{O}\bigg(\frac{\sigma D}{\sqrt{KMR}} + \frac{LD\sq}{R} \bigg)\,.
\$
Under the same computation and communication structure, our algorithm enjoys adaptivity, achieves the same linear speed-up in the variance term $\frac{\sigma D}{\sqrt{KMR}}$, and improves dependence on the gradient upper bound $\|G\|_\infty$ and the smoothness parameter $L$, which is a desirable property for problems where these parameters are large.
\end{remark}

\begin{remark}[Baseline 2: EG on a single worker]
Another natural baseline is to run EG on a single worker for $T$ iterations with batch-size equal to one. The convergence rates for this procedure in nonsmooth and smooth cases
are $O(\sigma D /\sqrt{T} + \|G \|_\infty D / \sqrt{T})$ 
and $O(\sigma D /\sqrt{T} + L D\sq/T)$, respectively.
In the smooth case, EG on a single worker is inferior to minibatch EG, since the dominant term for the former is $1/\sqrt{T}$, but it is $1/\sqrt{MT}$ for the latter. On the other hand, in the nonsmooth case, minibatch EG reduces the variance term, but the term involving the deterministic part degrades. Therefore, in the nonsmooth case, we can only claim  that the minibatch EG is better than the single-worker mode in the noise-dominant regime $\sigma = \Omega(\|G\|_\infty \sqrt{M})$.
\end{remark}

\begin{remark} [On the choice of $K$]
Consider the baseline minibatch EG (see \cref{rm:minibatch_EG}) which runs as follows: the algorithm takes $R$ extragradient steps, with each step using a minibatch of size $KM$, resulting in a procedure that communicates exactly $R$ times. Note this procedure has exactly the same computation and communication structure as $\ALGO$, facilitating a fair comparison. In the non-smooth case, our theory shows that $\ALGO$ dominates minibatch EG regardless of the choice $K$. Therefore, let us focus the discussion on the \emph{smooth loss with slow gradient growth case}. Suppose that the gradient growth term $\mathcal{V}_{m}(T):=\mathbb{E} \big[(\sum_{t=1}^{T}\left\|g_{t}^{m}\right\|_{*}^{2}+\left\|M_{t}^{m}\right\|_{*}^{2})^{1/2} \big]$ admits a rate $\mathcal{V}_{m}(T) = O(T^b)$ for some $0 < b < 1/2$. Theorem 2 then shows that $\ALGO$ enjoys a convergence rate (ignoring problem parameters $L, D$ and $G$)
\begin{align*}
    \frac{1}{\sqrt{MKR}} + \frac{\sqrt{M}}{(KR)^{1-b}} + \frac{\sqrt{M}}{ KR} \;,
\end{align*} 
where $M$ is the number of machines, $R$ the communication rounds, and $K$ is the length between two communications. The minibatch EG attains the convergence rate 
\begin{align*}
    \frac{1}{\sqrt{MKR}} + \frac{1}{R} \;.
\end{align*}
Both algorithms achieve linear speedup, i.e., the dominant term is $O(\sigma/\sqrt{MKR})$.
In order for $\ALGO$ to be comparable with minibatch EG in the higher order term, we set $\sqrt{M}/(KR)^{1-b} = \Theta(1/R)$ and $\sqrt{M}/(KR) = O(1/R)$ and obtain $K = \Theta(\sqrt{M}T^b)$. With this choice of $K$, $\ALGO$ achieves a communication efficiency no worse than minibatch EG with the crucial advantage of being tuning-free. Compared with case of optimizing strongly-convex functions, local SGD needs $K = O(\sqrt{T})$ to achieve linear speedup \cite{stich2018local}.
The discussion here is purely theoretical, since the exponent of gradient growth $b$ is hard to estimate in practice.
\end{remark}

\paragraph{Proof Sketch of Theorem \ref{thm:smooth}}
    We present a proof sketch for the smooth case. Recall the update formula 
    \begin{align*}
        \zmt &= 
        \Pi_{\cZ}[ \tzmtmst - \etamt \Mmt ]
        &&\text{ with } \Mmt = \tilde{G}(\tzmtmst),
        \\
        \tzmt &= 
        \Pi_{\cZ} [ \tzmtmst - \etamt \gmt ] 
        &&\text{ with } \gmt = \tilde{G}(\zmt).
        \end{align*}
\cref{fig:computation} provides a computation diagram and illustrates the relationship between the above variables.
\begin{figure}[ht]
        \centering
        \includegraphics[scale = .6]{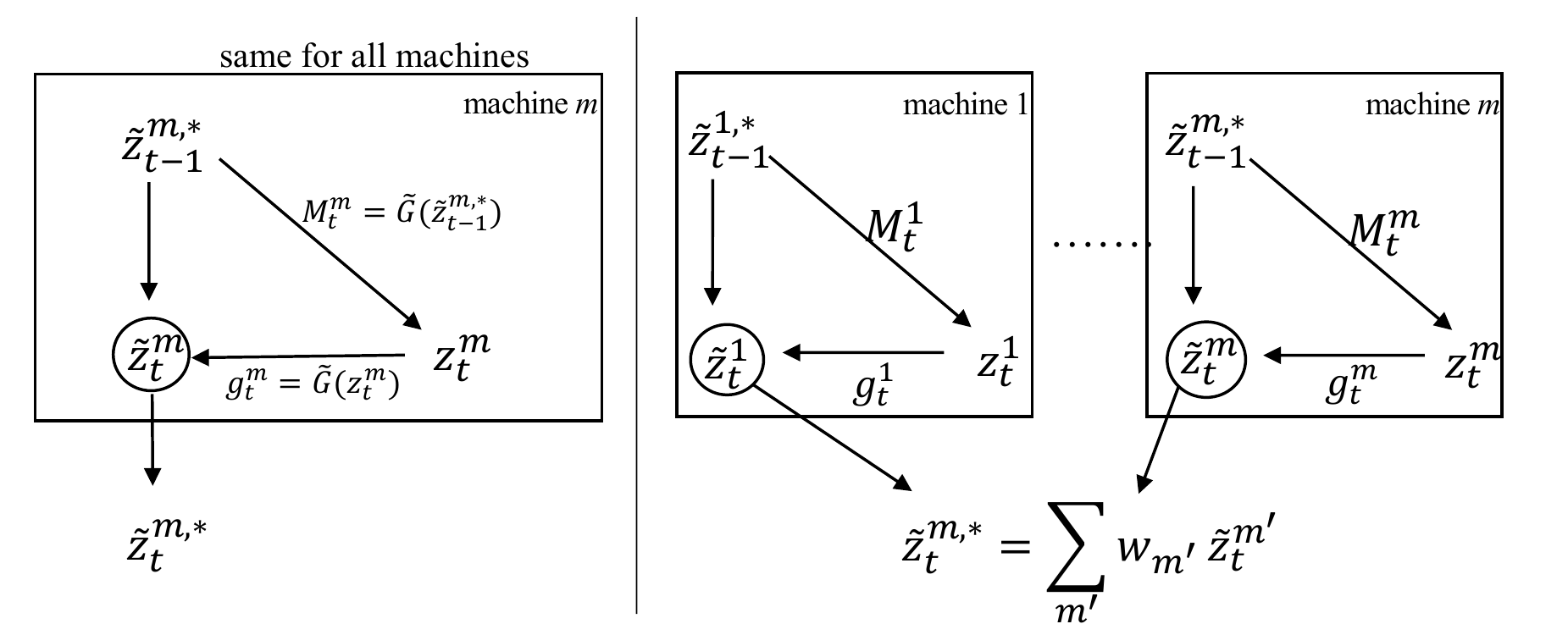}
        \caption{The computation diagram for $\ALGO$. Left panel: computation on machine $m$ when no communication ($t\notin S$). Right panel: computation on machine $m$ when on communication round ($t\in S$)}
        \label{fig:computation}
\end{figure}

We define the noise in the gradient operator $G$ by
\$\ximt \defeq G(\zmt) - \gmt = G(\zmt) - \tilde{G}(\zmt).\$ 
Moreover, we define a gradient-like quantity 
$$\left(Z_{t}^{m}\right)^{2}:=\frac{\left\|z_{t}^{m}-\tilde{z}_{t-1}^{m, *}\right\|^{2}+\left\|z_{t}^{m}-\tilde{z}_{t}^{m}\right\|^{2}}{5\left(\eta_{t}^{m}\right)^{2}}.
$$ 
If we ignore the projection operator in the update, the term $(Z_{t}^{m})$ will be of a similar scale as the gradients $\tilde G(\zmt)$ and $\tilde G(\tzmt)$.

We begin with the following decomposition: for all $z\in \cZ$,
\begin{align*}
   &  \sumT\sumM \big\langle \zmt - z, G(\zmt) \big\rangle 
    \\
&= \sumT\sumM \big\langle  \zmt - z, \ximt \rangle+ \sumT\sumM \langle \zmt - z, \gmt \rangle
\\ 
&\leq \underbracket{\sumT\sumM \big\langle  \zmt - z, \ximt \big\rangle}_{{I}(z)} 
\\
&\quad + \underbracket{\sumT\sumM \frac{1}{\etamt}\Big(\hnormsq{z-\tzmtmst} - \hnormsq{z-\tzmt}\Big) }_{{II}(z)}
\\ 
& \quad \underbracket{ - \sumT\sumM  \frac{1}{\etamt}\Big(\hnormsq{\zmt-\tzmtmst} + \hnormsq{\zmt-\tzmt}\Big)}_{{III}} 
\\
& \quad + \underbracket{\sumT\sumM \|\gmt-\Mmt \|_* \cdot \|\zmt -\tzmt\|}_{{IV}}, 
\end{align*}
where we have used a descent lemma for EG updates common in the literature (Lemma 4 in our paper).
The reason we care about the above quantity is that by the convexity-concavity of the problem, the duality gap metric can be upper-bounded by this term.

Next, we analyze each term separately. The term $I(z)$ characterizes the noise of the problem and eventually contributes to the noise term $\frac{\sigma}{\sqrt{KMR}}$. For the term $II$ we use a telescoping argument and show that it can be upper bounded by $\sum_{m,t}\etamt \Zmtsq$. The telescoping argument can be applied due to the averaging weights $w^m_t$ in the algorithm. The term $III$ is negative. We keep the tail part of $III$ which cancels the tail part of the term $IV$. For the term $IV$ we use the smoothness property of the problem and show that it can be bounded by $\sum_{m,t} \etamtsq \Zmtsq$. Finally, two sums of the form $\sum_{m,t}\etamt \Zmtsq$ and $\sum_{m,t} \etamtsq \Zmtsq$ remain to be handled. For this we use the well-known basic inequality $\sum_{i=1}^{n} {a_{i}}/({a_{0}+\sum_{j=1}^{i-1} a_{j}}) = O(\log(1+\sum_i a_i))$ and $\sum_{i=1}^{n} {a_{i}}/{\sqrt{a_{0}+\sum_{j=1}^{i-1} a_{j}}} = \Theta (\sqrt{\sum_i a_i})$ for positive numbers $a_i$'s.

Nonadaptive local algorithms rely on choosing a vanishing stepsize that is usually inversely proportional to a prespecified number of total iterations $T$. The freedom to choose the stepsize based on a prespecified $T$ is crucial in the proofs of these algorithms and allows canceling of the asynchronicity of updates caused by local updates and the bias in those updates caused by data heterogeneity. This is the case for both convex optimization and convex-concave optimization. However, in the adaptive algorithm regimes, such a proof technique is clearly not viable.

Our algorithm requires a carefully designed iterates averaging scheme, with weight inversely proportional to stepsize. Such averaging-scheme is designed to account for the asynchronicity of local iterates and is automatically determined by the optimization process. This is what enables the extension of an Adam-type stepsize to parallel settings, which is highly nontrivial.

\section{Experiments}\label{sec:experiments}

We apply $\ALGO$ to the stochastic bilinear minimax problem introduced in \cite{gidel2018a,beznosikov2021distributed} and train the Wasserstein generative adversarial neural network (Wasserstein GAN) \cite{arjovsky2017wasserstein}. For the homogeneous setting, to demonstrate the efficiency of our proposed algorithm, we compare $\ALGO$ with minibatch stochastic extragradient gradient descent (MB-SEGDA) \cite{nemirovski2004prox}, minibatch universal mirror-prox (MB-UMP) \cite{bach2019universal},
minibatch adaptive single-gradient mirror-Prox (MB-ASMP) \cite{ene2020adaptive},
extra step local SGD (LocalSEGDA) \cite{beznosikov2021distributed}, and 
local stochastic gradient descent ascent (LocalSGDA) \cite{deng2020local}.
We further extend the proposed $\ALGO$ algorithm to solve federated WGANs with a heterogeneous dataset to verify its efficiency. 
\begin{newstuff}
To validate the practicality of $\ALGO$, we also train the BigGAN \cite{brock2018large} over CIFAR10 dataset under the heterogeneous setting. 
\end{newstuff}
In this setting, we also compare $\ALGO$ with Local Adam \cite{beznosikov2021distributed}. We emphasize here that whether Local Adam converges is still an open question, even for the stochastic convex-concave setting. 

\begin{figure*}
	\centering
	\subfigure{
		\includegraphics[width = 0.45\textwidth]
		{ 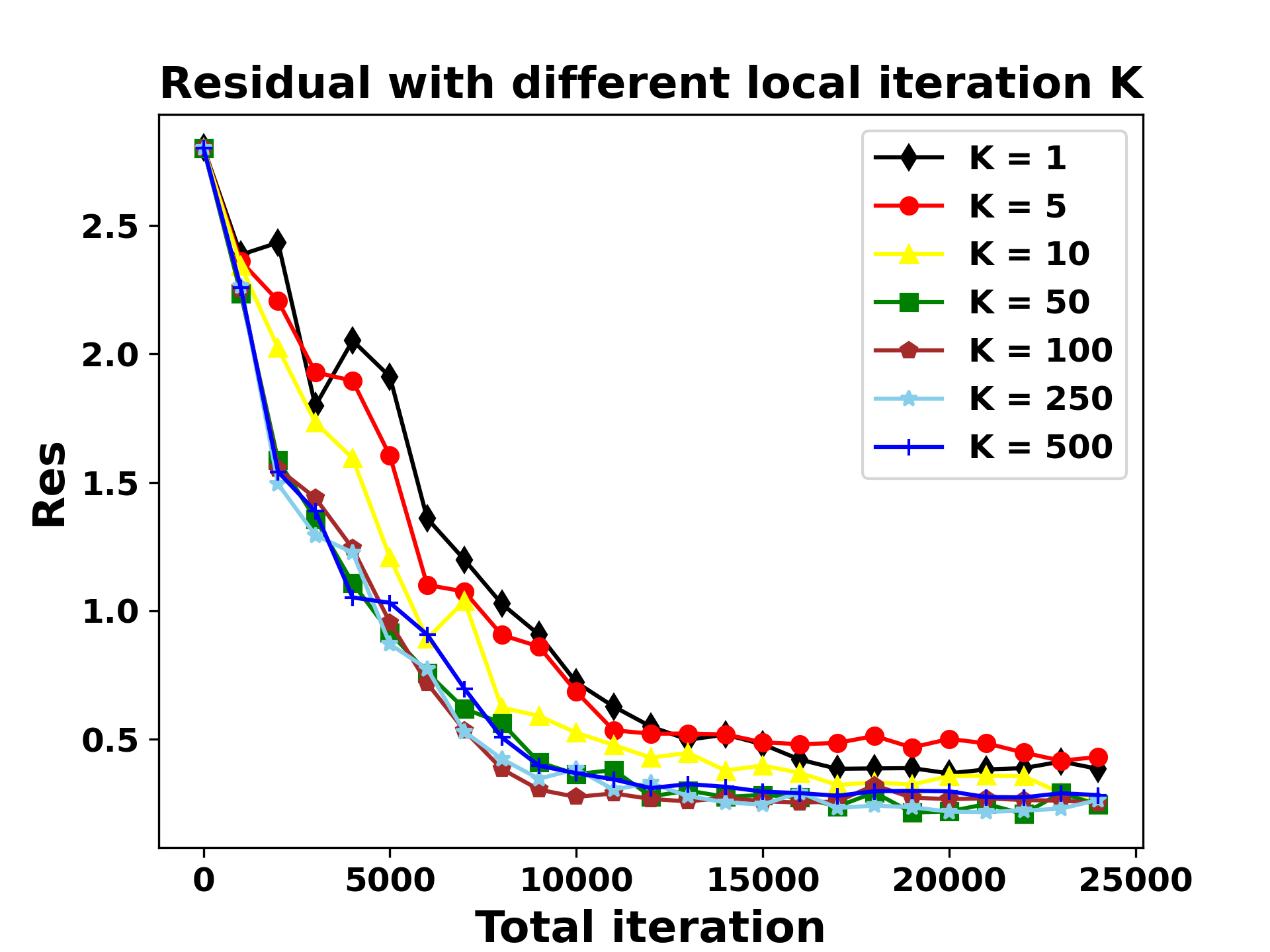}}
	\subfigure[]{
		\includegraphics[width = 0.45\textwidth]
		{ 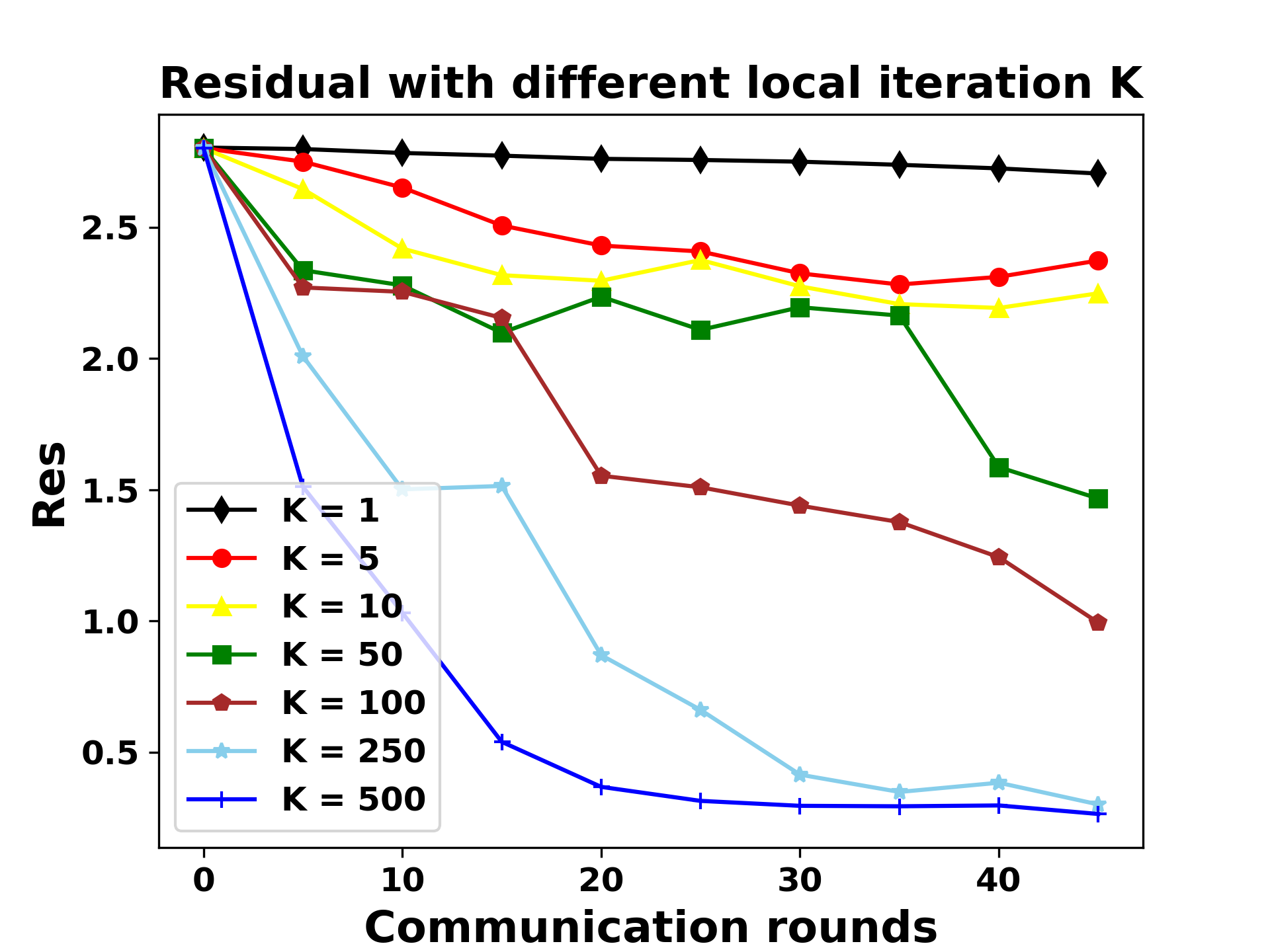}}
	\subfigure[]{
		\includegraphics[width = 0.45\textwidth]
		{ 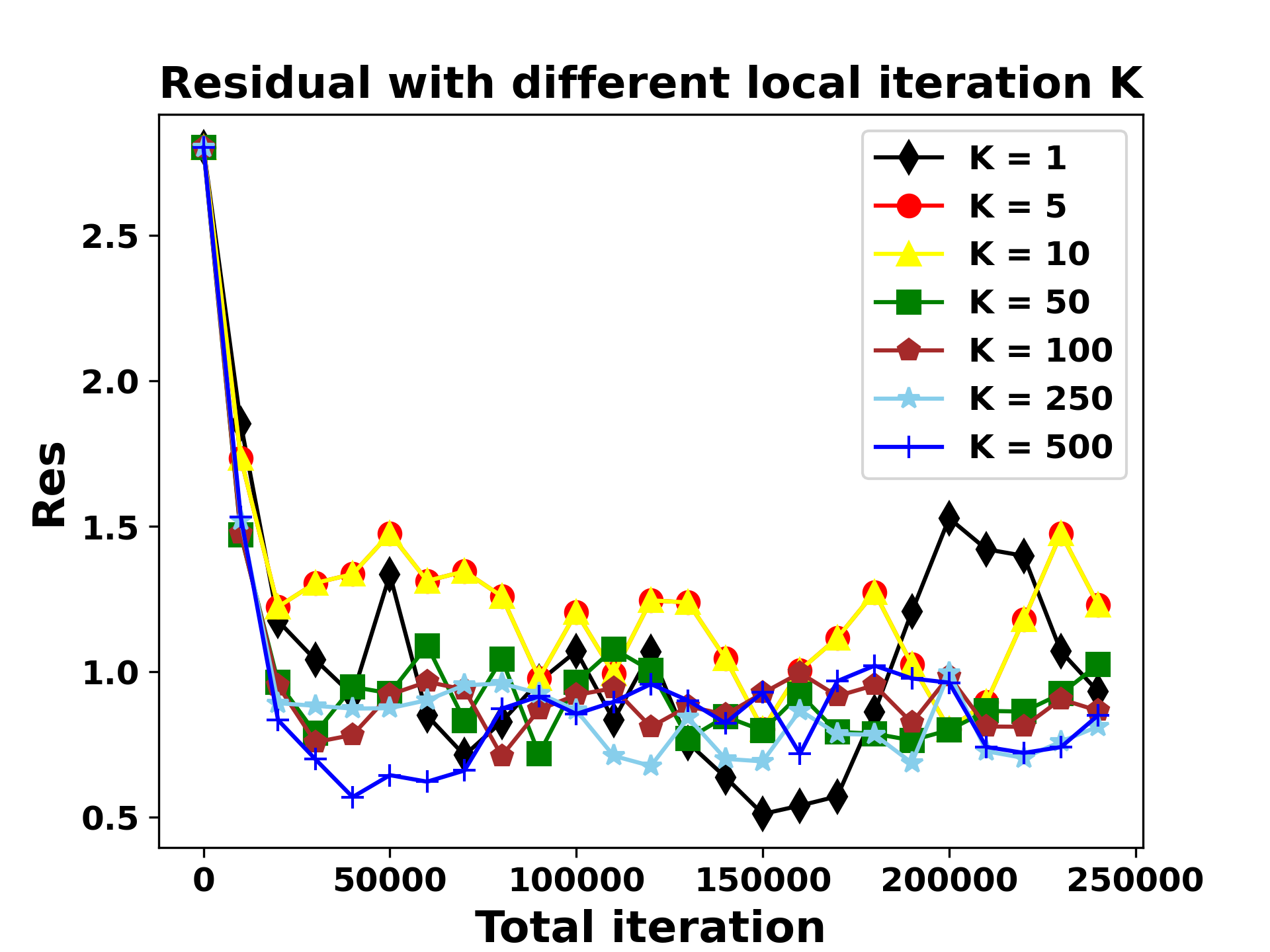}}
	\subfigure[]{
		\includegraphics[width = 0.45\textwidth]
		{ 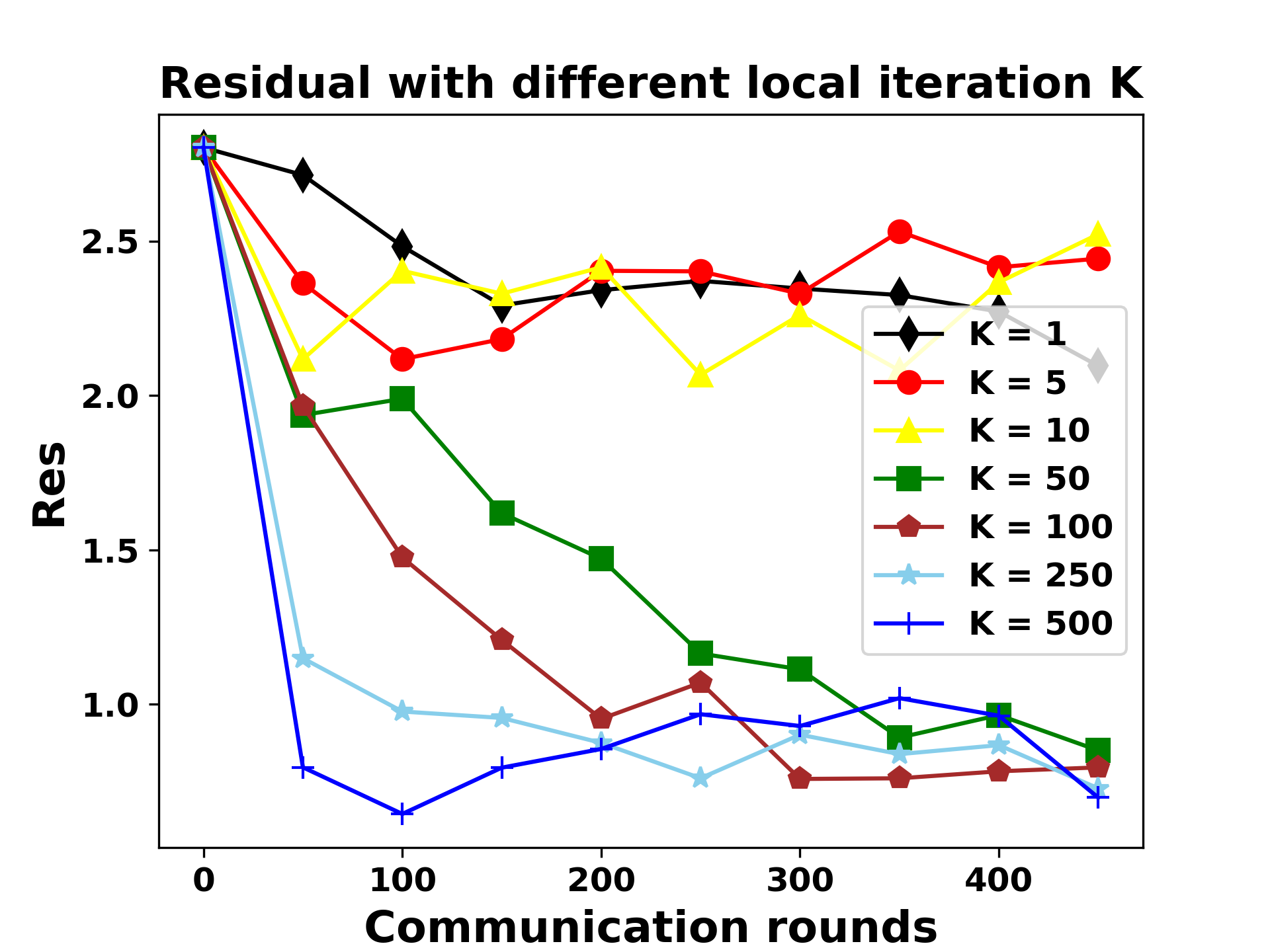}}
		%
	
	\caption{\small Subfigures (a)-(b) and (c)-(d) plot the residual of $\ALGO$ 
	against the total number of iterations $T$ and communications $R$, 
	with varying numbers of local iterations $K$.
	We also investigate the effect of noise level ($\sigma = 0.1$ in (a)(b) and $\sigma = 0.5$ in (c)(d)). }
	\label{fig:diff_local_iter}
\end{figure*}

\begin{figure*}
	\centering
	\subfigure[]{
    	\includegraphics[width = 0.45\textwidth]
    	{ 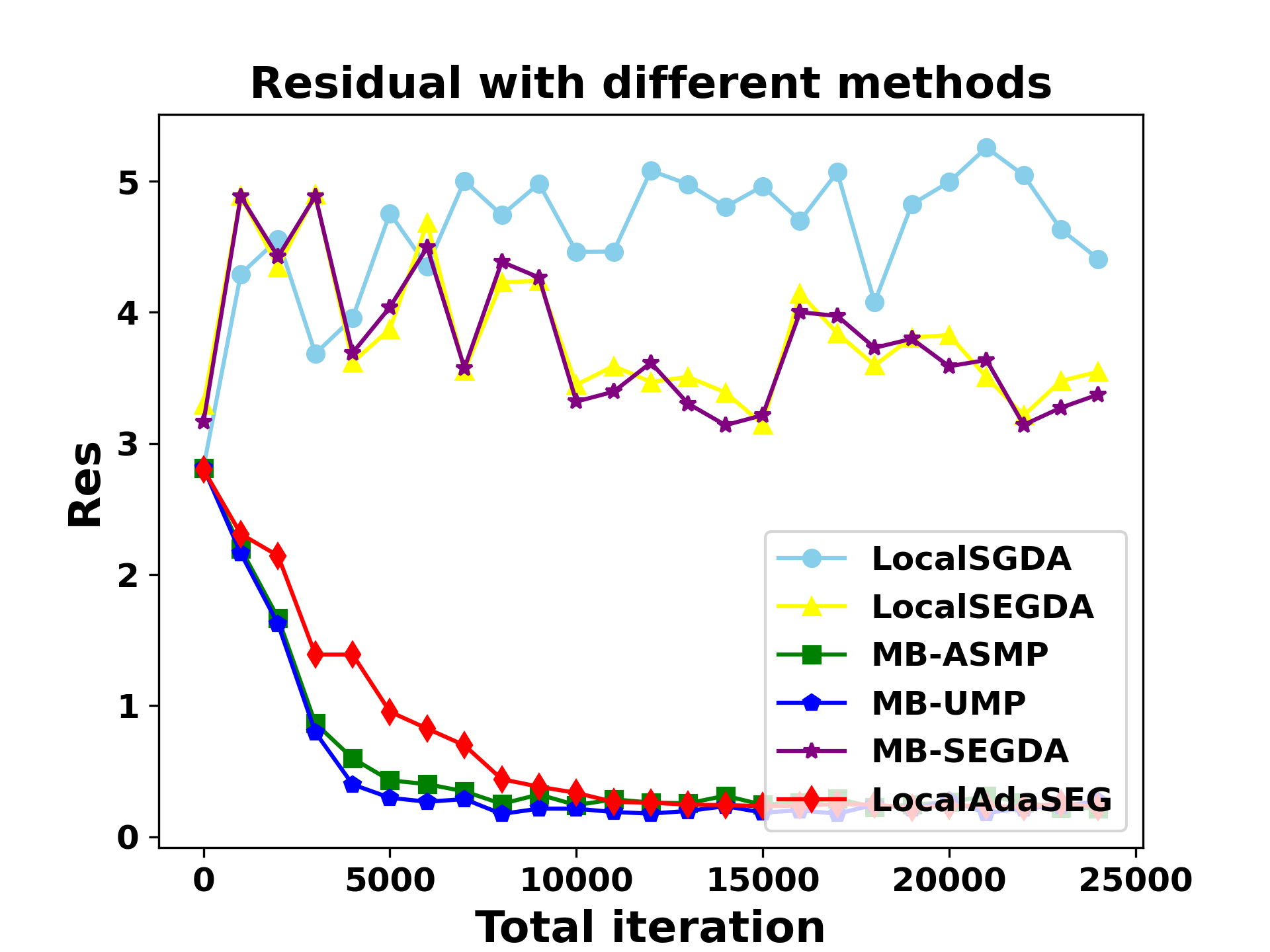}}
	\subfigure[]{
		\includegraphics[width = 0.45\textwidth]
		{ 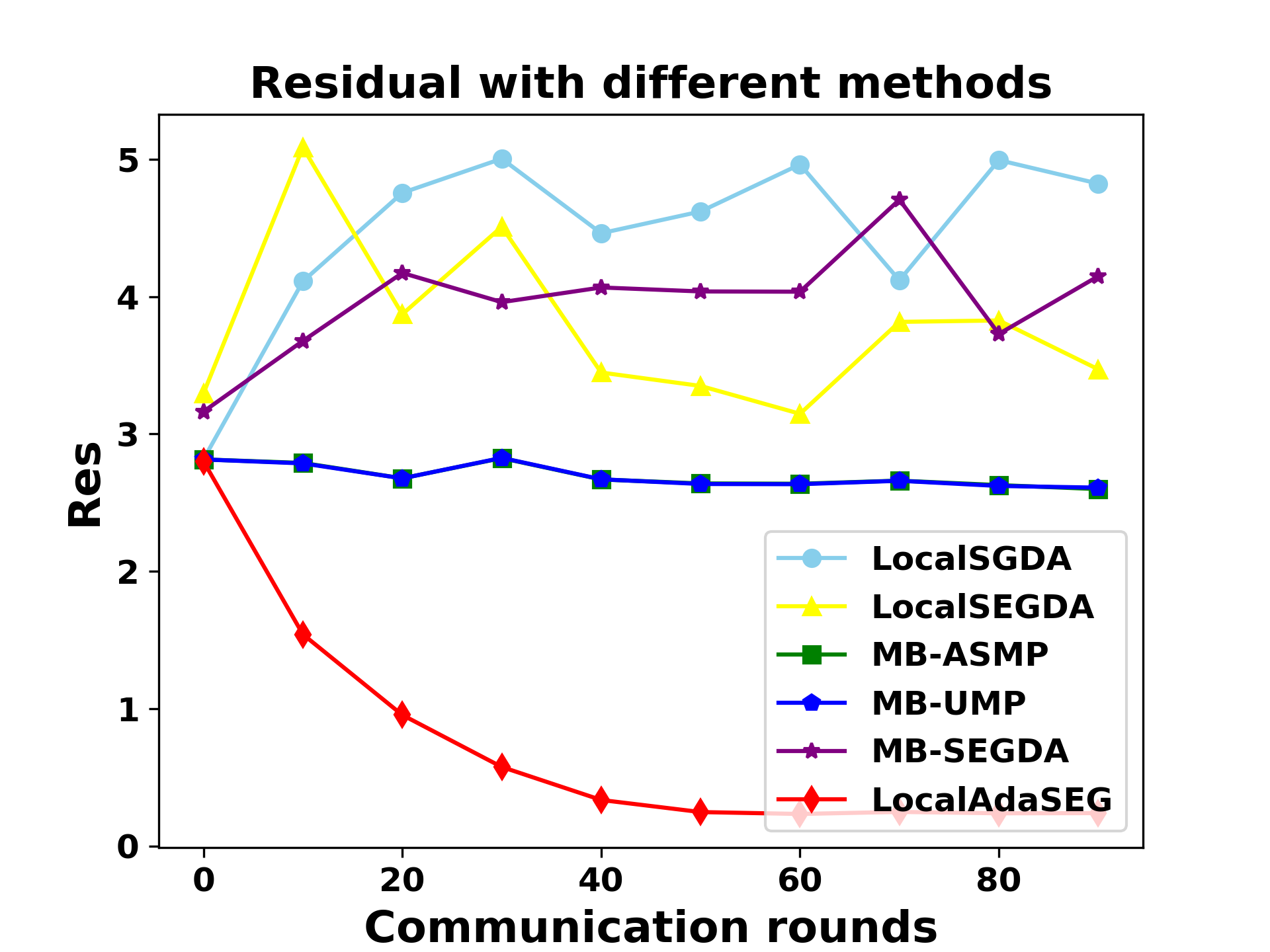}}

	\subfigure[]{
		\includegraphics[width = 0.45\textwidth]
		{ 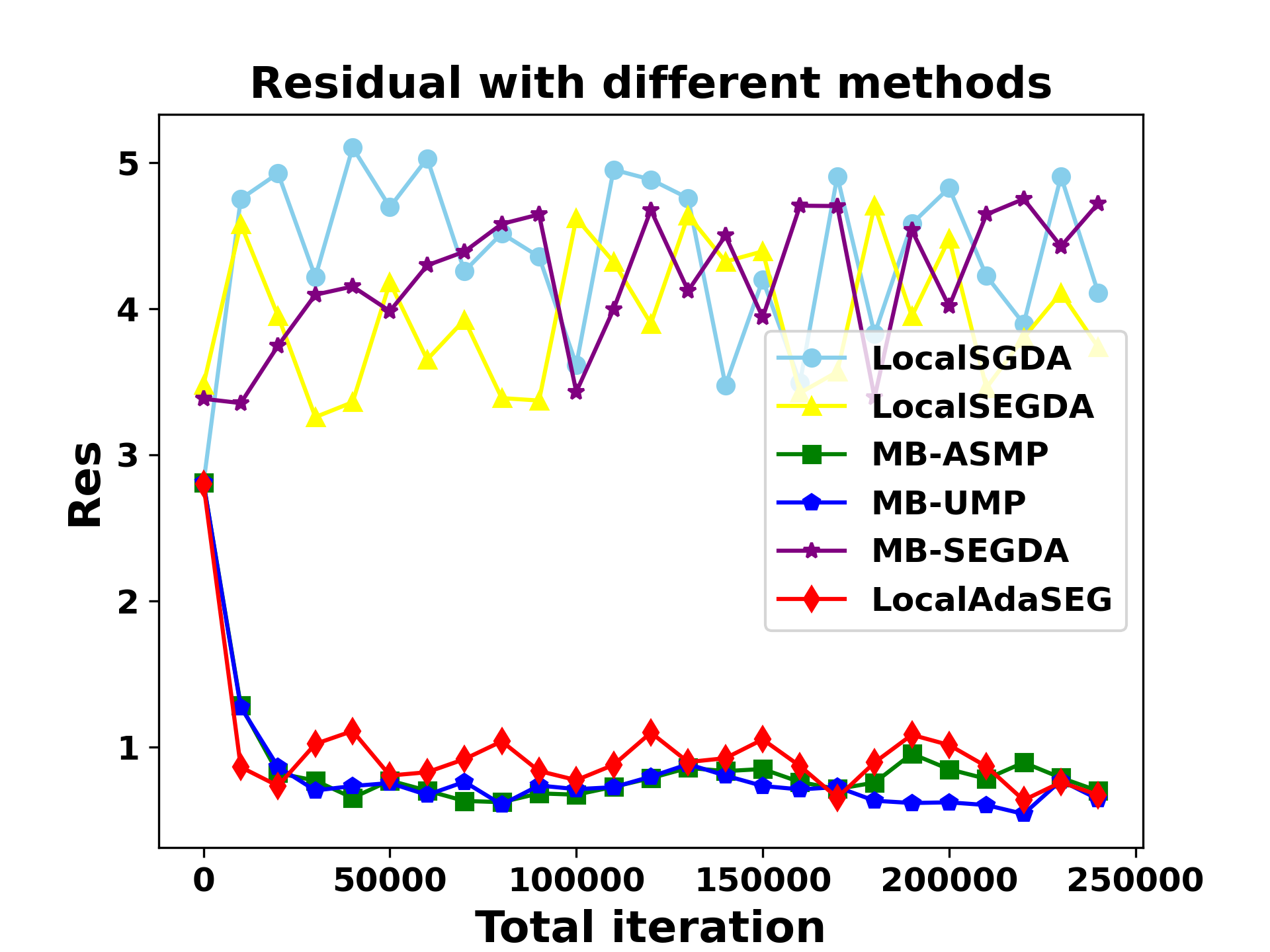}}
	\subfigure[]{
		\includegraphics[width = 0.45\textwidth]
		{ 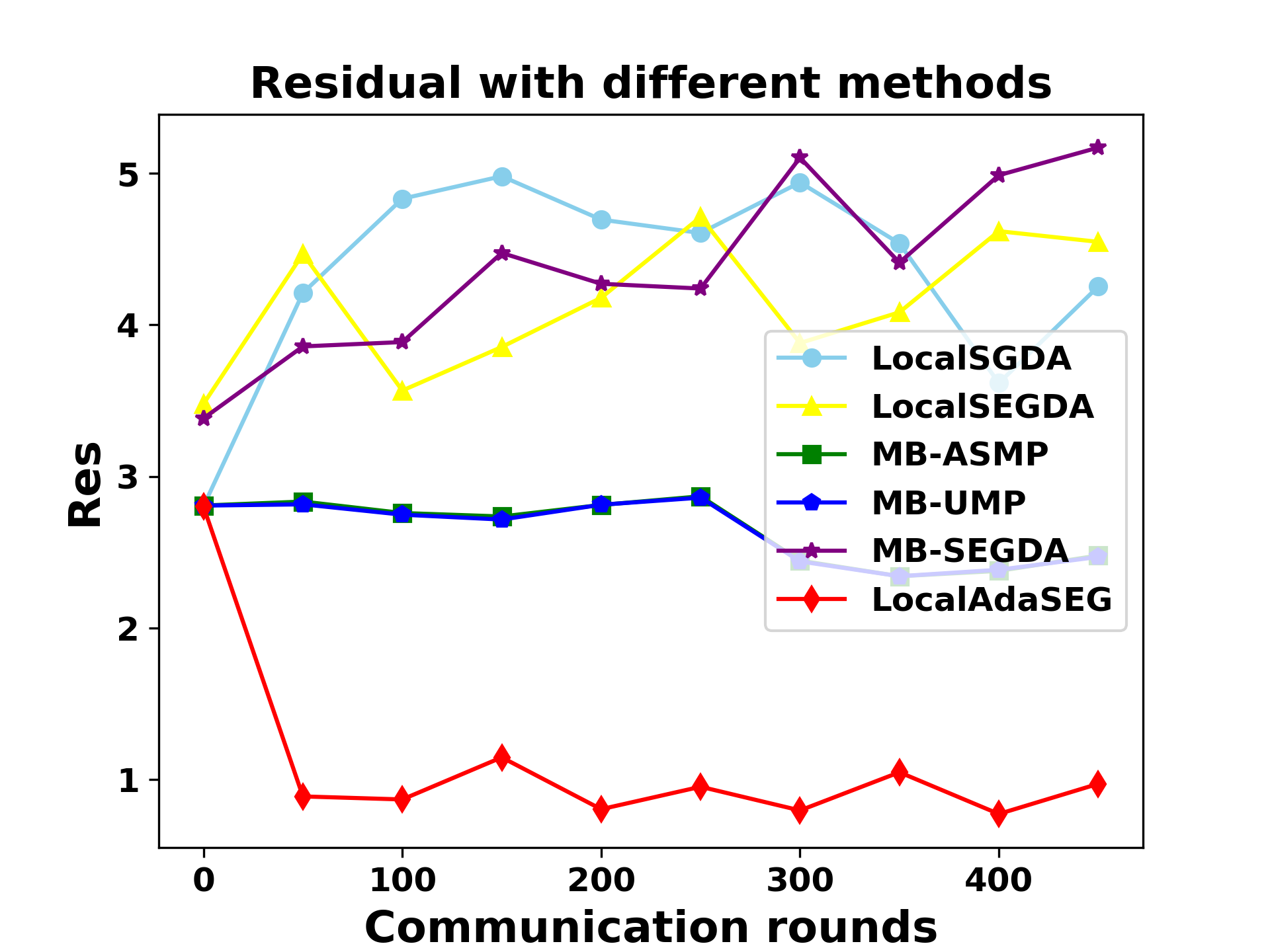}}

	\caption{\small Subfigures (a)-(b) and (c)-(d) compare $\ALGO$ with existing optimizers.
	We plot the residuals against the total number of iterations $T$ and communications $R$
	with different noise levels ($\sigma = 0.1$ in (a)(b) and  $\sigma = 0.5$ in (c)(d)).}
	\label{fig:diff_local_iter_agaisnt_SOTA}
	\vspace{0cm}
\end{figure*}

\subsection{Stochastic bilinear minimax problem}
We consider the stochastic bilinear minimax problem with box constraints 
\#
\label{spp-bilinear}
\min_{x\in C^{n}}\max_{y\in C^n} F(x,y) \#
where
\$
   F(x, y) \defeq  
    \E_{\xi\sim P}\big[x^{\top}Ay + (b+\xi)^{\top}x +(c+\xi)^{\top}y \big],
\$
Here $C^n = [-1,1]^n$ is a box in $\R^n$, the tuple $(A,b,c)$ is deterministic, and the perturbation variable $\xi$ follows the normal distribution with variance $\sigma$. We define the KKT residual ${\rm Res}(x,y)$ as: 
\$ 
{\rm Res}(x,y)^2   \defeq &  \|x- \Pi_{C^n}\big(x-(Ay+b)\big)\|^2 
 \\
 & + \|y- \Pi_{C^n}\big(y+(Ax+c)\big)\|^2.
\$
It is not hard to verify that given $(x^*,y^*) \in \mathbb{R}^{n}\times \mathbb{R}^{n}$, 
${\rm Res}(x^*,y^*) = 0$ if and only if $(x^*,y^*)$ belongs to the set of saddle-points of the bilinear minimax problem \eqref{spp-bilinear}. During experiments, we use ${\rm Res}(x,y)$ to measure the quality of the approximate solution obtained by different optimizers.

{\bf Dataset Generation.}\ 
We uniformly generate $b$ and $c$ in $[-1, 1]^n$ with $n = 10$. The symmetric matrix $A$ is constructed as
$A = {\bar{A}}/{\max\big(|b|_{\max}, |c|_{\max}\big)}$, where $\bar{A}\in [-1, 1]^{n \times n}$ is a random symmetric matrix. 
We emphasize that $A$ is merely symmetric, but not semi-definite.  To simulate the distributed environment, we distribute $(A, b, c)$ to $M$ workers, where $M = 4$. Each worker solves the above bilinear problem locally with an optimization algorithm. We instantiate $\ALGO$ with different numbers of local iterations $K \in \{1, 5, 10, 50, 100, 250, 500\}$, and different noise levels $\sigma \in \{0.1, 0.5\}$, shown in \cref{fig:diff_local_iter}. A larger $\sigma$ indicates more noise in the stochastic gradients, making problem~\eqref{spp-bilinear} harder. Furthermore, we compare $\ALGO$ by setting the local iteration $K = 50$ against several existing optimizers, illustrated in \cref{fig:diff_local_iter_agaisnt_SOTA}.

{\bf Experimental Results.}\ 
In \cref{fig:diff_local_iter}, $\ALGO$ provides stable convergence results under different configurations of local iterations $K$ and noise levels $\sigma$. Figure (b)(d) illustrates that a suitably large $K$ could accelerate the convergence speed of $\ALGO$.
Figure (a)(c) illustrates that a large variance would result in unstable optimization trajectories. The findings of the experiment agree with our theoretical predictions: (i) a larger $T=KR$ improves convergence; (ii) the variance term dominates the convergence rate of $\ALGO$; a large variance term will slow down $\ALGO$.  In \cref{fig:diff_local_iter_agaisnt_SOTA}, (a)(c) illustrate that adaptive variants of stochastic minimax optimizers, i.e., $\ALGO$, MB-UMP, and MB-ASMP, achieve better performance compared to standard ones such as LocalSGDA, LocalSEGDA, and MB-SEGDA, whose learning rates are hard to tune for minimax problems. Furthermore, when compared in terms of communication rounds in (b)(d), $\ALGO$ converges faster than other distributed stochastic minimax optimizers,
demonstrating the superiority of $\ALGO$.  


To validate the performance of our proposed method, we conduct the comparison of the asynchronous case and the synchronous case of $\ALGO$ for the stochastic bilinear minimax problem. We also compare asynchronous and synchronous cases with the single-thread version (SEGDA with MKR iterations) from the aspects of residual and wallclock time. Finally, we evaluate the quantity of $V_t$ with the update $t$. The experimental details are described in Appendix \ref{app:add_exp}. As can be seen in \cref{fig:additional_exp} (in Appendix \ref{app:add_exp}), compared with synchronous cases, asynchronicity only affects the convergence rate that is slower than the synchronous version with respect to the communication rounds. Compared to SEGDA of MKR iterations, our proposed $\ALGO$ can achieve more stable and better performance. Regarding the quantity of $Vt$, it is really much smaller than the dominant variance term.

\subsection{Wasserstein GAN}

We train Wasserstein GAN (WGAN) to validate the efficiency of $\ALGO$ on a real-world application task. This is a challenging minimax problem as the objectives of both generator and discriminator are non-convex and non-concave. The description of the problem and implementation details are placed in \cref{wgan-discription}.


{\bf Experimental results.}
Fig. \ref{fig:wgan_iid_agaisnt_SOTA} and \ref{fig:wgan_noniid_agaisnt_SOTA}  (in \cref{wgan-discription})
compare MB-UMP, MB-ASMP, LocalAdam and $\ALGO$ in a homogeneous and 
heterogeneous setting, respectively.
In \cref{fig:wgan_iid_agaisnt_SOTA}(a) and \cref{fig:wgan_noniid_agaisnt_SOTA}(a), 
MB-UMP, MB-ASMP, LocalAdam and $\ALGO$ quickly converge to a solution with a low FID value. 
However, when compared in terms of communication rounds in
\cref{fig:wgan_iid_agaisnt_SOTA}(b) and \cref{fig:wgan_noniid_agaisnt_SOTA}(b), 
$\ALGO$ and Local Adam converge faster than other optimizers and reach 
a satisfactory solution in
just a few rounds. In \cref{fig:wgan_iid_agaisnt_SOTA}(c)
and \cref{fig:wgan_noniid_agaisnt_SOTA}(c),
all the listed optimizers achieve a high IS. 
In particular, the IS of $\ALGO$ and Local Adam increases much faster
with less communication than MB-UMP, MB-ASMP as shown in 
\cref{fig:wgan_iid_agaisnt_SOTA}(d) and \cref{fig:wgan_noniid_agaisnt_SOTA}(d).  

In \cref{fig:wgan_under_diff_dist} and \cref{fig:compare_wgan_under_diff_dist}, we show and compare the FID and IS of $\ALGO$ with other optimizers under different data distributions. As can be seen from \cref{fig:wgan_under_diff_dist}, $\ALGO$ converges faster when the Dirichlet distribution parameter $\alpha$ decreases. In \cref{fig:compare_wgan_under_diff_dist}, when data distribution changes, our $\ALGO$ can still converge faster than other existing optimizers.

\subsection{BigGAN}
To validate the practicability of our proposed $\ALGO$ method, we apply LocaAdaSEG to train the large-scale BigGAN \cite{brock2018large} model over the CIFAR10 dataset. The description of BigGAN and parameter setup are placed in $\cref{biggan_cifar10}$.

{\bf Experimental results.}
\cref{fig:bigGAN_FID_IS} illustrates the comparison of the FID and IS against communication rounds by using $\ALGO$ and existing optimizers.
As can be seen from  \cref{fig:bigGAN_FID_IS}(a), $\ALGO$ and Local Adam can reach a satisfactory FID value in a few rounds.  Similarly, from \cref{fig:bigGAN_FID_IS}(b), we can see that the IS value of $\ALGO$ and Local Adam is much higher than that of MB-UMP and MB-ASMP. In a word, the FID and IS values of $\ALGO$ and Local Adam converge much faster than that of other optimizers.

\paragraph{Additional Discussions} To end this section, we briefly discuss the limitation of current work.

Theoretical limitations. Our theory is applicable to the homogeneous setting, meaning each worker has access to data from one distribution. However, in practice, data heterogeneity is a main factor practitioners must take into account for distributed learning. 
    We briefly discuss technical challenges here.
    For the heterogeneous case, the theory for \emph{non-adaptive} algorithms relies on choosing a very small stepsize, usually inverse proportional to a prespecified number of total iterations $T$. The freedom to choose the stepsize based on a prespecified $T$ is crucial in those proofs and enables canceling the bias caused by local updates, a.k.a.\ client drifts. The same situation also occurs in the convex optimization case. However, our goal is to have an adaptive algorithm that does not depend on the problem parameters or a prespecified $T$. For this reason, we leave such an important open question for future work. 

Experimental limitations.
    In the scale of the dataset, we experimented with should be increased to showcase the computation benefit of the proposed algorithm. 
    At the current stage we have experimented with MNIST data and further, add CIFAR 10 experiments after reviewers' suggestions.
    Application to other ultra-large datasets such as ImageNet requires significant engineering efforts and will be left for future investigation. We should emphasize that our paper mainly contributes to the theoretical understanding of adaptive algorithms in distributed settings.

\section{Conclusion}\label{sec:conclusion}
We proposed an adaptive communication-efficient distributed stochastic extragradient algorithm in the Parameter-Server model for stochastic convex-concave minimax problem, $\ALGO$. We theoretically showed $\ALGO$ that achieves the optimal convergence rate with a linear speed-up property for both nonsmooth and smooth objectives. Experiments verify our theoretical results and demonstrate the efficiency of $\ALGO$. 

For future work, since that the current analysis merely holds for the homogeneous setting, a promising direction is to extend the theoretical result of $\ALGO$ to the heterogeneous setting that better models various real-world applications, such as federated GANs \cite{beznosikov2021distributed} and robust federated learning \cite{NEURIPS2020_ac450d10}. In addition, extending theoretical results from the stochastic convex-concave setting to the stochastic nonconvex-(non)concave setting is an interesting and challenging research direction.

\section*{Declarations}

\begin{itemize}
\item Funding
    (This work is supported by the Major Science and Technology Innovation 2030 “Brain Science and Brain-like Research” key project (No. 2021ZD0201405).)
\item Conflict of interest/Competing interests (The authors declare that they have no conflict of interest.)
\item Ethics approval 
    (Not Applicable.)
\item Consent to participate (Not Applicable.)
\item Consent for publication (Not Applicable.)
\item Availability of data and materials (The data used in this work is all public.)
\item Code availability (The codes of the proposed method will be released after
publishing.)
\item Authors' contributions (All authors contributed to the study conception
and design. The first draft of the manuscript was written by Luofeng Liao, and all authors commented on previous versions of the manuscript. All authors read and approved the final manuscript.)
\end{itemize}

\noindent
If any of the sections are not relevant to your manuscript, please include the heading and write `Not applicable' for that section. 

\bigskip
\begin{flushleft}%
Editorial Policies for:

\bigskip\noindent
Springer journals and proceedings: \url{https://www.springer.com/gp/editorial-policies}

\bigskip\noindent
Nature Portfolio journals: \url{https://www.nature.com/nature-research/editorial-policies}

\bigskip\noindent
\textit{Scientific Reports}: \url{https://www.nature.com/srep/journal-policies/editorial-policies}

\bigskip\noindent
BMC journals: \url{https://www.biomedcentral.com/getpublished/editorial-policies}
\end{flushleft}

\bibliography{citation.bib}
\bibliographystyle{abbrv}

\clearpage
\begin{appendices}

\section{Related Works} \label{sec:related_work}

\textbf{Stochastic minimax algorithms.}\ 
Stochastic convex-concave minimax problems~\eqref{spp} have been extensively studied in the optimization literature and are usually solved via variants of PDHG or extragradient methods, for example, \cite{Chambolle2010,zhao2021accelerated, nemirovski2004prox, nemirovski2009robust, Juditsky2011solving,judisky2011firstorder, Chen2014,beznosikov2021distributed}. 
\cite{chen2017accelerated} and \cite{Juditsky2011solving} adopted mirror-prox-type methods to tackle the stochastic convex-concave minimax problem with ${O}({1}/{\sqrt{T}})$ convergence rate. \cite{zhao2021accelerated} proposed an accelerated stochastic PDHG-type algorithm with Bergman divergence to solve the stochastic convex-concave minimax problem with a similar ${O}({1}/{\sqrt{T}})$ convergence rate dominated by the stochastic variance term. However, while all these algorithms \cite{chen2017accelerated,Juditsky2011solving,zhao2021accelerated} have achieved the optimal rate according to the low and upper bound for the stochastic convex-concave minimax problem \cite{beznosikov2021distributed}, their performance is highly influenced by the choice of the learning rate, which is either using sufficiently small constants or diminishing learning rates. \\
 
\textbf{Adaptive minimax algorithms.}\ 
Adaptive learning rate in stochastic optimization is first developed for minimization problems \cite{duchi2011adaptive}. Its variants \cite{kingma2017adam,Reddi2018on, zou2019sufficient,chen2021towards,chen2021quantized} are widely used to train deep learning models. The key feature of the adaptive learning rate is that it can automatically adjust the learning rate during the training process and achieve faster convergence. Recently, the adaptive learning rate has also been developed for minimax algorithms to accelerate the training process, since the learning rate in stochastic minimax algorithm is hard to tune based on the minimax loss, as compared to minimization problems. Several recent papers have tried to analyze the convergence rate of adaptive extragradient in the convex-concave minimax settings. The universal mirror-prox method \cite{bach2019universal} proposed a new adaptive learning rate technique that adapts to problem parameters, such as the unknown Lipschitz parameter, and achieves optimal convergence rates in stochastic setting. \cite{babanezhad2020geometry} extended the universal mirror-prox \cite{bach2019universal} by replacing the norm dependence in the learning rate with a general Bregman divergence dependence. \cite{ene2020adaptive} proposed an adaptive stochastic single-call extragradient algorithm for variational inequality problems. \cite{antonakopoulos2021adaptive} proposed a similar adaptive mirror-prox algorithm, but their method handles an unbounded domain by introducing the notion of local norms in the deterministic setting. In addition to the adaptive extragradient methods mentioned above for the general stochastic minimax problem, \cite{yan2020adaptive} proposed an adaptive primal-dual method for expectation-constrained convex stochastic programs, which can be formulated as a minimax optimization with the coupled term being a linear function with dual variable. Training of a GAN model \cite{goodfellow2014generative} corresponds to solving a specific non-convex non-concave minimax problem. Several works have heuristically adopted a stochastic adaptive extragradient for training GANs \cite{gidel2018a,mertikopoulos2018optimistic,beznosikov2021distributed}. Recently, \cite{liu2020towards} studied the convergence behavior of an adaptive optimistic stochastic gradient algorithm for a class of non-convex non-concave minimax problems under the MVI condition to train GANs.  \\
  
\textbf{Distributed minimax algorithms.}\ 
As datasets and deep learning architectures become larger and larger distributed minimax algorithms are needed for GANs and adversarial training. \cite{beznosikov2021distributed} established upper and lower bounds for iteration complexity for  strongly-convex-strongly-concave and convex-concave minimax problems in both a centralized and decentralized setting. However, the convergence rate for their Extra Step Local SGD is established only in a strongly-convex-strongly-concave setting with a linear speed-up property with respect to the number of works; while for their proposed local Adam no convergence results are provided.  \cite{deng2020local} provided convergence guarantees for a primal-dual local stochastic gradient algorithm in the strongly-convex-strongly-concave-setting and several non-convex settings with PL-inequality-type conditions.  \cite{chen2020distributed} and \cite{mingruiliu2020decentralized} studied the convergence of a distributed optimistic stochastic gradient algorithm for non-convex non-concave minimax problems under the pseudomonotonicity condition and MVI condition, respectively. However, their convergence rates hold only for a sufficiently large minibatch size or a sufficiently large number of workers. In addition, there also exist several decentralized or federated algorithms for stochastic strongly-convex-strongly-concave minimax problems \cite{hou2021efficient, rogozin2021decentralized}. In this work, we mainly focus on the centralized setting for the stochastic convex-concave minimax problems.

\section{Appendix to Main Text}\label{appen:bouded-variance}

\subsection{Extension to Unbounded Stochastic Gradient Oracle}
\label{app:unbdsg}
Let $\{Z_i\}_{i=1}^n$ be a sequence of i.i.d.~standard normals. We have the following well-known results (see Appendix A of \cite{Chatterjee2014}):
\begin{align*}
& \P\big(\max_i Z_i > \E[\max_i Z_i]  + t \big) \leq \exp(-t\sq / 2) \text{ for all } t > 0,
\\
& \E[\max_i\vert Z_i\vert] \leq \sqrt{2\log(2n)}.
\end{align*}
With this, we have $\P(\max_i \vert Z_i\vert \geq \sqrt{2\log(2n)}+t)\leq 2\exp(-t\sq/2)$. 
We apply this result to the sequence $\big\{\|G(\zmt) - \tilde{G}(\zmt) \|_*/\|G\|_\infty, \|G(\tzmtmst) - \tilde{G}(\tzmtmst) \|_*/\|G\|_\infty \big\}_{m,t}$, which is a sequence of $2MT$ i.i.d.~standard normals by the homogeneity of the oracle.

\section{Proof of Theorems}\label{appen:proof}

  \begin{lemma} \label{lm:bdimprovement}
      For all $m\in[M]$, consider the sequence $\{\etamt, \tzmtmst, \zmt,\tzmt\}_{t=1}^T$ defined in \cref{algo}. It holds
      \$
          \|\tzmtmst - \zmt \| / \etamt \leq G,\quad
          \|\tzmt - \zmt \| / \etamt \leq G.
      \$
  \end{lemma}
  \begin{proof}[Proof of \cref{lm:bdimprovement}]
  Let $I:\cZ \to\cZ^*$ be the identity map which maps an element $z\in\cZ$ to the corresponding element in the dual space $\cZ^*$ (we are considering Euclidean case). The first-order optimality condition of the update rule $\zmt = \Pi_\cZ[\tzmtmst - \etamt \Mmt]$ is 
  \$
  \la \etamt \Mmt + I(\zmt - \tzmtmst), z - \zmt \ra \geq 0, \forall z\in \cZ.
  \$
  Set $z = \tzmtmst $, apply the Cauchy-Schwartz inequality and we obtain 
  \$
  \etamt \|\Mmt\|_* \cdot \| \tzmtmst - \zmt\| & \geq \la \etamt \Mmt, \tzmtmst - \zmt \ra 
  \\
  &\geq \la I(\tzmtmst - \zmt), \tzmtmst - \zmt \ra= \| \tzmtmst - \zmt \|\sq.
  \$
      The second inequality holds due to similar reasoning. We conclude the proof of Lemma~\ref{lm:bdimprovement}.
  \end{proof}
  
  \begin{lemma}[One-step analysis] \label{lm:onestep}
      For all $m\in[M]$, consider the sequence $\{\etamt, \tzmtmst,\Mmt =\tilde{G}(\tzmtmst), \zmt,\gmt = \tilde{G}(\tzmt),\tzmt\}_{t=1}^T$ defined in Algorithm~\ref{algo}. It holds for all $z\in \cZ$, 
      \$
      \langle \zmt - z, \gmt \rangle  \leq \frac{1}{\etamt} \Big(\hnormsq{z-\tzmtmst} - \hnormsq{z-\tzmt}\Big) 
      - \frac{1}{\etamt}\Big(\hnormsq{\zmt-\tzmtmst} + \hnormsq{\zmt-\tzmt}\Big)
      \\ + \|\gmt-\Mmt \|_* \cdot \|\zmt -\tzmt\|.
      \$    
  \end{lemma}
  
  \begin{proof}[Proof of Lemma~\ref{lm:onestep}]
      For any $c,g\in\cZ$, consider the update of the form $a^* = \Pi_{\cZ}[c - g]=\operatorname*{argmin}_{z\in \cZ} \,\la g, z\ra + \hnormsq{z - c}$. It holds for all $b\in\cZ$,
      \$
      \langle g, a^* - b \rangle \leq \hnormsq{b-c} - \hnormsq{b-a^*} - \hnormsq{a^* - c}.
      \$
      By the update rule of $\zmt$ and $\tzmt$, we have (taking $a^* \leftrightarrow \zmt$, $b \leftrightarrow \tzmt$, $g\leftrightarrow \etamt \Mmt$, $c \leftrightarrow \tzmtmst$)
      \#
      \la \etamt \Mmt, \zmt - \tzmt\ra \leq \hnormsq{\tzmt -\tzmtmst} - \hnormsq{\tzmt - \zmt} - \hnormsq{\zmt - \tzmtmst} , 
      \label{eq:eq1}
      \#
      and for all $z\in\cZ$ (taking $a^* \leftrightarrow \tzmt$, $b \leftrightarrow z$, $g\leftrightarrow \etamt\gmt$, $c \leftrightarrow \tzmtmst$)
      \#
          \la \etamt \gmt, \tzmt - z\ra \leq \hnormsq{z-\tzmtmst} - \hnormsq{z- \tzmt} - \hnormsq{\tzmt - \tzmtmst}.
          \label{eq:eq2}
      \#
      Finally we apply the Cauchy-Schwarz inequality and plug in Eqs.~\eqref{eq:eq1} and \eqref{eq:eq2}.
      \$
      \la \gmt, \zmt - z\ra 
      & = \la \gmt , \zmt - \tzmt\ra + \la \gmt, \tzmt - z\ra 
      \\
      & = \la \gmt - \Mmt, \zmt - \tzmt\ra + \la \gmt, \tzmt - z\ra + \la \Mmt, \zmt - \tzmt\ra
      \\
      & \leq  \| \gmt - \Mmt\|_* \cdot \| \zmt - \tzmt\| + \la \gmt, \tzmt - z\ra + \la \Mmt, \zmt - \tzmt\ra
      \\
      & \leq  \| \gmt - \Mmt\|_* \cdot \| \zmt - \tzmt\| \
      \\
      & \quad\quad
          + \ooetamt \Big({\hnormsq{\tzmt -\tzmtmst}} - \hnormsq{\tzmt - \zmt} - \hnormsq{\zmt - \tzmtmst}  \Big )
      \\
      & \quad\quad
      + \ooetamt \Big( \hnormsq{z-\tzmtmst} - \hnormsq{z- \tzmt} - {\hnormsq{\tzmt - \tzmtmst}} \Big ) 
      \\
      & = \frac{1}{\etamt} \Big(\hnormsq{z-\tzmtmst} - \hnormsq{z-\tzmt}\Big) 
      - \frac{1}{\etamt}\Big(\hnormsq{\zmt-\tzmtmst} + \hnormsq{\zmt-\tzmt}\Big)
      \\& \quad\quad + \|\gmt-\Mmt \|_* \cdot \|\zmt -\tzmt\|.
      \$
      This finishes the proof of Lemma~\ref{lm:onestep}.
  \end{proof}
  
  \subsection{Proof of Theorem~\ref{thm:nonsmooth}}
  \label{sec:proof:thmnonsmooth}
  \begin{proof}[Proof of Theorem~\ref{thm:nonsmooth}, Non-smooth Case] 
  The proof strategy follows closely that of Bach and Levy \cite{bach2019universal}.
  \textbf{Step 1.} We apply the Lemma~\ref{lm:onestep} and sum over all $m\in[M]$ and $t\in[T]$. Define 
  \$\ximt \defeq G(\zmt) - \gmt = G(\zmt) - \tilde{G}(\zmt).\$ 
  For all $z\in \cZ$,
  \#
  \sumT\sumM \big\langle \zmt - z, G(\zmt) \big\rangle 
  &= \sumT\sumM \big\langle  \zmt - z, \ximt \rangle+ \sumT\sumM \langle \zmt - z, \gmt \rangle
  \\ 
  &\leq \underbracket{\sumT\sumM \big\langle  \zmt - z, \ximt \big\rangle}_{{I}(z)} 
  \label{eq:defI}
  \\
  &\quad + \underbracket{\sumT\sumM \frac{1}{\etamt}\Big(\hnormsq{z-\tzmtmst} - \hnormsq{z-\tzmt}\Big) }_{{II}(z)}
  \label{eq:defII}
  \\ 
  & \quad \underbracket{ - \sumT\sumM  \frac{1}{\etamt}\Big(\hnormsq{\zmt-\tzmtmst} + \hnormsq{\zmt-\tzmt}\Big)}_{{III}} 
  \label{eq:defIII}
  \\
  & \quad + \underbracket{\sumT\sumM \|\gmt-\Mmt \|_* \cdot \|\zmt -\tzmt\|}_{{IV}}. \label{eq:defIV}
  \#
  Now we use Lemma~\ref{lm:gaptoregret} and obtain
  \#
  TM\cdot \E[\dualgap(\bar z)] 
  &\leq \E[\sup_{z\in\cZ} \{ {I}(z)+{II}(z) + {III}+{IV} \}]
  \\
  &\leq \E[\sup_{z\in\cZ} I(z)] + \E[\sup_{z\in\cZ} II(z)] + \E[III]+\E[IV]
  \label{eq:boundnonsmooth}
  \#
  Next we upper bound each term in turns. Steps 2--5 rely heavily on the learning rate scheme. Define 
  \$(Z^m_t)\sq  \defeq \frac{\| \zmt - \tzmtmst \|\sq + \|\zmt - \tzmt \|\sq}{{5}( \etamt)\sq} \$ 
  for all $t\in[T]$ and $m\in[M]$. By Lemma~\ref{lm:bdimprovement} we know $Z^m_t\leq G$ almost surely. This is due to 
  \$
  \| \zmt - \tzmtmst \|\sq + \|\zmt - \tzmt \|\sq 
  \leq \| \zmt - \tzmtmst \|\sq + 2 \|\zmt - \tzmtmst \|\sq + 2 \| \tzmtmst - \tzmt\|\sq \leq 5 G\sq \etamtsq.
  \$

  Moreover, for the nonsmooth case ($\alpha = 1$), $\etamt$ can be expressed by
  \# \label{eq:etatoZ}
  \etamt = \frac{D}{\sqrt{G_0\sq + \sum_{\tau=1}^{t-1} (Z^m_\tau)\sq}}.
  \#
  
  \textbf{Step 2.} Show $ \E[\sup_{z\in\cZ} I(z)] =  O(\sigma D\sqrt{MT})$. For all $z\in\cZ$,
  \$
  I(z) = \sumT\sumM \langle  \zmt - \tilde{z}^m_0, \ximt \rangle + \sumT\sumM \langle   \tilde{z}^m_0 - z, \ximt \rangle.
  \$
  The first term is a martingale difference sequence (MDS) and is zero in expectation. For the second term, we use the Cauchy–Schwarz inequality. For all $z\in\cZ$,
  \$
  \E\bigg[\sup_z \sumT\sumM \langle   \tilde{z}^m_0 - z, \ximt \rangle\bigg]
  & = \E\bigg[\sup_z\Big\langle    \tilde{z}_0 - z, \sumT\sumM \ximt \Big\rangle\bigg]
  \\
  & \leq \E\bigg[ \sup_z \|\tilde{z}_0 - z\| \cdot \bigg\| \sumT\sumM \ximt \bigg\|_* \bigg]
  \\
  & \leq  D \cdot \sqrt{\E\bigg[ \bigg\| \sumT\sumM   \ximt \bigg\|_*^2 
  \bigg]} \leq \sigma D\sqrt{MT}.
  \$
  In the last equality, we use the fact that $\{\ximt\}$ is an MDS. This establishes $\E[\sup_z I(z)] \leq \sigma D \sqrt{MT}$.
  
  \textbf{Step 3.} Show $\E[\sup_{z\in\cZ} II(z)] = O( DG\cdot M\sqrt{T})$. For all $z\in\cZ$,
  \tiny
  \$
  II(z) 
  & =\sumT\sumM \ooetamt \Big(\hnormsq{z-\tzmtmst} - \hnormsq{z-\tzmt}\Big)
  \\
  & =\sumM\sumTnotinSp \ooetamt \Big(\hnormsq{z-\tzmtmst} - \hnormsq{z-\tzmt}\Big) + \sumM\sumTinSp \ooetamt \Big(\hnormsq{z-\tzmtmst} - \hnormsq{z-\tzmt}\Big) 
  \\
  & =\sumM\sumTnotinSp \ooetamt \Big(\hnormsq{z-\textcolor{red}{\tzmtm}} - \hnormsq{z-\tzmt} \Big) + \sumM\sumTinSp \ooetamt \Big(\hnormsq{z-\textcolor{red}{\tzmtmo}} - \hnormsq{z-\tzmt}\Big) 
  \\
  &= \underbracket{\sumM\sumT \ooetamt \Big(\hnormsq{z-\tzmtm} - \hnormsq{z-\tzmt}\Big) }_{A} + 
  \underbracket{\sumM\sumTinSp \ooetamt \Big(\hnormsq{z-\tzmtmo} - \hnormsq{z-\tzmtm}\Big)}_{B}
  \$ \normalsize
  where we used the definition of $\tzmtmst$ for two cases $t\in S+1$ and $t \notin S+1$ (Line~\ref{line:sync} and \ref{line:nosync} in algorithm).
  
  We upper bound $A$ and show $B\leq 0$.
  
  Recall for $t\in S+1$, we have $\tzmtmst = \tzmtmo=\sumM w_m \cdot \tzmtm$, and for $t\notin S+1$, we have $\tzmtmst = \tzmtm$. For the first term $A$ we use $\hnormsq{z - \tzmt} \leq D\sq $ and then telescope.
  \$
  A 
  & = \sumM \bigg[ \frac{1}{\eta^m_1} \Big(\hnormsq{\tilde{z}^m_0 - z }\Big)  - \frac{1}{\eta^m_T} \Big(\hnormsq{\tilde{z}^m_T - z}\Big)+ \sum_{t=2}^T \Big(\frac{1}{\etamt} - \frac{1}{\eta^m_{t-1}}\Big) \Big(\hnormsq{\tzmtm - z}\Big) \bigg]
  \\
  & \leq \sumM \bigg[ \frac{D\sq}{\eta^m_1}  + \sum_{t=2}^T \Big(\frac{1}{\etamt} - \frac{1}{\eta^m_{t-1}}\Big) D\sq  \bigg]
  \\
  & \leq  \sumM \bigg[ \frac{D\sq}{\eta^m_1}    + \frac{D\sq}{\eta^m_T}\bigg]
  \$
  For each $m$, we have $D\sq / \eta^m_1 = DG_0$. For $D\sq /\eta^m_T$ we use the learning rate scheme. Recall the definition of $\Zmt$.
  Then
  \$
  \frac{D\sq}{\eta^m_T} 
  = D \sqrt{G_0\sq + \sum_{t=1}^{T-1}(Z^m_t)\sq}
  \leq D\sqrt{G_0 + G\sq T}
  \leq DG_0 + DG\sqrt{T}.
  \$
  This implies $A \leq M(2DG_0 + DG\sqrt{T})= O( DG\cdot M \sqrt{T})$.
  
  For the term $B$, we use the definition of $\tzmtmo$ and the weights $\{w_m\}$ to show $B\leq 0$. For each $t$, since $\tzmtmo$ the same for all workers,
  \$
  \sumM \ooetamt \Big(\hnormsq{z-\tzmtmo}\Big)
  & =
  \Big(\sumM \ooetamt\Big) \Big(\hnormsq{z-\tzmtmo} \Big)
  \\
  &= \Big(\sumM \ooetamt\Big) \Big( \tfrac12 \big\| {\textstyle \sum_{m=1}^M {w_m}^{1/2} \cdot {w_m}^{1/2}(z-\tzmtm ) } \big\|\sq \Big)
  \\
  & \leq  \Big(\sumM \ooetamt\Big) \Big(\sumM w_m\Big) \Big(\sumM w_m\cdot \tfrac12 \| z - \tzmtm \|^2\Big)
  \\
  & = \sumM  \ooetamt \Big(\hnormsq{z - \tzmtm}\Big) .
  \$
  In the last equality we use $\sumM w_m = 1$ and $(\sumM 1/\etamt)w_{m'} = 1/\eta^{m'}_t$ for all $m'\in[M]$. This implies $B\leq 0$. This establishes $\E[\sup_z II(z)]\leq\E[\sup_z A]=O(DG\cdot M\sqrt{T})$.

  \textbf{Step 4.} Show $\E[III] \leq 0$. This is obviously true.
  
  \textbf{Step 5.} Show $\E[IV]=\tilde{O}(\gamma DG\cdot M\sqrt{T})$. Define $\gamma = G/G_0$. By \ref{as:bdsg} we have $\| \gmt - \Mmt\|_* \leq 2G$. It holds almost surely that
  \$
  IV 
  & \leq 2G\sumM \sumT \| \zmt - \tzmt\|
  \\
  & \leq 2G\sqrt{T} \cdot \sumM \sqrt{\sumT \| \zmt - \tzmt\|\sq }
  \\
  & \leq 2G\sqrt{T} \cdot \sumM \sqrt{\sumT (\etamt \Zmt)\sq }
  \\
  & =  2G \sqrt{T} \cdot D \cdot \sumM \sqrt{\sumT \frac{(\Zmt)\sq }{G_0\sq + \sumtautotm (Z^m_{\tau})\sq } }
  \\
  & \leq 2G D\sqrt{T} \cdot  \sumM \sqrt{2+ 4\gamma\sq + 2\log \Big(\frac{G_0\sq + \sumttoTm \Zmtsq }{G_0\sq} \Big) } 
  \tag{Lemma~\ref{lm:boundwithlog}}
  \\
  & \leq 2G D\sqrt{T} \cdot  \sumM \sqrt{2+ 4\gamma\sq + 2\log \Big(\frac{G_0\sq + G\sq T }{G_0\sq} \Big) } 
  \\
  & \leq 2G D \sqrt{T} \cdot \sumM \sqrt{2 + 4\gamma\sq + 2 \log (1 + \gamma\sq T)} 
  \\
  & = \tilde{O}(\gamma GD\cdot M\sqrt{T})
  \$
  Finally, we plug in the upper bounds for $I$--$IV$ and continue Eq~\eqref{eq:boundnonsmooth}.
  \$
  TM\cdot \E[\dualgap(\bar z)] = \tilde{O}(\gamma DG\cdot M\sqrt{T} + \sigma D \sqrt{MT}).
  \$
  This finishes the proof of Theorem~\ref{thm:nonsmooth}
  \end{proof}

  \subsection{Proof of Theorem~\ref{thm:smooth}}
  \label{sec:proof:thmsmooth}
  \begin{proof}[Proof of Theorem~\ref{thm:smooth}, Smooth Case]
  The proof strategy follows closely that of Bach and Levy \cite{bach2019universal}.
  Using the notation for Step~1 in the proof for the nonsmooth case, we have the bound
  \$
  TM \E[\dualgap(\bar z)] \leq \E[\sup_z\{ I(z) + II(z) + III + IV\}],
  \$
  where $I$--$IV$ are defined in Eqs.~\eqref{eq:defI}--\eqref{eq:defIV}. We deal with these terms in a different manner. 
  
  For the term $I(z)$ in Eq.~\eqref{eq:defI}, following Step~2 we have $\E[\sup_z I(z)] = O(\gamma \sigma D\sqrt{MT})$.
  
  Next, we define a stopping time. For each $m\in[M]$, let
  \#
  \tau^*_m \defeq \max \bigg\{ t\in[T]: \ooetamt \leq 1/(2L)\bigg\}.
  \#
  Recall our learning rate scheme for the smooth case
  \$
  \eta^m_1 = \frac{D\alpha}{G_0}, \quad \etamt = \frac{D\alpha}{\sqrt{G_0\sq + \sum_{\tau=1}^{t-1} (Z^m_\tau)\sq}}.
  \$
  
  For the term $II(z)$ in Eq.~\eqref{eq:defII}, we follow Step~3 and obtain for all $z\in\cZ$,
  \$
  II(z) \leq
   \sumM   \bigg( \frac{D\sq}{\eta^m_1}    + \frac{D\sq}{\eta^m_T} \bigg).
  \$
  By the definition of $\eta^m_1$, we have $\sumM D \sq / \eta^{m}_1 \leq DMG_0/\alpha$. For the second term, for fixed $m\in [M]$,
  \#
  \sumM D \sq  / \etamT 
  & = \sumM \frac{D}{\alpha}\sqrt{G_0\sq+  \sumttoTm \Zmtsq}
  \\
  & \leq \sumM \frac{D}{\alpha} \Bigg(G_0 + \sumT \frac{\Zmtsq}{\sqrt{G_0\sq + \sumtautotm \Zmtausq}} \Bigg) \tag{Lemma~\ref{lm:boundwithsqrt}}
  \\
  & = \frac{ MDG_0}{\alpha} + \underbracket{\sumM \sumT \frac{1}{\alpha \sq} \etamt \Zmtsq}_{\defeq \cA}
  \label{eq:defcA:nonsmooth}
  \# 
  So we have $\E[\sup_z II(z)] \leq 2\gamma MDG/\alpha + \E[\cA]$.
  
  For the term $III$ in Eq.~\eqref{eq:defIII}, we also split it into two parts by $\tau^*_m$.
  \#
  III & \defeq  - \sumT\sumM  \frac{1}{\etamt}\Big(\hnormsq{\zmt-\tzmtmst} + \hnormsq{\zmt-\tzmt}\Big)
  \\
  & = - \sumT\sumM \frac52 \etamt \Zmtsq
  \\
  & = - \underbracket{\sumM\sumttotaust \frac52 \etamt \Zmtsq}_{\geq 0 }
  - 
  \underbracket{\sumM \sumtfromtaustmtoT  \frac52 \etamt \Zmtsq}_{\defeq \cB_\tail}
  \label{eq:defBtail}
  \#
  For the term $IV$ in defined in Eq.~\eqref{eq:defIV}, we first introduce a margtingale difference sequence. For all $t\in[T],m\in[M]$, let
  \#
  \zeta^m_t \defeq \big(\gmt - G(\zmt) \big) + \big(\Mmt - G(\tzmtmst) \big).
  \#
  By the triangular inequality, we have 
  \#
  IV & \defeq \sumT\sumM \|\gmt-\Mmt \|_* \cdot \|\zmt -\tzmt\|
  \\
  & \leq  \underbracket{
  \sumT\sumM \|\zeta^m_t \|_* \cdot \|\zmt -\tzmt\|}_{\defeq V} + \sumT\sumM \| G(\zmt) - G(\tzmtmst) \|_* \cdot \|\zmt -\tzmt\|
  \label{eq:defV}
  \\
  & \leq  V + \sumT\sumM \Big( \frac{L}{2} \| \zmt - \tzmtmst \|\sq + \frac{L}{2} \|\zmt -\tzmt\|\sq \Big) 
  \label{eq:smoothnesskickin}
  \\
  & = V + \sumT\sumM \frac{5L}{2}\etamtsq \Zmtsq
  \\
  & = V + 
  \underbracket{\sumM \sumttotaust \frac{5L}{2}\etamtsq \Zmtsq}_{\defeq \cC_\head}
  + 
  \underbracket{\sumM \sumtfromtaustmtoT \frac{5L}{2}\etamtsq \Zmtsq}_{\defeq \cC_\tail} \label{eq:defcC}
  \#
  Eq.~\eqref{eq:smoothnesskickin} holds due to smoothness, i.e., for all $z,z'\in\cZ$, $\| G(z) - G(z')\|_*\leq L \| z- z'\|$. Using smoothness, we can verify Eq.~\eqref{eq:smoothnesskickin} as follows.
  \$
  & \| G(\zmt) - G(\tzmtmst) \|_* \cdot \|\zmt -\tzmt\|
  \\
  &\leq \frac{1}{2L} \| G(\zmt) - G(\tzmtmst) \|_*\sq + \frac{L}{2}\|\zmt -\tzmt\|\sq 
  \\
  & \leq \frac{L}{2} \| \zmt - \tzmtmst \|\sq + \frac{L}{2} \|\zmt -\tzmt\|\sq.
  \$
  To summarize, we have shown 
  \#
  TM \cdot \E[\dualgap(\bar z)] 
  & \leq \E[\sup_z\{ I(z) + II(z) + III + IV\}]
  \\
  & \leq O \Big(\gamma \sigma D \sqrt{MT}\Big) + {2\gamma MDG}/{\alpha} 
  \\
  & \quad + \E [\cA + \cC_\head + (- \cB_\tail + \cC_\tail) + V]  .
  \label{eq:smoothdecomp_1}
  \#
  
  \textbf{Step a.} Show $\E[\cA] \leq 8\gamma G D M / \alpha + 3DM\cVoneT/\alpha$. Recall its definition in Eq.~\eqref{eq:defcA:nonsmooth}.
  \$
  \cA 
  & \defeq \sumM\sumT \frac{1}{\alpha\sq} \etamt \Zmtsq
  \\
  & =\frac{D}{\alpha} \sumM  \sumT \frac{\Zmtsq}{\sqrt{G_0\sq + \sumtautotm \Zmtausq}}
  \\
  & \leq  \frac{D}{\alpha} \sumM \Bigg( 5\gamma G + 3\sqrt{G_0\sq + \sumttoTm  \Zmtsq}  \Bigg) \tag{Lemma~\ref{lm:boundwithsqrt}}
  \\
  & \leq  \frac{D}{\alpha} \sumM \Bigg( 8\gamma G + 3\sqrt{\sumttoTm  \Zmtsq}  \Bigg) 
  \$
  Note by Lemma~\ref{lm:bdimprovement} we know $\Zmtsq \leq (\|\gmt\|_*\sq + \|\Mmt\|_*\sq)/5\leq \|\gmt\|_*\sq + \|\Mmt\|_*\sq$. Recall the definition of $\cVmT$ in Eq.~\eqref{eq:defcVmT}. By the symmetry of the algorithm over all workers, we know $\cVoneT = \cVmT$ for all $m\in[M]$. Then 
  \$
  \E[\cA]& \leq 8 \gamma DMG/\alpha + \frac{3D}{\alpha} \sumM \E\Bigg[ \sqrt{\sumttoTm  \Zmtsq}\Bigg]
  \\
  & \leq 8 \gamma DMG/\alpha + \frac{3D}{\alpha} \sumM \E\Bigg[ \sqrt{\sumttoTm  \|\gmt\|_*\sq + \|\Mmt\|_*\sq }\Bigg]
  \\
  & = 8 \gamma DMG/\alpha + \frac{3D}{\alpha} \sumM \cVmT 
  = 8 \gamma DMG/\alpha + 3DM\cVoneT / \alpha.
  \$
  By our choice of $\alpha$ we have $\E[\cA]=O(\gamma D M^{3/2}G + DM^{3/2}\cVoneT) $.
  
  \textbf{Step b.} Show $\E[\cC_\head]=O(1)$. Recall its definition in Eq.~\eqref{eq:defcC}.
  \#
  \cC_\head
  & \defeq \sumM \sumttotaust \frac{5L}{2}\etamtsq \Zmtsq
  \\
  & = \frac{5\alpha\sq D\sq L}{2} \sumM \sumttotaust \frac{\Zmtsq}{G_0\sq + \sumtautotm \Zmtausq}
  \\
  &  \leq \frac{5\alpha\sq D\sq L}{2} \sumM  \bigg( 6\gamma\sq + 2\log \Big( \frac{G_0\sq + \sumttotaustm \Zmtausq}{G_0\sq}\Big) \bigg) \tag{Lemma~\ref{lm:boundwithlog}}
  \\
  & = \frac{5\alpha\sq D\sq L}{2} \sumM  \bigg( 6\gamma\sq + 2\log\Big( \frac{\alpha\sq D\sq}{G_0\sq (\eta^m_{\tau^*_m})\sq}\Big) \bigg)
  \\
  & \leq \frac{5\alpha\sq D\sq LM}{2}  \bigg (6\gamma\sq + 4\log\Big(\frac{\alpha D}{2G_0 L}\Big) \bigg) \label{eq:endofcChead}
  \#
  The last inequality is due to the definition of $\tau_m^*$.
  By our choice of $\alpha$ we have $\E[\cC_\head] = \tO(\gamma\sq LD\sq )$.
  
  \textbf{Step c.} Show $\cC_\tail- \cB_\tail \leq 0$. Recall $\cB_\tail$ is defined in Eq.~\eqref{eq:defBtail}. By definition,
  \$
  \cC_\tail- \cB_\tail 
  =
  \sumM \sumtfromtaustmtoT \Big(  \frac{5L}{2}\etamt - \frac52 \Big) \etamt \Zmtsq .
  \$
  We show $ \frac{5L}{2}\etamt - \frac52  \leq 0$ for all $t\in[T],m\in[M]$. Note that for all $t\geq \tau^*_m + 1$ we have $\etamt \leq 1/(2L)$. And so $\frac{5L}{2}\etamt - \frac52 \leq (5/4) - (5/2) = -5/4$. Summarizing, we have shown $\cC_\tail- \cB_\tail \leq 0$.
  
  \textbf{Step d.} Show $\E[V] = \tO(\gamma \sigma D \sqrt{MT})$. Recall its definition in Eq.\eqref{eq:defV}. Also note $\E[\| \zeta^m_t\|_*\sq ] \leq 4\sigma\sq$.
  \#
  \E[V] 
  & \defeq \E\Bigg[ \sumT\sumM \|\zeta^m_t \|_* \cdot \|\zmt -\tzmt\| \Bigg]
  \\
  & \leq  \E\Bigg[ \sqrt{  \sumT\sumM \|\zeta^m_t \|_* \sq } \Bigg ] \cdot \E\Bigg[  \sqrt{ \sumT\sumM \|\zmt -\tzmt\| \sq } \Bigg] 
  \\
  &  \leq \sqrt{  \sumT\sumM \E \big [\|\zeta^m_t \|_* \sq  \big]} \cdot \E\Bigg[  \sqrt{ \sumT\sumM \|\zmt -\tzmt\| \sq } \Bigg] 
  \\
  & \leq 2 \sigma \sqrt{MT} \cdot  \E\Bigg[  \sqrt{\sumM\sumT \| \zmt - \tzmt \|\sq}\Bigg ] 
  \\
  & \leq 2 \sigma \sqrt{MT} \cdot  \E\Bigg[  \sqrt{\sumM\sumT \| \zmt - \tzmtmst \|\sq+\| \zmt - \tzmt \|\sq}\Bigg ]
  \\
  & = 2 \sigma \sqrt{MT} \cdot \E \Bigg[ \sqrt{\sumM \sumT 5 \cdot \etamtsq \Zmtsq }\Bigg ]
  \\
  & =  2\sqrt{5} \cdot \sigma \sqrt{MT}  \cdot D\alpha \cdot \E\Bigg[  \sqrt{\sumM \sumT  \frac{\Zmtsq }{G_0\sq + \sumtautotm \Zmtausq}}\Bigg ]
  \\
  & \leq 6  \cdot \sigma \sqrt{MT}  \cdot D\alpha  \cdot \E\Bigg[  \sqrt{\sumM \bigg(6\gamma\sq + 2\log \Big(\frac{G_0\sq + \sumttoTm \Zmtsq }{G_0\sq} \Big) \bigg)}\Bigg ] \tag{Lemma~\ref{lm:boundwithlog}}
  \\
  &  \leq 6\sigma \sqrt{MT}  \cdot D\alpha \cdot  \sqrt{ M (6\gamma\sq + 2 \log(1 + \gamma\sq T) )} .
  \label{eq:boundVend}
  \#
  By our choice of $\alpha$, we have $\E[V] = \tO( \gamma \sigma D \sqrt{MT})$.
  
  Continuing Eq.~\eqref{eq:smoothdecomp_1}, we have 
  \$
  & TM\cdot \E[\dualgap(\bar z )] 
  \\
  & \leq O \Big(\gamma \sigma D \sqrt{MT}\Big) + {2\gamma MDG}/{\alpha} 
  + \E [\cA + \cC_\head + (- \cB_\tail + \cC_\tail) + V] 
  \\
  & = \tO \Big(
      \gamma \sigma D \sqrt{MT}
  + \underbracket{\gamma D M^{3/2}G + DM^{3/2}\cVmT}_{\cA}
  + \underbracket{\gamma \sq LD\sq }_{\cC_\head} 
  + \underbracket{\gamma \sigma D \sqrt{M T}}_{V}
  \Big).
  \$
  
  This finishes the proof of Theorem~\ref{thm:smooth}.
      
  \end{proof}
  
  \begin{remark}[Getting rid of $\cVoneT$]
  We could also use the free parameters $\alpha$ (base learning rate) and obtain the following near linear speed-up result.
  
      \begin{theorem}[Smooth Case, free of $\cVoneT$] 
          \label{thm:smooth_noV}
          
          Assume \ref{as:bddomain}, \ref{as:bdsg}, \ref{as:bdvar} and \ref{as:smooth}. Let $\sigma, D, G,L$ be defined therein. For any $\epsilon \in (0,\frac12)$, let $\bar z = \ALGO(G_0, D;K,M,R;T^\epsilon/\sqrt{M})$. If $T\geq M^{1/(2\epsilon)}$, then
      \[\E[\operatorname*{DualGap}(\bar z)] = 
      \tilde{O}\bigg( \frac{\sigma D}{{\sqrt{MT^{1-2\epsilon}}}}
      + \frac{\gamma\sq LD\sq}{T^{1-2\epsilon}}
      + \frac{LD\sq M }{T}
      + \frac{\gamma G D M^{3/2}}{T^{1+\epsilon}}
      \bigg) \,,\]
      where $\tilde{O}$ hides absolute constants, logarithmic factors of problem parameters and logarithmic factors of $T$.
      \end{theorem}
  
      \begin{proof}[Proof of Theorem~\ref{thm:smooth_noV}]
          We decompose the term $II$ in Eq.\eqref{eq:defII} in a different way. Recall in Step~3 we have shown for all $z\in\cZ$,
          $
          II(z) \leq
           \sumM   \frac{D\sq}{\eta^m_1}    + \frac{D\sq}{\eta^m_T} .
          $
          For the second term, for fixed $m\in [M]$,
          \#
          \sumM D \sq  / \etamT 
          & = \sumM \frac{D}{\alpha}\sqrt{G_0\sq+  \sumttoTm \Zmtsq}
          \\
          & \leq \sumM \frac{D}{\alpha} \Bigg(G_0 + \sumT \frac{\Zmtsq}{\sqrt{G_0\sq + \sumtautotm \Zmtausq}} \Bigg) \tag{Lemma~\ref{lm:boundwithsqrt}}
          \\
          & = \frac{ MDG_0}{\alpha} + \sumM \sumT \frac{1}{\alpha \sq} \etamtsq \Zmtsq
          \\
          & \leq \frac{\gamma MDG}{\alpha} + 
          \underbracket{\sumM\sumttotaust \frac{1}{\alpha\sq} \etamtsq \Zmtsq}_{\defeq\cA_\head}
          + 
          \underbracket{\sumM \sumtfromtaustmtoT \frac{1}{\alpha \sq} \etamtsq \Zmtsq}_{\defeq\cA_\tail} \label{eq:defcA}
          \# 
          So we have $\E[\sup_z II(z)] \leq 2\gamma MDG/\alpha + \E[\cA_\head + \cA_\tail]$.
          Then, following the proof in the smooth case, we have 
          
          \#
  TM \cdot \E[\dualgap(\bar z)] 
  & \leq \E[\sup_z\{ I(z) + II(z) + III + IV\}]
  \\
  & \leq O \Big(\gamma \sigma D \sqrt{MT}\Big) + {2\gamma MDG}/{\alpha} 
  \\
  & \quad + \E [\cA_\head + \cC_\head + (\cA_\tail - \cB_\tail + \cC_\tail) + V]  .
  \label{eq:smoothdecomp_1}
  \#

  Recall our choice of $\alpha = T^\epsilon / \sqrt{M}$.

  Show $\E[\cA_\head]=\tO(1)$. Recall its definition in Eq.~\eqref{eq:defcA}.
  \$
  \cA_\head 
  & \defeq \sumM\sumttotaust \frac{1}{\alpha\sq} \etamtsq \Zmtsq
  \\
  & =\frac{D}{\alpha} \sumM  \sumttotaust \frac{\Zmtsq}{\sqrt{G_0\sq + \sumtautotm \Zmtausq}}
  \\
  & \leq  \frac{D}{\alpha} \sumM \Bigg( 5\gamma G + 3\sqrt{G_0\sq + \sumttotaustm \Zmtsq}  \Bigg) \tag{Lemma~\ref{lm:boundwithsqrt}}
  \\
  & = \frac{D}{\alpha} \sumM \Big ( 5\gamma G + \frac{3D\alpha }{\eta^m_{\tau^*_m}}\Big)
  \\
  & \leq \frac{D}{\alpha} \sumM \Big ( 5\gamma G + 6\alpha LD \Big) = \frac{5\gamma GDM}{\alpha} + 6LD\sq M.
  \$
  By our choice of $\alpha$ we have $\E[\cA_\head]\leq 5\gamma GDM^{3/2}T^{-\epsilon} + 6LD\sq M $.

  For $\cC_{\head}$ defined in Eq.~\eqref{eq:defcC}, following Eq~\eqref{eq:endofcChead}, we have $\E[\cC_\head] = \tO(\gamma\sq LD\sq T^{2\epsilon})$.

  Show $\cA_\tail + \cC_\tail- \cB_\tail \leq 0$. Recall $\cB_\tail$ is defined in Eq.~\eqref{eq:defBtail}. By definition,
  \$
  \cA_\tail + \cC_\tail- \cB_\tail 
  =
  \sumM \sumtfromtaustmtoT \Big( \frac{1}{\alpha\sq} + \frac{5L}{2}\etamt - \frac52 \Big) \etamt \Zmtsq .
  \$
  We show $\frac{1}{\alpha\sq} + \frac{5L}{2}\etamt - \frac52  \leq 0$ for all $t\in[T],m\in[M]$. Note that 
  \$T\geq M^{1/ (2\epsilon)}\implies \alpha\sq = (T^\epsilon / \sqrt{M})\sq \geq 1 ,\$ 
  and that for all $t\geq \tau^*_m + 1$ we have $\etamt \leq 1/(2L)$. And so $\frac{1}{\alpha\sq} + \frac{5L}{2}\etamt - \frac52 \leq 1 + (5/4) - (5/2) = -1/4$. Summarizing, we have shown $\cA_\tail + \cC_\tail- \cB_\tail \leq 0$.
  
  For $V$ defined in Eq.~\eqref{eq:defV}, following Eq.~\eqref{eq:boundVend}, $\E[V] = \tO( \gamma \sigma D \sqrt{MT^{1+2\epsilon}})$.
  
  Putting together we have 
  \$
  & TM\cdot \E[\dualgap(\bar z )] 
  \\
  & \leq O \Big(\gamma \sigma D \sqrt{MT}\Big) + {2\gamma MDG}/{\alpha} 
  + \E [\cA_\head + \cC_\head + (\cA_\tail - \cB_\tail + \cC_\tail) + V] 
  \\
  & = \tO \Big(
      \gamma \sigma D \sqrt{MT}
  + \underbracket{\gamma GDM^{3/2}T^{-\epsilon} + LD\sq M}_{\cA_\head}
  + \underbracket{\gamma \sq LD\sq T^{2\epsilon}}_{\cC_\head} 
  + \underbracket{\gamma \sigma D \sqrt{M T^{1+2\epsilon}}}_{V}
  \Big).
  \$
  
  This finishes the proof of Theorem~\ref{thm:smooth_noV}
      \end{proof}
  \end{remark}

\section{Helper Lemmas}\label{appen:lemmas}

  \begin{lemma} \label{lm:boundwithlog}
  For any non-negative real numbers $a_1,\dots, a_n \in [0,a]$, and $a_0>0$, it holds
  \$
  \sum_{i=1}^{n} \frac{a_{i}}{a_{0}+\sum_{j=1}^{i-1} a_{j}} \leq 2+\frac{4 a}{a_{0}}+2 \log \Big(1+\sum_{i=1}^{n-1} a_{i} / a_{0}\Big).
  \$
  \begin{proof}[Proof of Lemma~\ref{lm:boundwithlog}]
      See Lemma A.2 of \cite{bach2019universal}.
      
  \end{proof}
  \end{lemma}
  
  \begin{lemma} \label{lm:boundwithsqrt}
      For any non-negative numbers $a_1,\dots, a_n \in [0,a]$, and $a_0>0$, it holds
      \$
      \sqrt{a_0 + \sum_{i=1}^{n-1} a_i} - \sqrt{a_0} \leq \sum_{i=1}^n \frac{a_i}{\sqrt{a_0 + \sum_{j=1}^{i-1}}a_j} \leq \frac{2a}{a_0} + 3\sqrt{a} + 3\sqrt{a_0 + \sum_{i=1}^{n-1} a_i}.
      \$
      \begin{proof}[Proof of Lemma~\ref{lm:boundwithlog}]
          See Lemma A.1 of \cite{bach2019universal}.
          
      \end{proof}
  \end{lemma}

  \begin{lemma}9 \label{lm:gaptoregret}
      For any sequence $\{ z_t\}_{t=1}^T\subset \cZ^o$, let $\bar z$ denote its mean. It holds
      \$
          T\cdot \dualgap(\bar z) \leq  \sup_{z\in \cZ}  \sumT \big\langle z_t - z,  G(z_t)\big\rangle .
      \$
  \end{lemma}
  \begin{proof}[Proof of Lemma~\ref{lm:gaptoregret}]
      This lemma depends on the convexity-concavity of the saddle function $F$.
      
      Denote $\bar z \defeq [\bar x, \bar y]$, $z_t \defeq [x_t,y_t]$. Note $\bar x = (1/T)\sumT x_t$ and $\bar y = (1/T)\sumT y_t$. By definition of the duality gap and the convexity-concavity of $F$,
      \$
      \dualgap(\bar z) 
      & \defeq \sup_{x\in\cX,y\in\cY} F(\bar x,y) - F(x,\bar y)
      \\
      & \leq  \sup_{x\in\cX,y\in\cY} \frac{1}{T}\sumT F(x_t,y) - \frac{1}{T}\sumT F(x, y_t) .
      \$
      Let $G(z_t)= G(x_t,y_t) \defeq [d_{x,t}, - d_{y,t}]$. Since $d_{x,t} \in \partial_x F(x_t,y_t)$, for all $x\in \cX$ and $y\in \cY$ , 
      \$
          F(x_t, y ) + \langle d_{x,t}, x-x_t\rangle \leq F(x,y).
      \$
      Similarly, for all $x\in \cX$ and $y\in \cY$, it holds
      \$
      F(x, y_t ) + \langle d_{y,t}, y-y_t \rangle \geq F(x,y).
      \$
      We have 
      \$
      T\cdot \dualgap(\bar z) 
      & \leq \sup_{x\in\cX,y\in\cY} \sumT
      \langle d_{x,t}, x_t - x\rangle - \langle d_{y,t} , y_t - y\rangle
      \\
      & =  \sup_{z\in \cZ} \sumT \langle G(z_t), z_t - z\rangle. 
      \$
      This completes the proof of Lemma~\ref{lm:gaptoregret}.
  \end{proof}

\section{Additional Experiments}\label{appen:experiments}
We implement our algorithm and conduct all the experiments on 
a computer with Intel Core i5 CPU @ 3.20GHz cores, 8GB RAM, and GPU @ GeForce RTX 3090. 
The deep learning framework we use is PyTorch 1.8.1. 
The OS environment was created by Conda over Ubuntu 20.04. 
We use Python 3.7. Python library requirement is specified 
in the configuration file provided in the supplemental materials. 
Due to the hardware limitation, we simulate the distributed environment
by creating object instances to simulate multiple clients and 
a central server on one GPU card.

\subsection{Stochastic bilinear minimax problem}\label{app:add_exp}

\begin{figure*}[tp]
	\centering
	\hspace{-.9cm}
	\subfigure[]{
		\includegraphics[width = 0.45\textwidth]
		{ 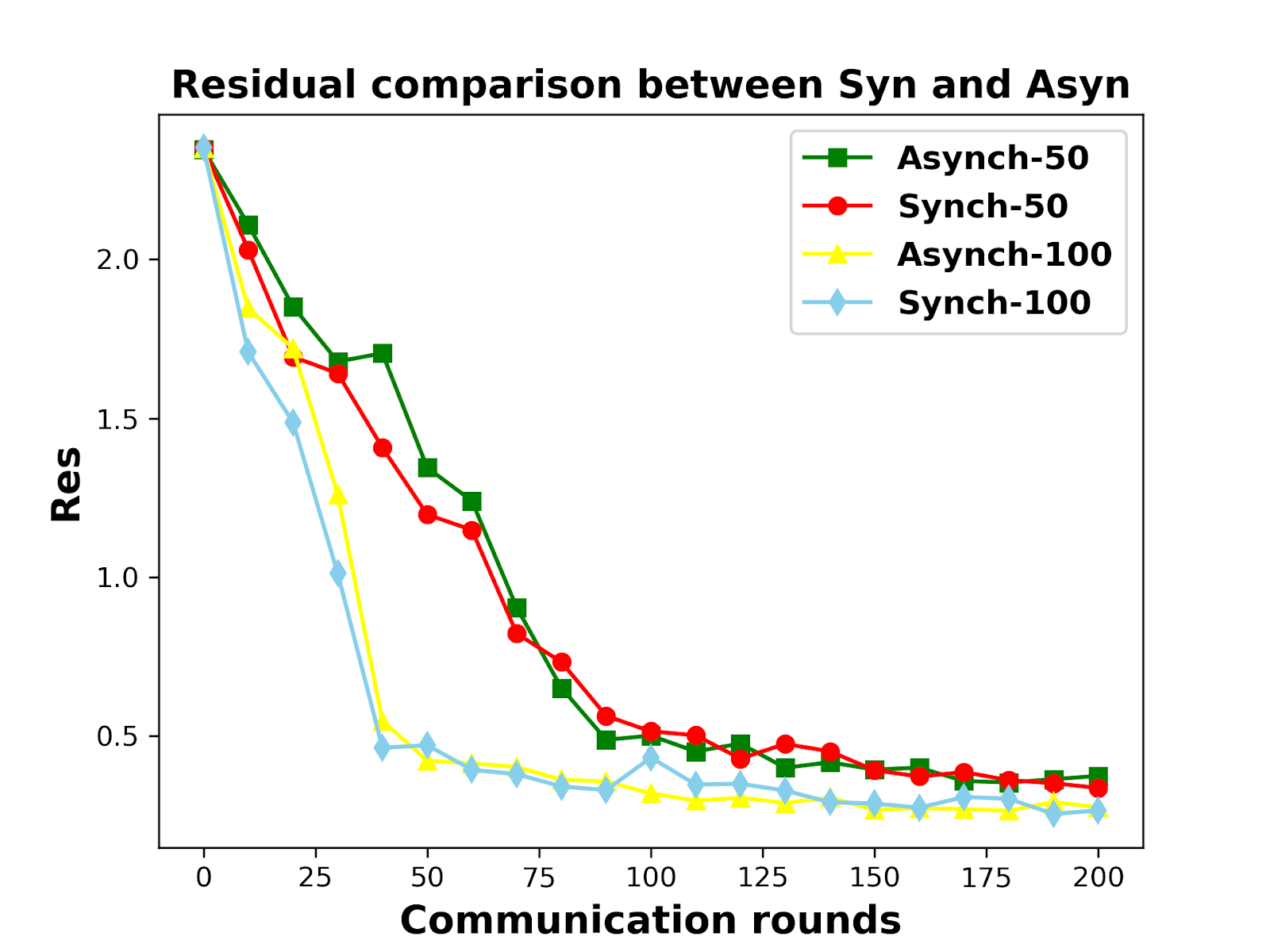}}
	\subfigure[]{
		\includegraphics[width = 0.45\textwidth]
		{ 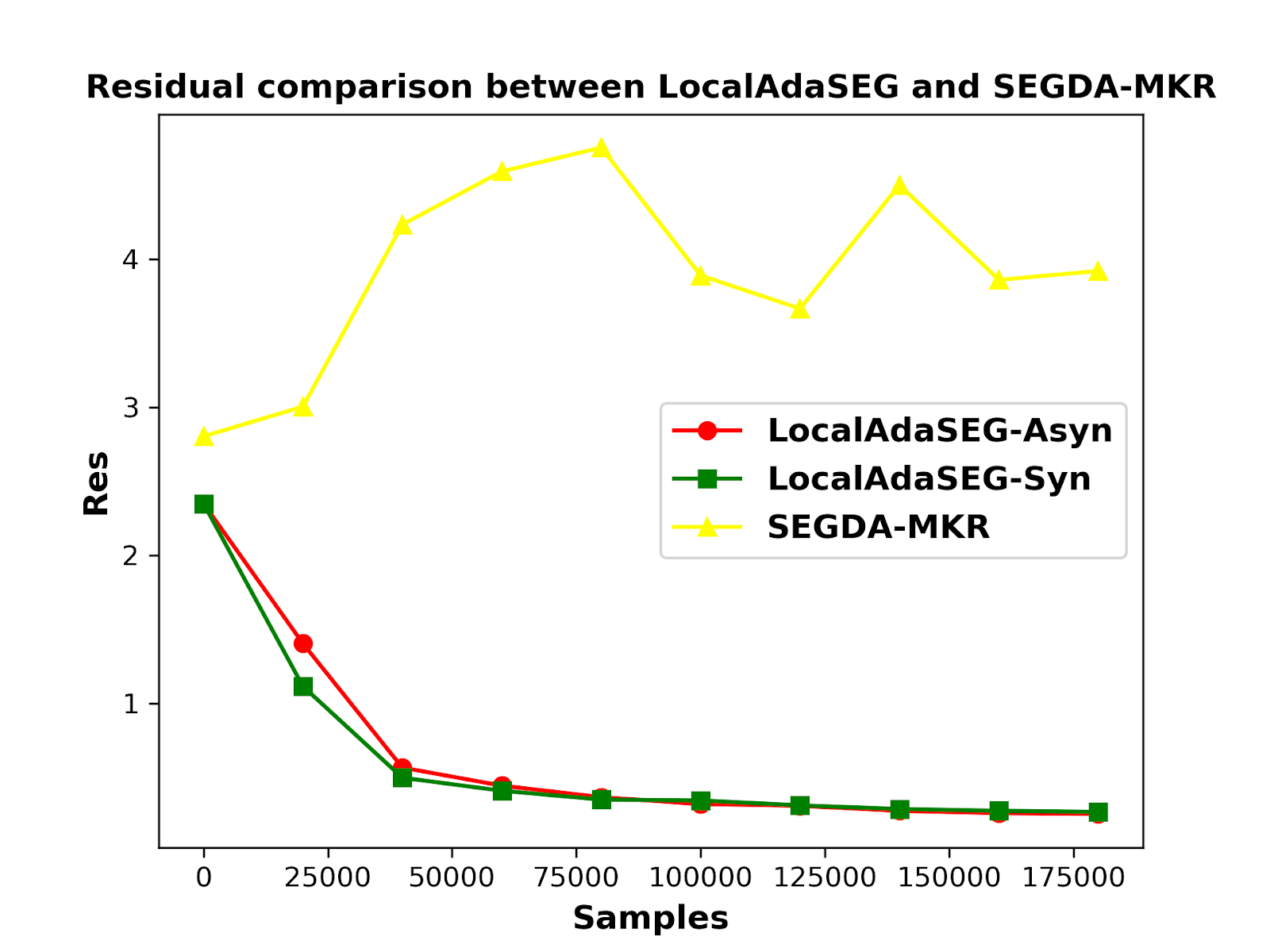}}
		
	\subfigure[]{
		\includegraphics[width = 0.45\textwidth]
		{ 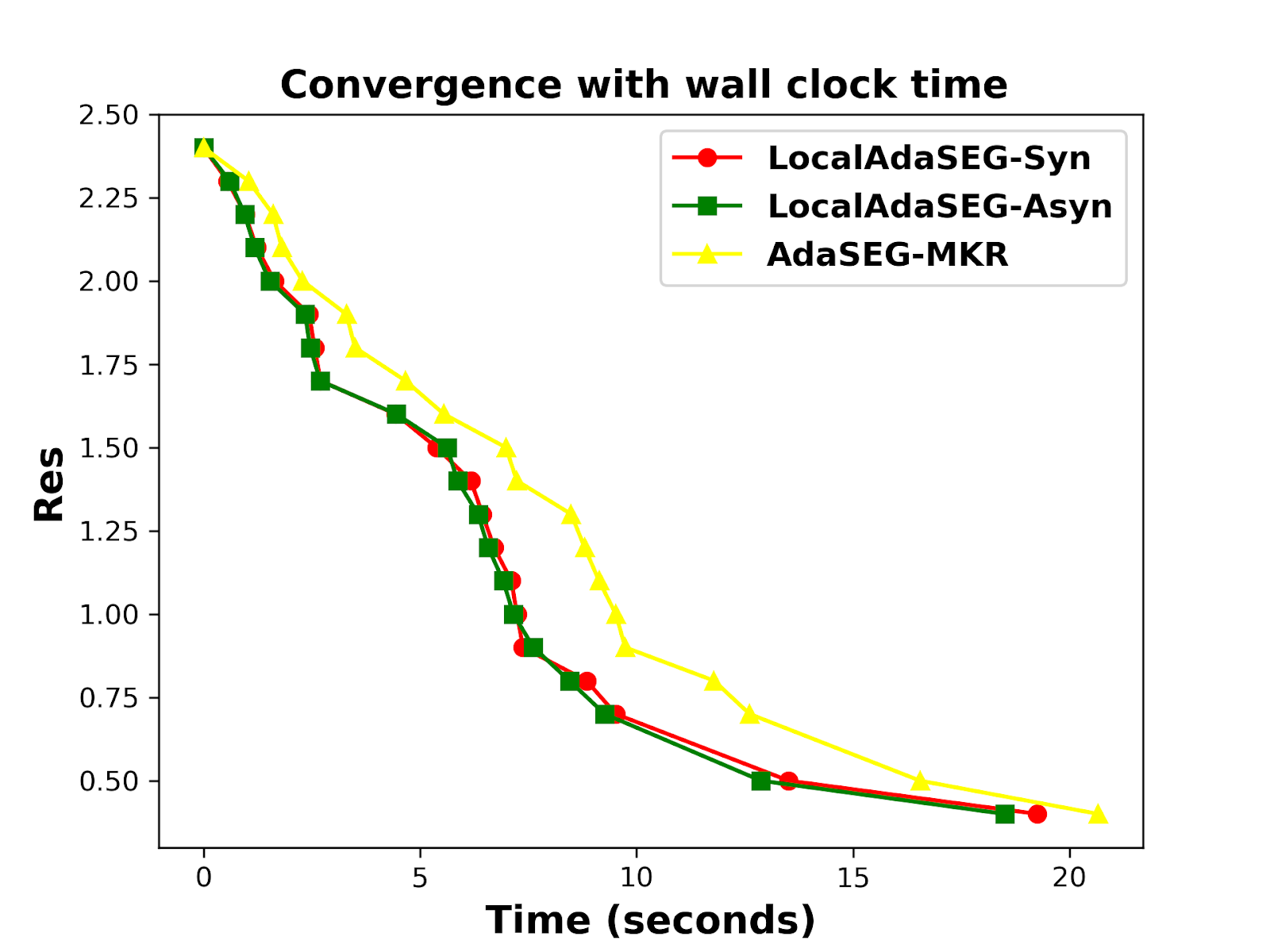}}
	\subfigure[]{
		\includegraphics[width = 0.45\textwidth]
		{ 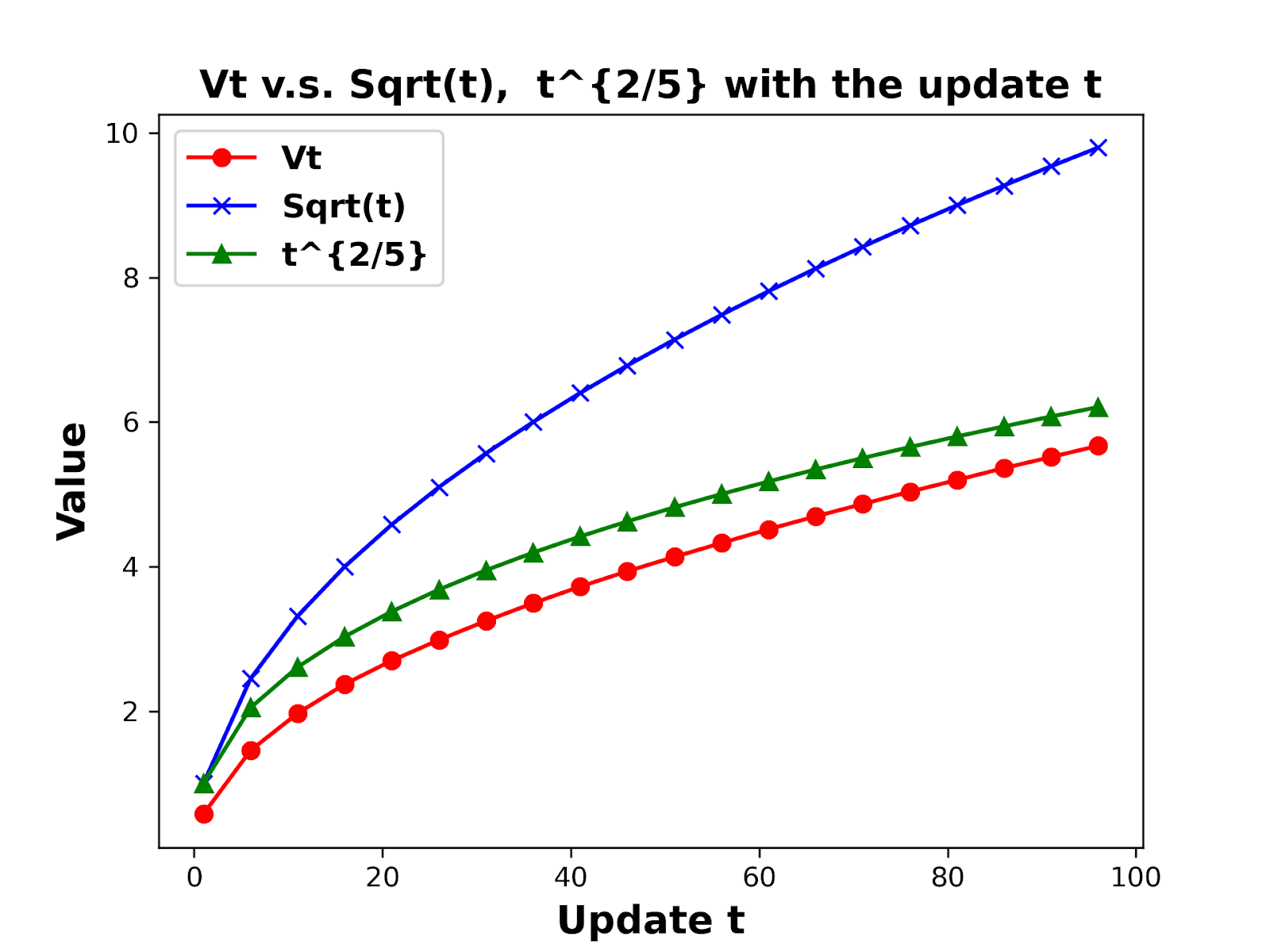}}
	
	\caption{\small (a) Residual comparison between synchronous and asynchronous cases with the communication rounds; (b) Residual comparison between $\ALGO$ (Asynchronous and Synchronous version) and SEGDA-MKR with samples; (c) Residual comparison between $\ALGO$ (Asynchronous and Synchronous version) and SEGDA-MKR with wallclock time; (d) Comparison of $V_t$, $\sqrt{t}$, $t^{2/5}$ with the update $t$. }
	\label{fig:additional_exp}
	\vspace{0cm}	
\end{figure*}

Experimentally, to validate the performance of our proposed method, we conduct an asynchronous variant of our proposed $\ALGO$ for the stochastic bilinear minimax problem. Specifically, we vary the number of local iterations $K$ in $M$ workers, where $M=4$ and the noise level $\sigma=0.1$.   In the case of 'Asynch-50', the local iteration $K$ is in the range of $\{50, 45, 40, 35 \}$  for each worker, and $K=50$ is adopted for all workers  in 'Synch-50'.
Similarly, in the case of 'Asynch-100', the local iteration $K$ varies in the range of $\{100, 90, 80, 70\}$. In the comparison, $K$ is fixed to 100 for each worker in the case of 'Synch-100'.  As can be seen from  \cref{fig:additional_exp} (a), both asynchronous and synchronous cases converge to an optimal point after several communication rounds. Compared with synchronous cases, asynchronicity only affects the convergence rate that is slower than the synchronous version with respect to the communication rounds.

Secondly, we compared our $\ALGO$ (both Asynchronous and Synchronous versions) with SEGDA of MKR iterations to solve bilinear minimax problems(refer to Section 6.1). Specifically, we choose $M =4$ workers, the noise level $\sigma=0.1$ and local iteration $K=50$ in the Synchronous case, and $K$ in the range of $\{50, 45, 40, 35, 30\}$ in the Asynchronous case. To provide fairness, we run vanilla  SEGDA with $M\times K \times R$ iterations on one worker with batchsize = 1, where $M$ denotes the number of workers, $K$ denotes the number of local iterations and $R$ represents the number of rounds. The experimental results are illustrated in \cref{fig:additional_exp} (b). As can be seen, the performance of SEGDA is unstable and worse than that of $\ALGO$ (Asyn. and Syn.). The reason is possible since the batchsize of stochastic gradient $bs=1$ in each iteration results in a large variance of stochastic gradient estimation.
Because there are several workers involved in the optimization in $\ALGO$, it has much more samples in each iteration than that of SEGDA-MKR. It indicates that the stochastic variance is smaller than that of SEGDA-MKR, resulting in the stable performance of $\ALGO$.

Thirdly, we also conduct experiments to validate the performance from the aspect of wallclock time on a bilinear minimax problem, where the number of workers $M=4$, and the noise level $\sigma=0.1$. We record the wallclock time of reaching the target residual value for synchronous $\ALGO$ ($K=50$), the asynchronous version ($K$ in the range of $\{50, 45, 40, 35, 30\}$), and the single thread version. The results are illustrated in \cref{fig:additional_exp} (c). As can be seen, compared with the single thread version, our proposed method speed-ups the convergence.  With respect to the wall clock time,  Asynchronous $\ALGO$($\ALGO$-Asyn) is slightly better than synchronous $\ALGO$(synchronous $\ALGO$-Syn). Since the tested bilinear minimax problem with noise level $\sigma=0.1$ is very simple (time cost is around 20 seconds), the differences in time cost between synchronous and asynchronous cases are not significant. 

Fourthly, we conduct the experiments with the bilinear case to evaluate the quantity of $Vt$ with the update t. Here, we adopt the same experimental settings as that of experiments in Section 6.1. The noise level $\sigma=0.1$ and the number of workers $M=4$. As can be seen from \cref{fig:additional_exp} (d), $Vt$ is really much smaller than the dominant variance term.

\begin{figure*}[tp]
	\centering

	\subfigure[]{
    	\includegraphics[width = 0.45\textwidth] 
    	{ 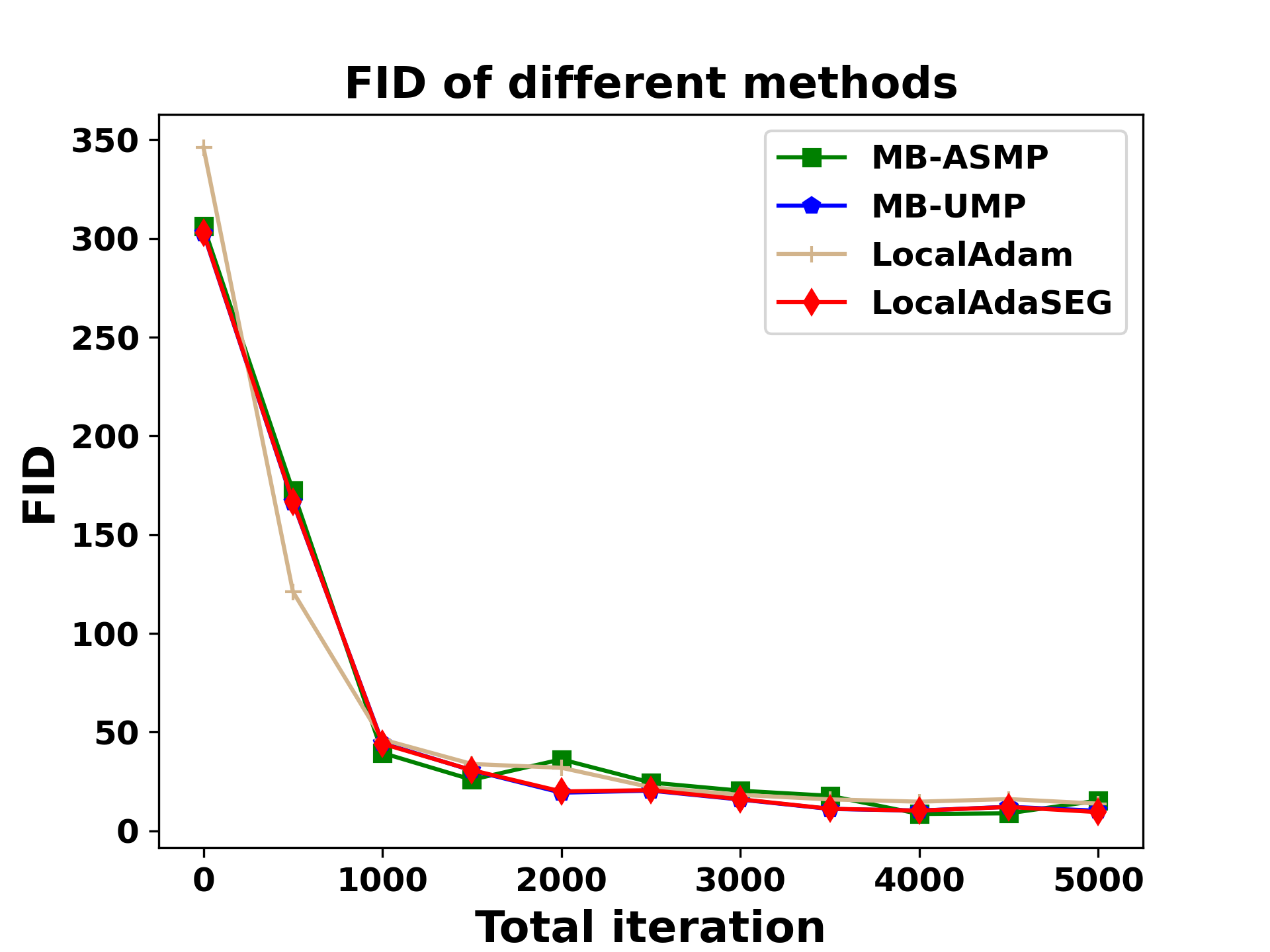}}
	\subfigure[]{
		\includegraphics[width = 0.45\textwidth]
		{ 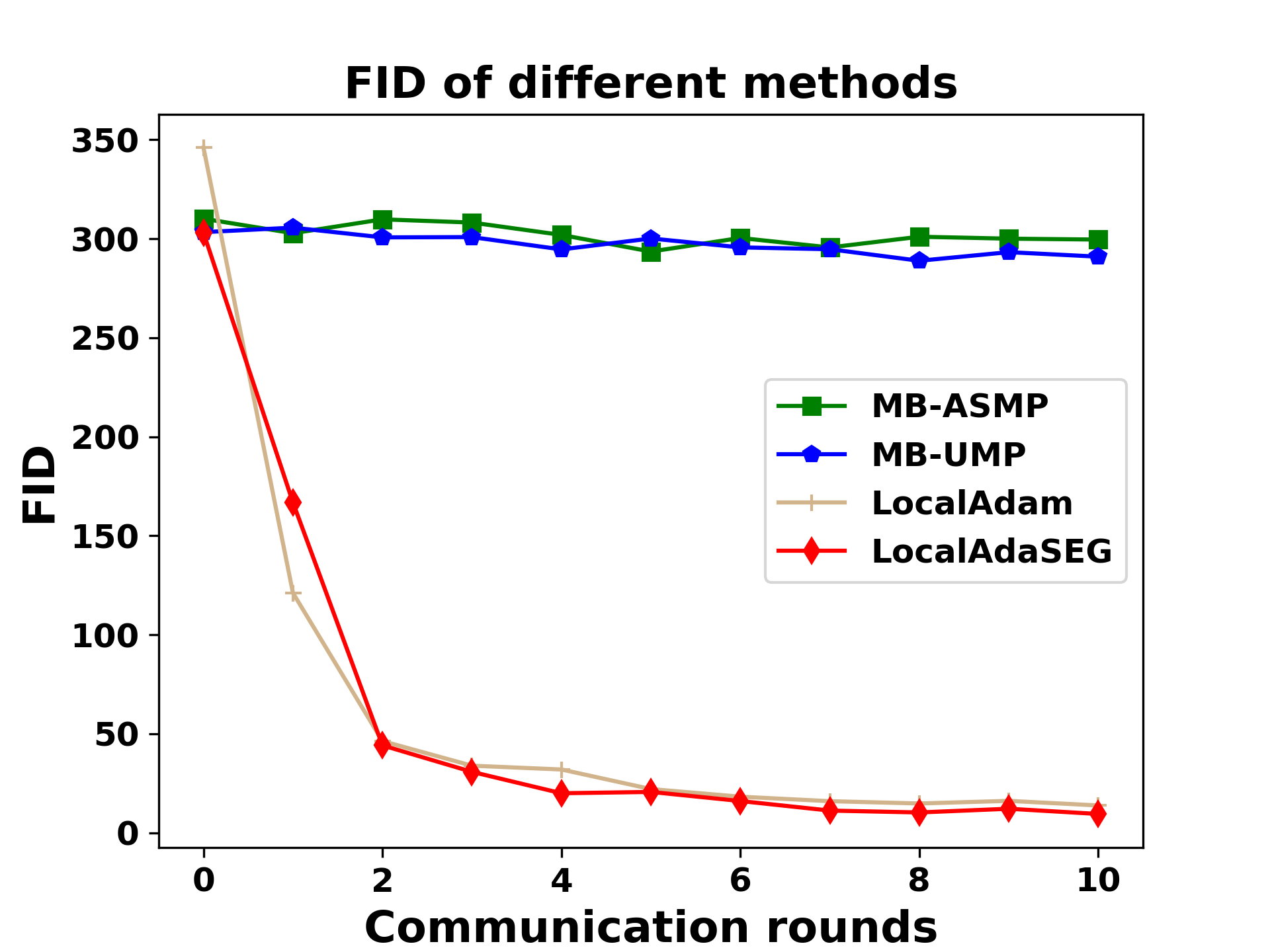}}

	\subfigure[]{
		\includegraphics[width = 0.45\textwidth]
		{ 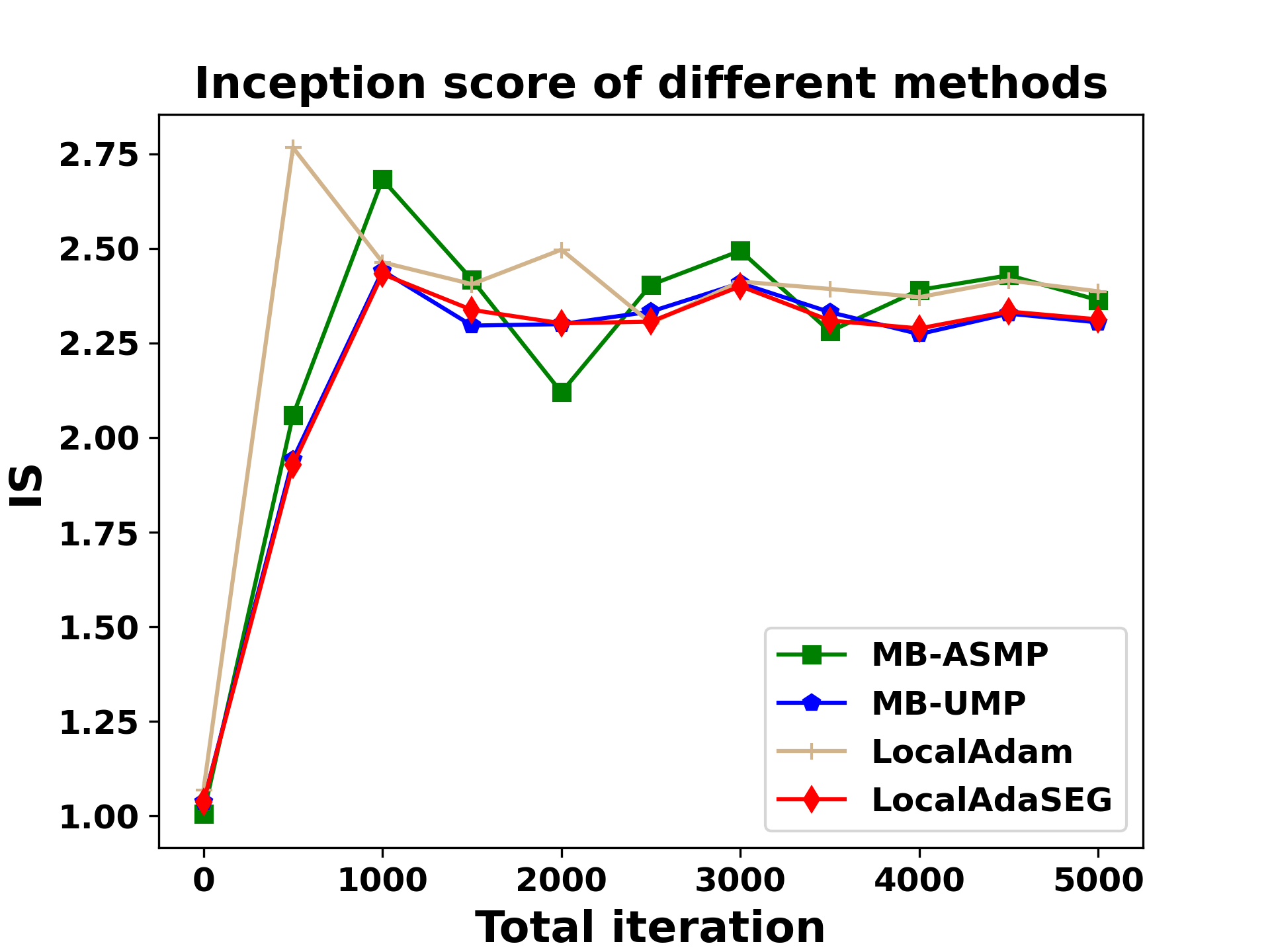}}
	\subfigure[]{
		\includegraphics[width = 0.45\textwidth]
		{ 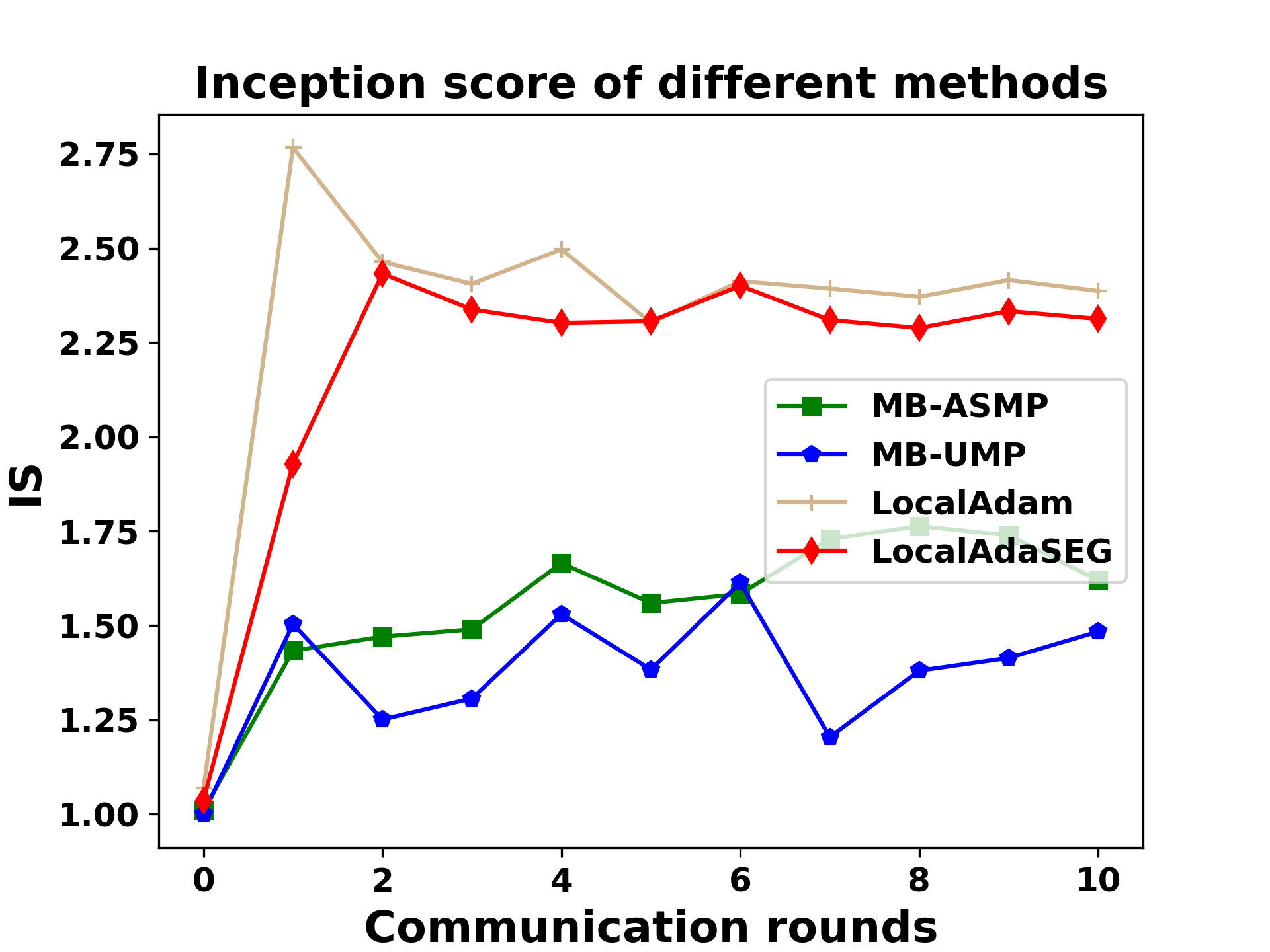}}

	\caption{\small Subfigures (a)-(b) and (c)-(d) show the results of WGAN trained with $\ALGO$ and existing optimizers. We plot FID and IS against the number of iterations and communications, respectively.}
	\label{fig:wgan_iid_agaisnt_SOTA}
\end{figure*}

\begin{figure*}[tp]
	\centering
	\subfigure[]{
     	\includegraphics[width = 0.45\textwidth]
    	{ 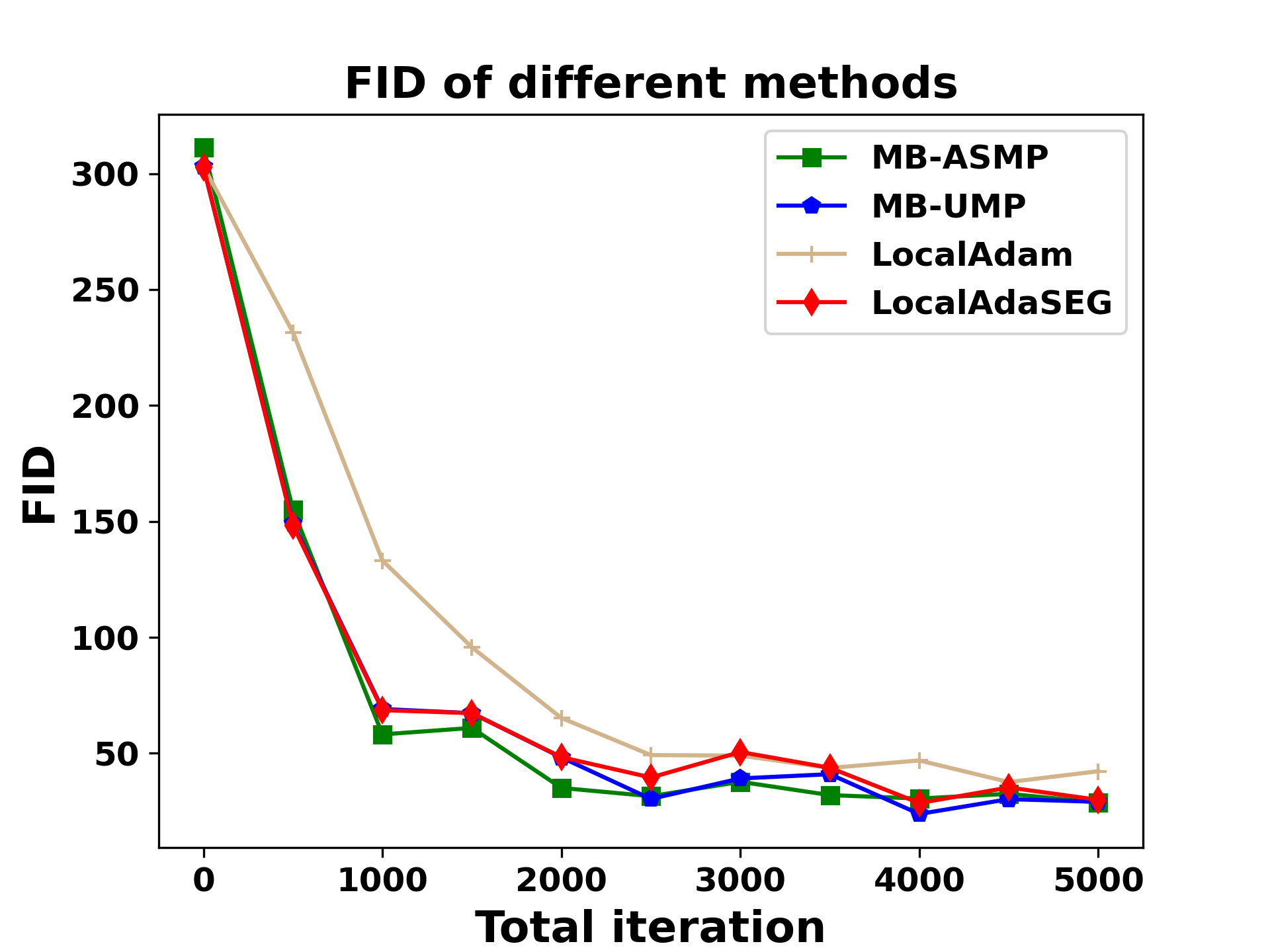}}
	\subfigure[]{
		\includegraphics[width = 0.45\textwidth]
		{ 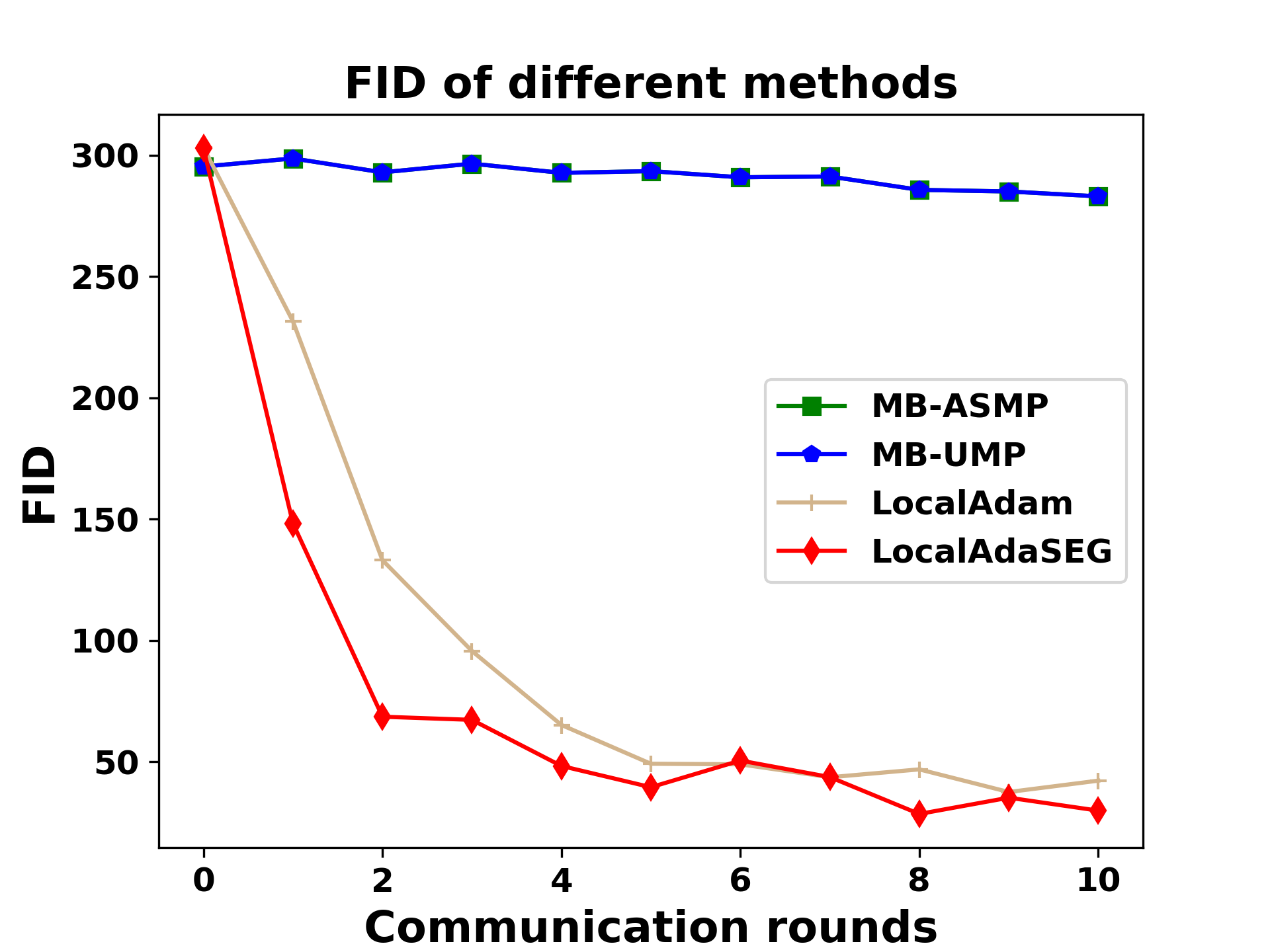}}

	\subfigure[]{
		\includegraphics[width = 0.45\textwidth]
		{ 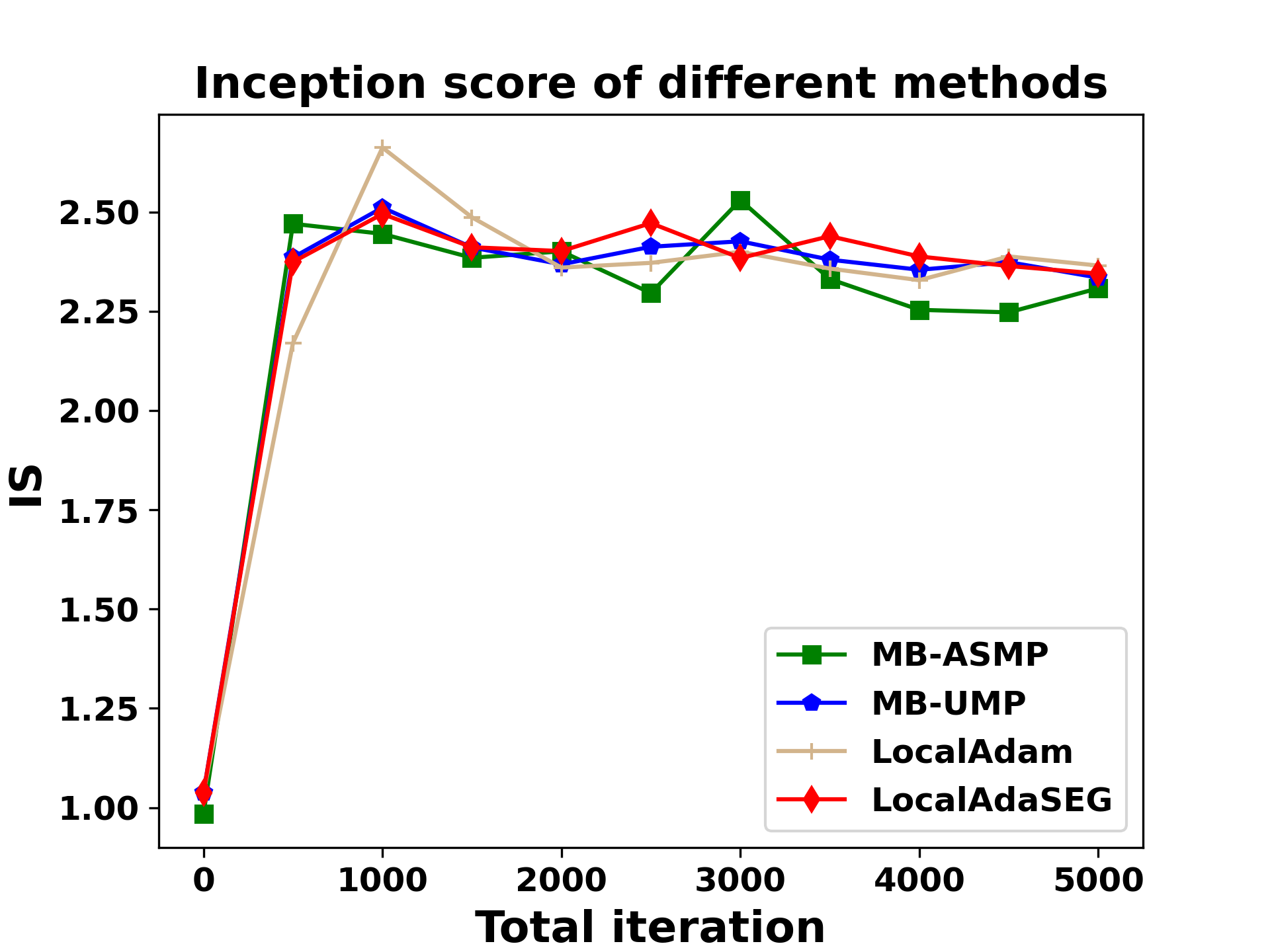}}
	\subfigure[]{
		\includegraphics[width = 0.45\textwidth]
		{ 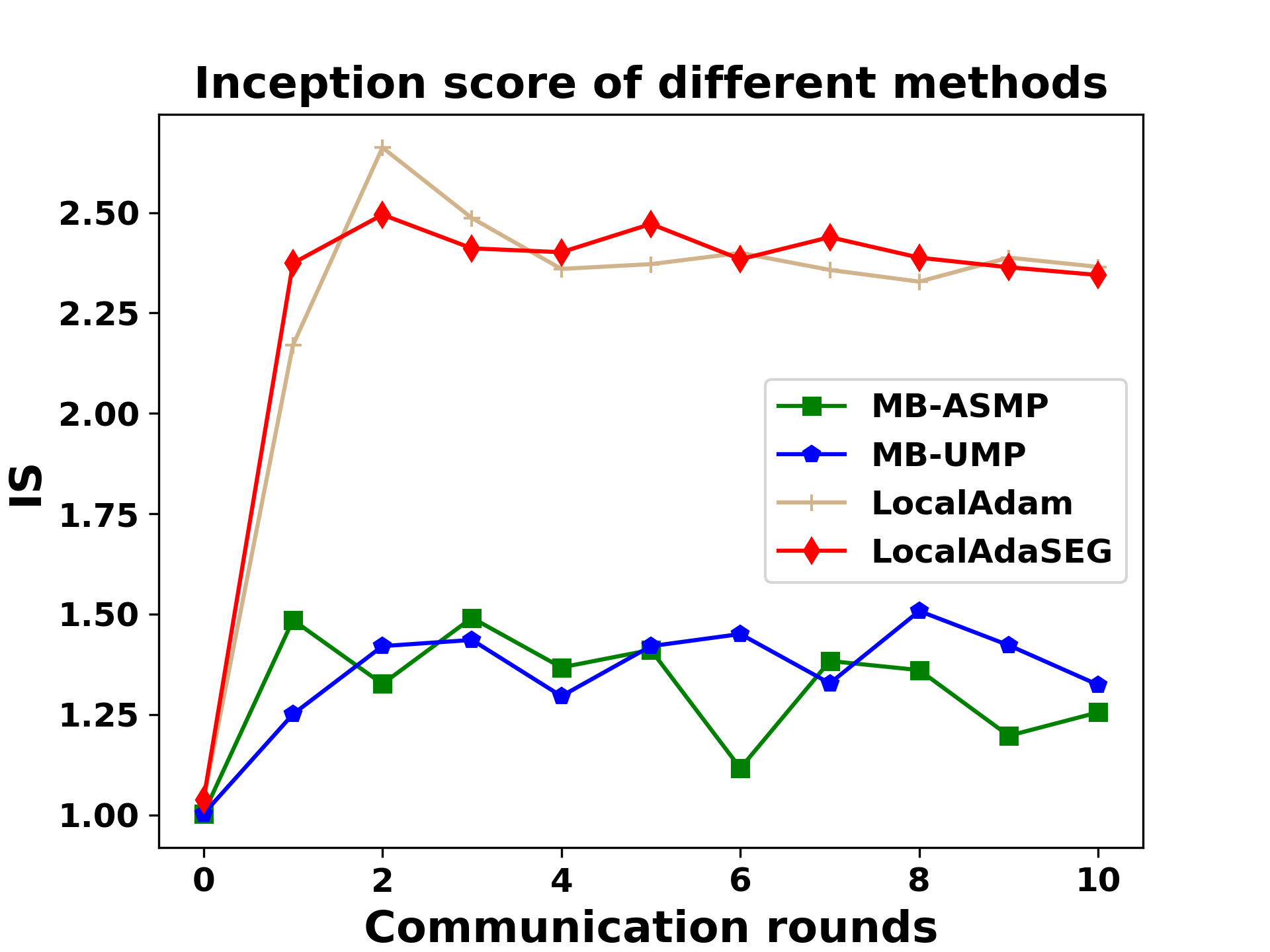}}
	\caption{\small Subfigures (a)-(b) and (c)-(d) show the results of Federated WGAN trained with $\ALGO$ and existing optimizers. We plot FID and IS against the number of iterations and communications, respectively.}
	\label{fig:wgan_noniid_agaisnt_SOTA}
\end{figure*}

\subsection{Wasserstein GAN}\label{wgan-discription}

Inspired by game theory, 
generative adversarial networks (GANs)
have shown great performance in many generative tasks
to replicate the real-world rich content, 
such as images, texts, and music. 
GANs are composed of two models, a generator and a discriminator, 
which are competing with each other to improve the
performance of a specific task.  In this experiment, 
we aim to train a digit image generator using the MNIST dataset.

It is challenging to train a GAN model due
to the slow convergence speed, instability of training
or even failure to converge. 
\cite{arjovsky2017towards, arjovsky2017wasserstein} proposed to use 
the Wasserstein distance as the GAN loss function to provide stable 
and fast training. To enforce the Lipschitz constraint on the
discriminator, we adopt WGAN with gradient penalty as 
our experimental model. The objective can be described as
\begin{align}\label{eq:wgan_loss}
	\mathop{\min}_{G} \mathop{\max}_{D} \bigg\{ \mathop{\mathbb{E}}_{x \sim \mathbb{P}_r} [D(x)] - \mathop{\mathbb{E}}_{z \sim \mathbb{P}_z} \big[D\big(G(z)\big)\big] -  \lambda  \big[\big(\| \nabla_{\hat{x}} D(\hat{x})   \|_2 - 1\big)^2\big]\bigg\}\,,
\end{align}
where $D$ and $G$ denote the generator and discriminator, 
$\mathbb{P}_r$ is the data distribution, and 
$\mathbb{P}_z$ represents the noise distribution
(uniform or Gaussian distribution).
The point
$\hat{x} \sim \mathbb{P}_{\hat{x}}$ is sampled
uniformly along straight lines between pairs of points
sampled from the real data distribution $\mathbb{P}_r$ and 
the generator distribution $\mathbb{P}_{\tilde{x}}$,
expressed as $\hat{x} := \epsilon x + (1-\epsilon) \tilde{x}$, 
where $\epsilon \sim U[0,1]$.

{\bf DCGAN.}
We implement WGAN with the DCGAN architecture,
which improves the original GAN  with convolutional layers. 
Specifically, the generator consists of $3$ blocks, 
which contain deconvolutional layers, 
batch normalization and activations. 
The details of the whole generator can be represented as 
sequential layers \emph{\{Linear, BN, ReLU, DeConv, BN, ReLU, DeConv, BN, ReLU, DeConv, Tanh\}},
where \emph{Linear}, \emph{BN}, \emph{DeConv} denote the linear,
batch normalization and deconvolutional layer, respectively. 
\emph{ReLU} and \emph{Tanh} represent the activation functions.
Similarly, the discriminator also contains $3$ blocks, 
which can be described as sequential layers
\emph{\{Conv, LReLU, Conv, LReLU, Conv, LReLU, Linear\}},
where \emph{Conv} and \emph{LReLU} denote the convolutional 
layer and Leaky-ReLU activation function, respectively.

{\bf Inception score (IS).}
Inception score (IS) is proposed to evaluate the performance of a GAN with 
an inception model. IS measures GAN from two aspects simultaneously. 
Firstly, GAN should output a high diversity of images.
Secondly, the generated images should contain clear objects. 
Specifically, we feed the generated images $x$ into a 
well-trained inception model to obtain the output $y$.
Then, IS can be calculated by the following equation:
\begin{align}
	\mathrm{IS} \defeq \exp\bigg(\mathop{\mathbb{E}}_{x \sim \mathbb{P}_g} \big[D_{\mathrm{KL}}\big(p(y \given x) \| p(y)\big)\big]\bigg),
\end{align}
where $\mathbb{P}_g$ is the generator model distribution. 
Essentially, IS computes the mutual information 
$I(y; x) = H(y) - H(y\given x)$, where $H(\cdot)$ denotes the entropy. 
The larger $H(y)$, the more diversity in the generated images.
The lower $H(y\given x)$ implies the input $x$ belongs to one class with a higher probability.
In summary, IS is bounded by $1 \leq \mathrm{IS} \leq 1000$. 
The higher IS implies a better performance of a GAN.

{\bf Fr\'echet Inception Distance (FID).}
Although IS can measure the diversity and quality of the generated images,
it still has some limitations, 
such as losing sight of the true data distribution,
failure to measure the model generalization. 
FID is an improved metric for GAN, 
which cooperates with the training samples and 
generated samples to measure the performance together. 
Specifically, we feed the generated samples and training samples 
into an inception model to extract the feature vectors, respectively.
Usually, we extract the logits value 
before the last sigmoid activation as the feature vector
with dimension $2048$. Essentially, FID is the Wasserstein 
metric between two multidimensional Gaussian distributions:
$\mathcal{N}(\mu_g, \Sigma_g)$ the distribution  of feature vectors from generated samples
and
$\mathcal{N}(\mu_r, \Sigma_r)$  the distribution of feature vectors from the training samples.
It can be calculated as
\begin{align}
	\mathrm{FID} \defeq \| u_r - u_g \|^2 + \mathrm{tr} \big(\Sigma_r + \Sigma_g - 2(\Sigma_r \Sigma_g)^{1/2}\big)
\end{align}
where $\mathrm{tr}(\cdot)$ denotes the trace of a matrix.
The lower the FID, the better the performance of a GAN. 


\begin{figure*}[tp]
	\centering

	\subfigure[]{
    	\includegraphics[width = 0.45\textwidth] 
    	{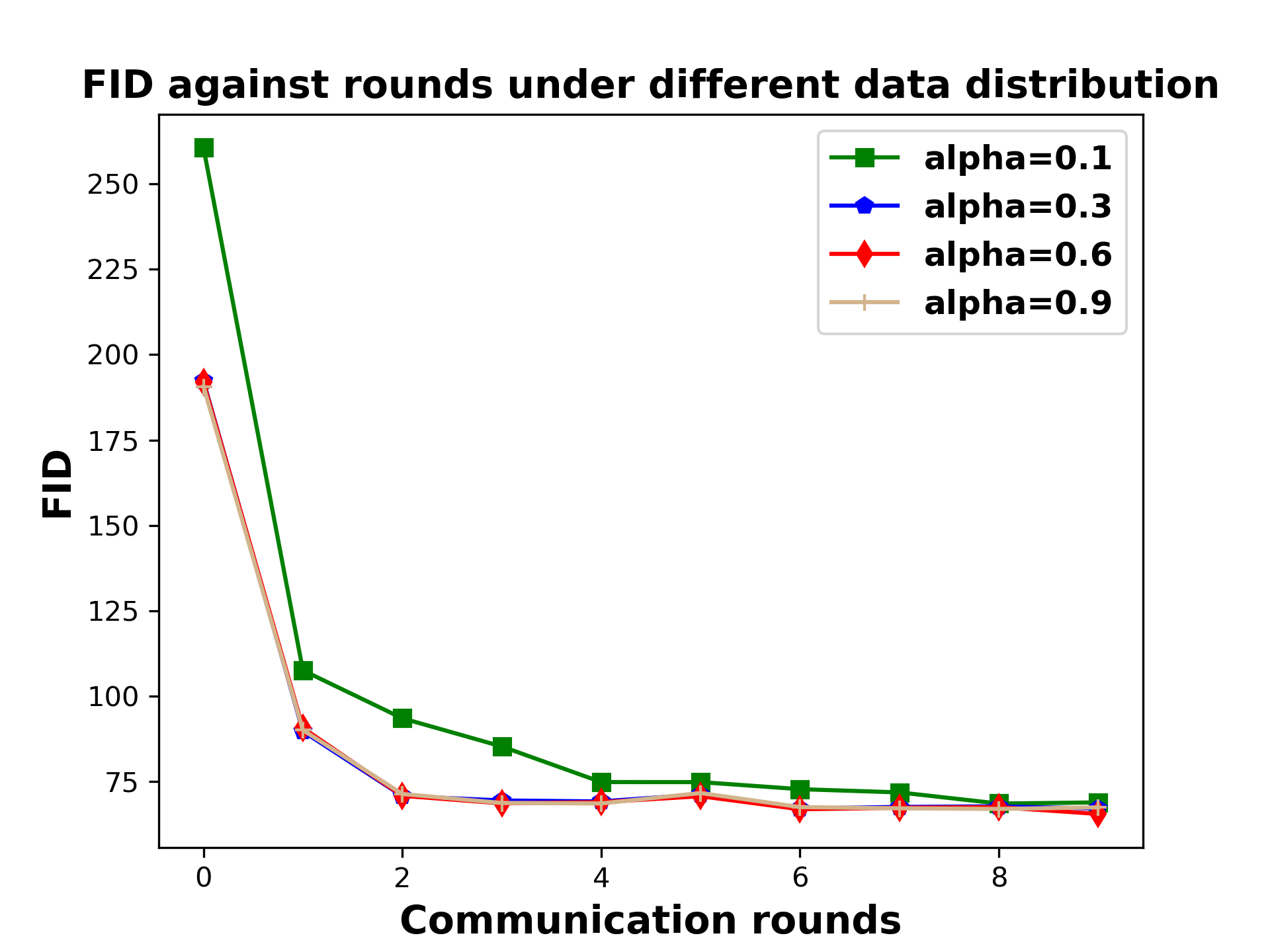}}
	\subfigure[]{
		\includegraphics[width = 0.45\textwidth]
		{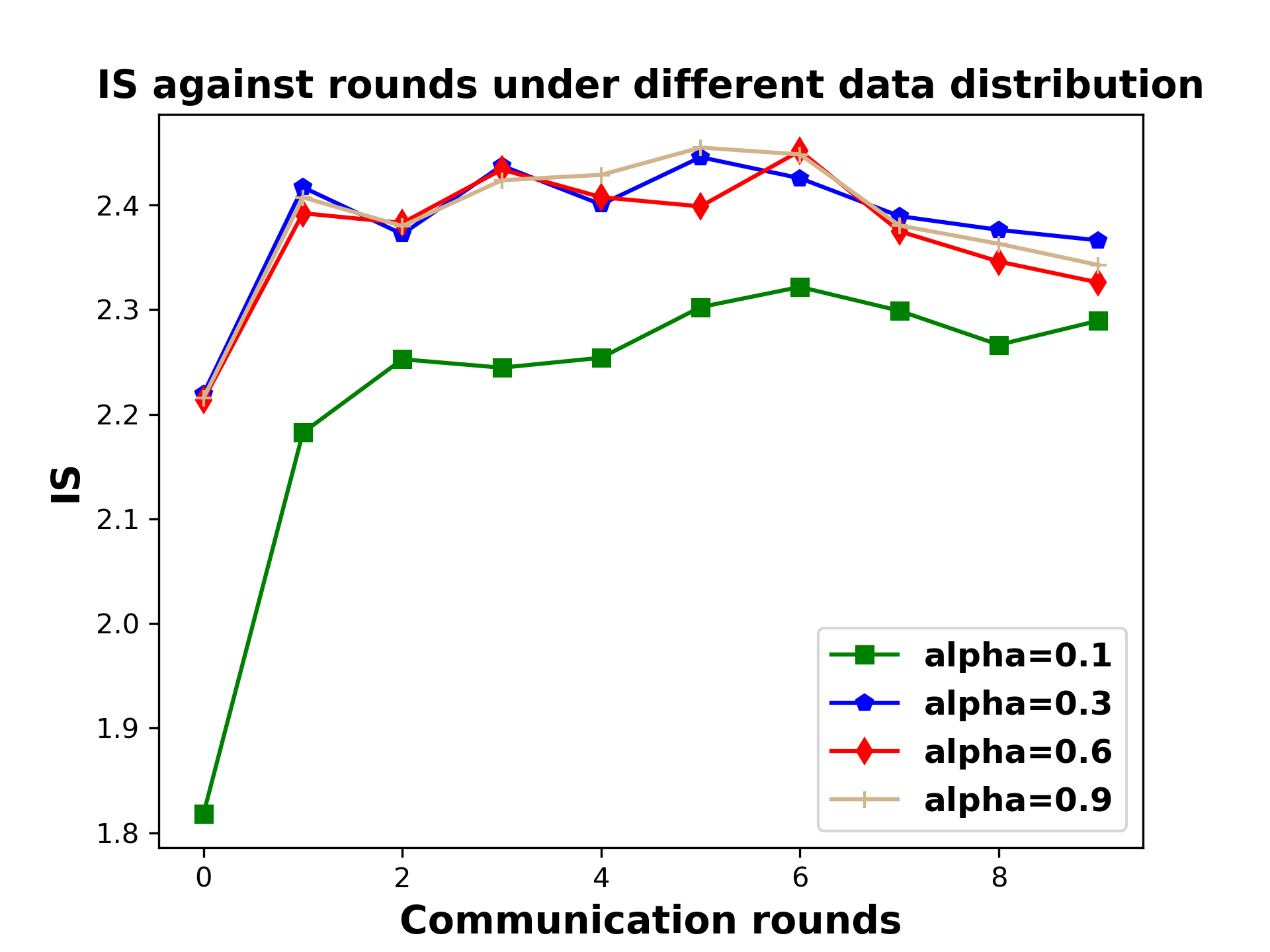}}

	\caption{\small Subfigures (a)-(b) show the FID and IS against communication rounds of WGAN trained over MNIST dataset under different Dirichlet distribution.}
	\label{fig:wgan_under_diff_dist}
\end{figure*}

{\bf Implementation details.}\ 
Experiments are conducted on the MNIST datasets of digits
from $0$ to $9$, with $60000$ training images of size $28 \times 28$.
We adopt the same network architecture of WGAN as that of DCGAN \cite{arjovsky2017wasserstein}. 
We simulate $M=4$ parallel workers and run $\ALGO$ with the batch
size $128$ and local iteration steps $K = 500$. 
In the homogeneous setting, the local data in each worker 
is uniformly sampled from the entire dataset. 
In the heterogeneous setting, 
we partition the MNIST dataset into $4$ subsets using the partition methods in \cite{lin2019don}.
Then each worker is loaded with a fraction of the dataset. 
Due to non-adaptive learning rates, LocalSGDA, LocalSEGDA, and MB-SEGDA
are hard to tune and do not achieve satisfactory performance for training WGAN. 
For a better illustration, we only show the performance of   
$\ALGO$, MB-UMP, MB-ASMP and Local Adam. To measure the efficacy 
of the compared optimizers, we plot FID and IS \cite{heusel2017gans} 
against the number of iterations and communications, respectively.

To investigate the convergence rate of $\ALGO$ under different data distributions, we conduct to train WGAN over MNIST dataset with different Dirichlet distributions.
In \cref{fig:wgan_under_diff_dist}, it  shows the FID and IS of training the WGAN under the Dirichlet distribution with various parameters $ \alpha \in \{0.1, 0.3, 0.6, 0.9\}$. It should be noted that when $\alpha$ increases, the data distribution trends closer to the homogeneous setting. As can be seen from \cref{fig:wgan_under_diff_dist}, it converges faster when the parameter $\alpha$ decreases.

Furthermore, we also compare $\ALGO$ with other existing optimizers under various Dirichlet distributions which parameter $\alpha \in \{0.1, 0.3\}$, respectively. As can be seen from \cref{fig:compare_wgan_under_diff_dist}, our proposed LocalAdaSEG also achieves a faster convergence compared with exiting minimax optimizers.

\begin{figure*}[tp]
	\centering
	\centering

	\subfigure[]{
    	\includegraphics[width = 0.45\textwidth] 
    	{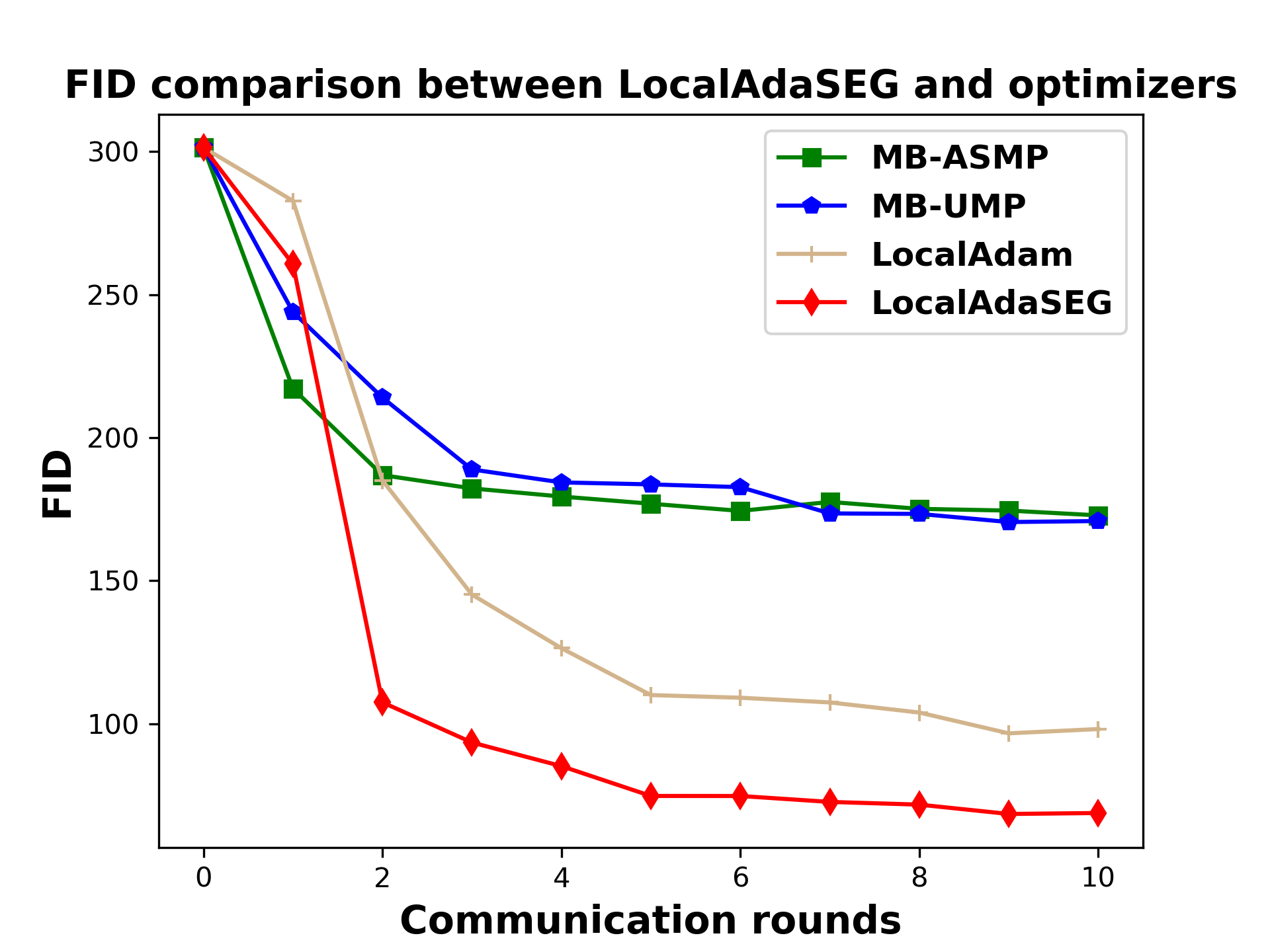}}
	\subfigure[]{
		\includegraphics[width = 0.45\textwidth]
		{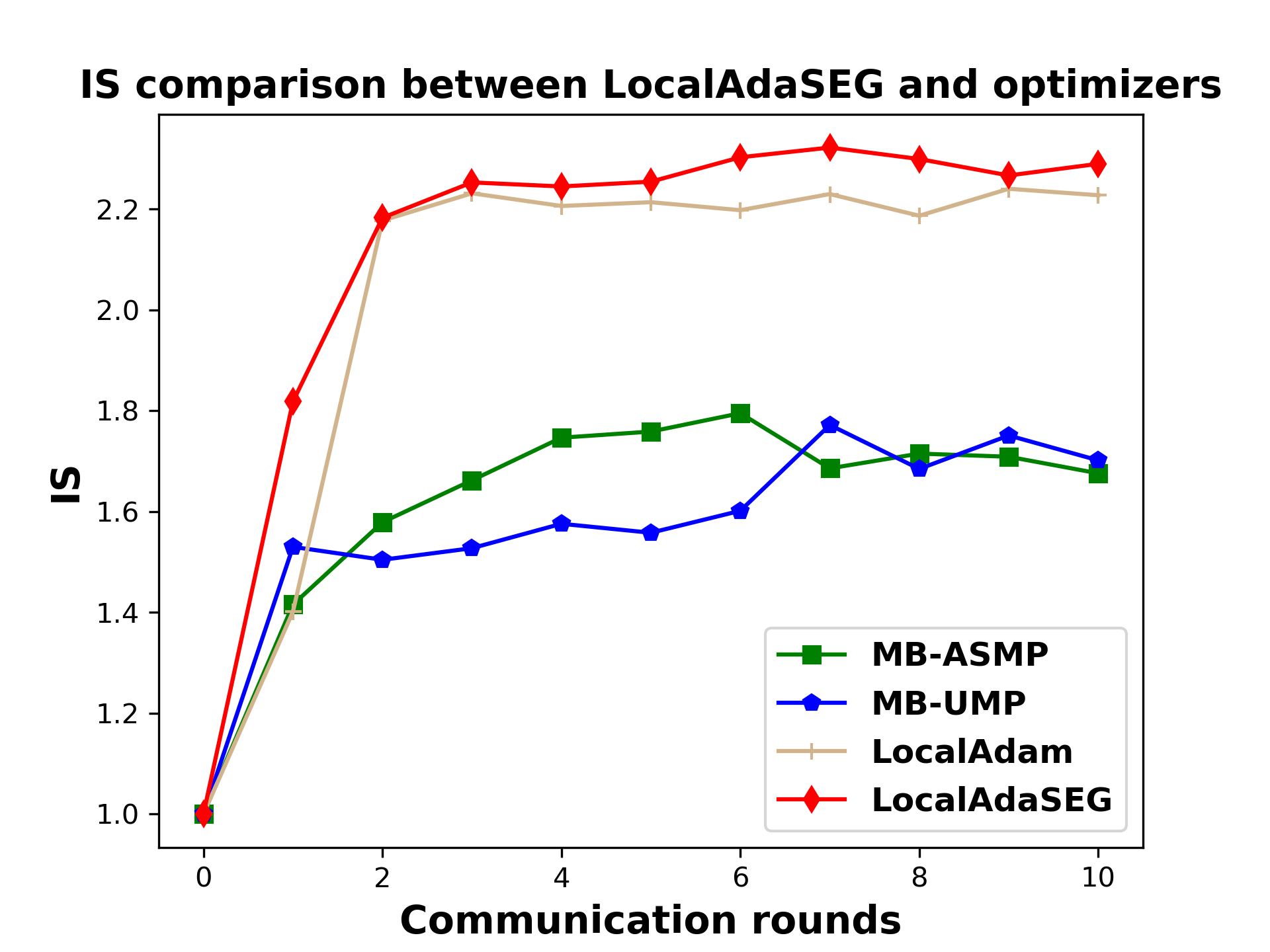}}

        \subfigure[]{
    	\includegraphics[width = 0.45\textwidth] 
    	{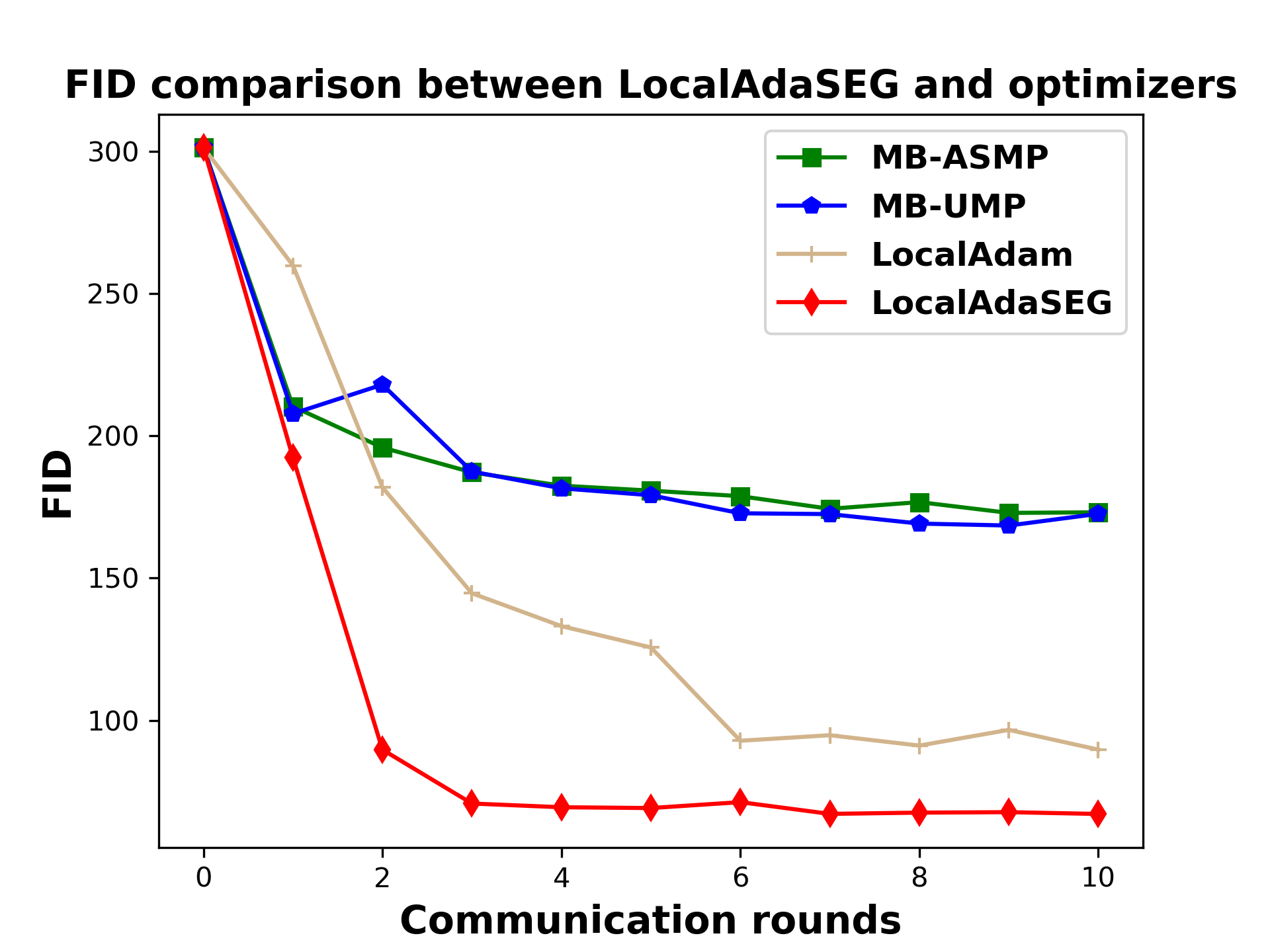}}
	\subfigure[]{
		\includegraphics[width = 0.45\textwidth]
		{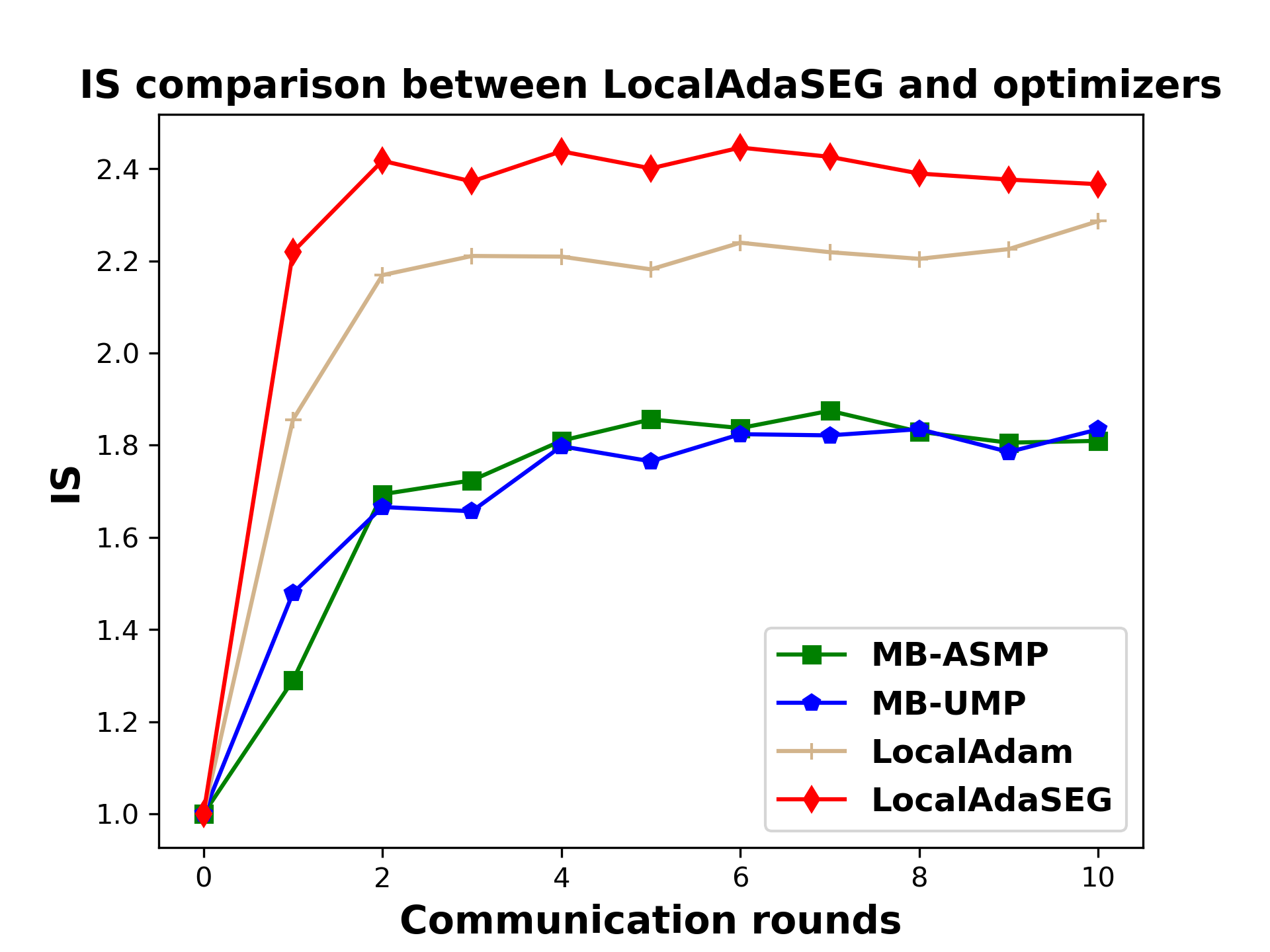}}

	\caption{\small Subfigures (a)-(b) show the FID and IS against communication rounds of WGAN trained over MNIST dataset with $\alpha = 0.1$; (c)-(d) show the FID and IS against communication rounds of WGAN trained over MNIST dataset with $\alpha = 0.3$.}
	\label{fig:compare_wgan_under_diff_dist}
\end{figure*}

\subsection{BigGAN on CIFAR10} \label{biggan_cifar10}

\textbf{BigGAN.}
Although a variety of GANs were investigated to effectively generate images, they are still restricted to small image synthesis and the training process remains dependent on augmentations and hyperparameters. BigGAN is a typical network to pull together a suite of recent best practices in training class-conditional images and scaling up the batch size and number of model parameters. The result is the routine generation of both high-resolution and high-fidelity natural images. Since utilizing the self-attention module and skip connections, BigGAN can not be simply considered as a sequential combination of layers. Here, we only describe the key module in the BigGAN. Specifically, the generator consists of $1$ linear layer, $N$ generator blocks, and $1$ sequential layer. Each generator block can be represented as several layers \emph{\{BN, ReLU, SNConv, BN, ReLU, SNConv, SNConv\}}, where \emph{SNConv} denotes 2d convolutional layer with the spectral norm.  Similarly, The discriminator also consists of linear layers and discriminator blocks. Each discriminator block can be listed as several layers \emph{\{ReLU, AvgPool, SNConv, SNConv, SNConv\}}, where \emph{AvgPool} denotes the average pooling.

\textbf{CIFAR10.}
The CIFAR-10 dataset consists of 60000 32x32 color images in 10 classes, with 6000 images per class. There are 50000 training images and 10000 test images. The test batch contains exactly 1000 randomly-selected images from each class. The training batches contain the remaining images in random order, but some training batches may contain more images from one class than another. Between them, the training batches contain exactly 5000 images from each class.

\textbf{Parameter Setup.}
We implement the BigGAN with the original CIFAR10 dataset, meaning that all the images are fed into the BigGAN without cropping and rotation. The experiments are conducted among $M=4$ parallel workers in a heterogeneous setting. It implies that all the training images are divided into $4$ parts by using Dirichlet distribution with parameter $\alpha = 0.6$. Each work runs the $\ALGO$ with the batch size $bs=125$ and local iteration steps $K=100$. The input noise dimension is $125$, and the channel of the generator block and discriminator is set to $16$. To show the performance of $\ALGO$, MB-UMP, MB-ASMP, and Local Adam, We plot FID and IS against the communication rounds in the \cref{fig:bigGAN_FID_IS}.

 \cref{fig:bigGAN_FID_IS} illustrates the FID and IS score of BigGAN over the CIFAR10 dataset. As can be seen that $\ALGO$ and Local Adam converge faster than other optimizers.

\begin{figure*}[tp]
	\centering

	\subfigure[]{
    	\includegraphics[width = 0.45\textwidth] 
    	{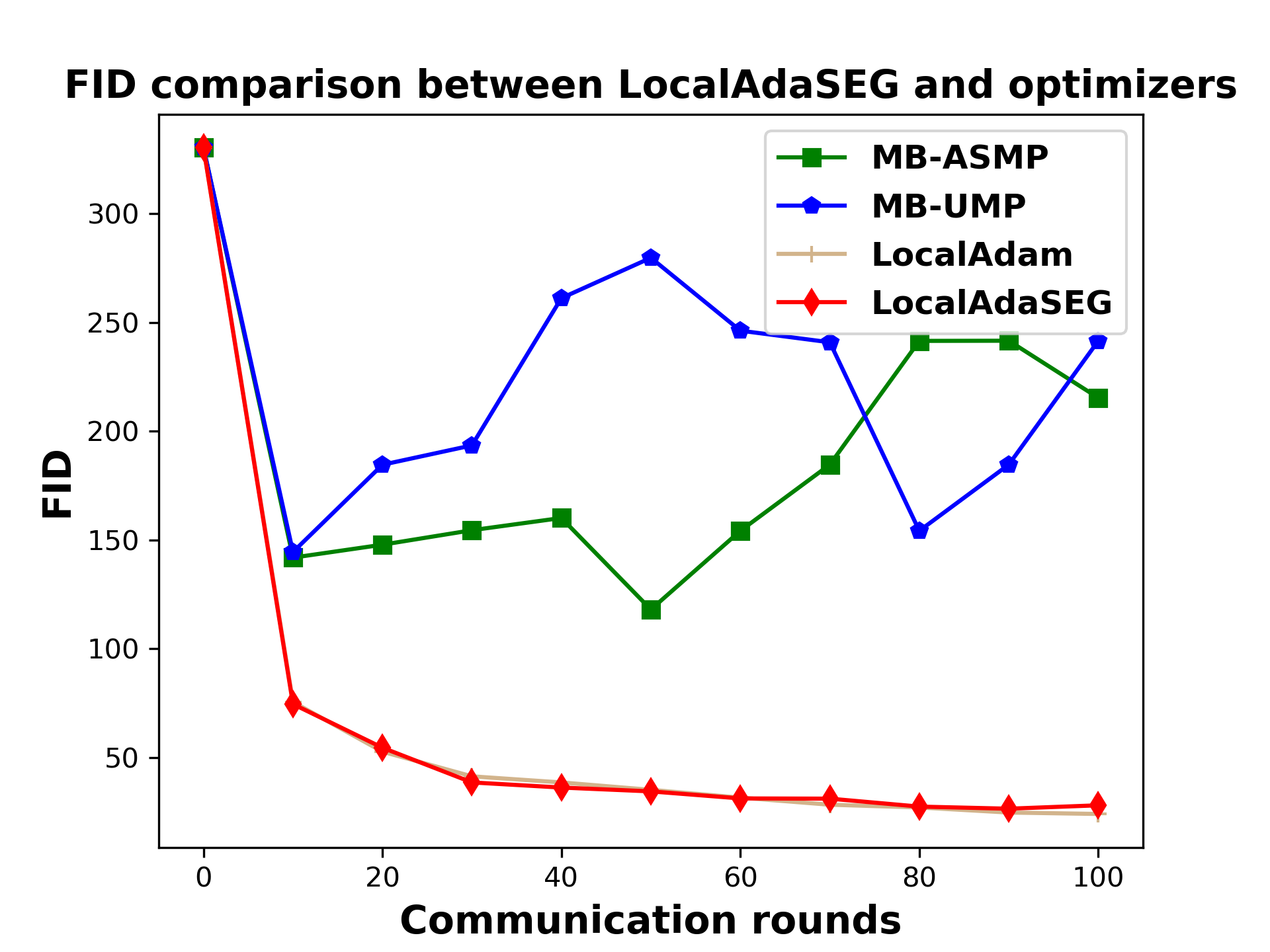}}
	\subfigure[]{
		\includegraphics[width = 0.45\textwidth]
		{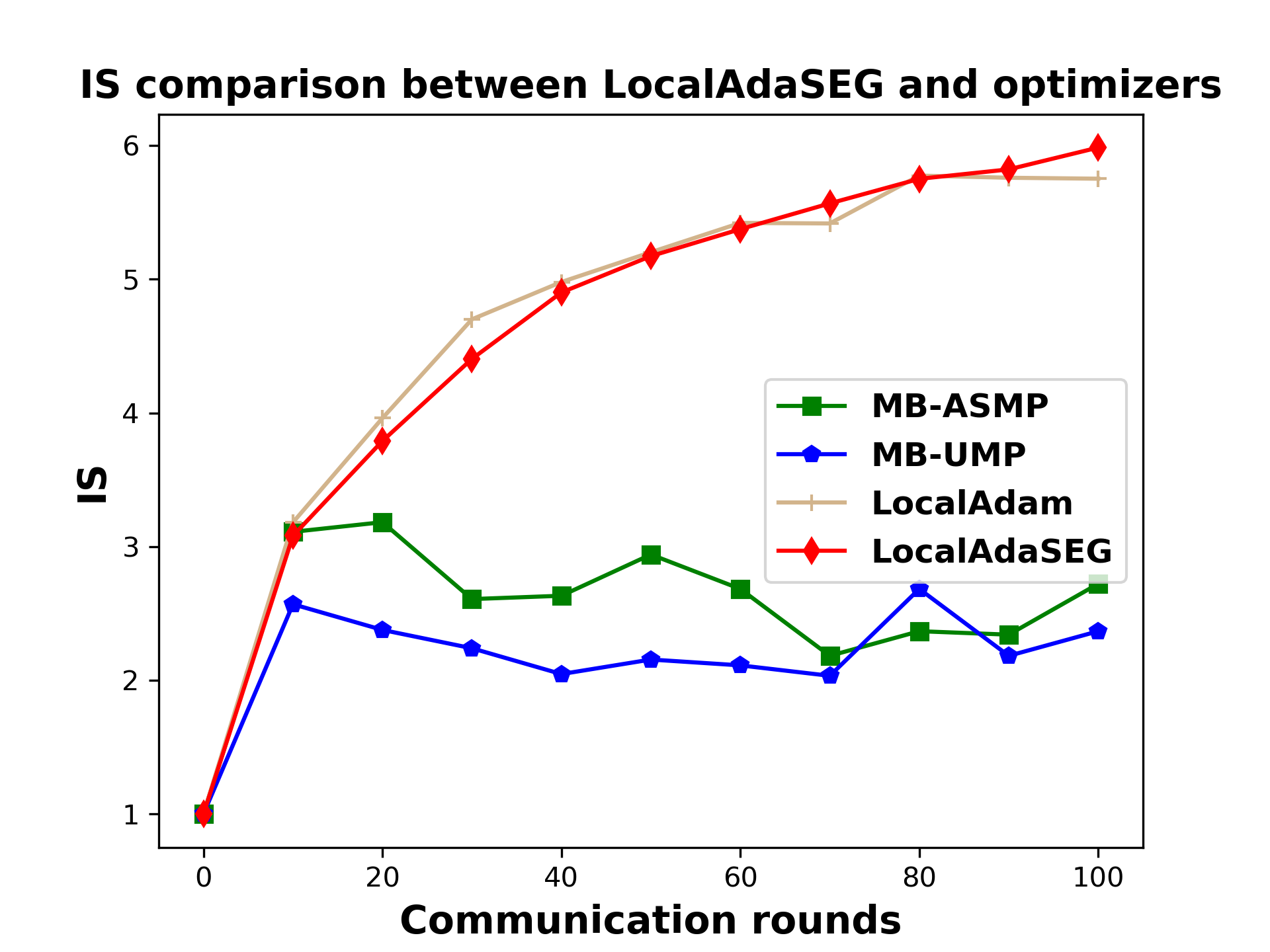}}

	\caption{\small Subfigures (a)-(b) show the FID and IS against communication rounds of BigGAN over the CIFAR10 dataset.}
	\label{fig:bigGAN_FID_IS}
\end{figure*}

\end{appendices}

\end{document}